%% file: main.tex
\documentclass[priprint]{article}

\usepackage{blindtext}

\usepackage[abbrvbib, preprint]{jmlr2e}

\newtheorem{assumption}[theorem]{Assumption}

\usepackage{lastpage}
\jmlrheading{26}{2026}{1-\pageref{LastPage}}{2/21; Revised -/-}{-/-}{-}{Yifei Liang and Yan Sun and Xiaochun Cao and Li Shen}

\ShortHeadings{Stability and Generalization of Push-Sum Based Decentralized Optimization}{Liang and Sun and Cao and Shen}
\firstpageno{1}

\usepackage{amsfonts}
\usepackage{array}
\usepackage[caption=false,font=normalsize,labelfont=sf,textfont=sf]{subfig}
\usepackage{textcomp}
\usepackage{stfloats}
\usepackage{url}
\usepackage{verbatim}
\usepackage{graphicx}
\usepackage{longtable}
\usepackage{nicefrac}       
\usepackage{microtype}      
\usepackage{soul, color,xcolor}         
\usepackage{hyperref} 
 \hypersetup{
	colorlinks=true,    
	linkcolor=blue,     
	citecolor=blue,     
	urlcolor=blue       
}
\usepackage{tcolorbox}
\usepackage{bm}             
\usepackage{threeparttable}

\usepackage{booktabs}       
\usepackage{makecell}       
\usepackage{adjustbox}      
\usepackage{graphicx}       
\usepackage{multirow}
\usepackage{tabularx} 

\usepackage{pdfpages}
\DeclareSubrefFormat{myparens}{#1(#2)}
\usepackage{orcidlink} 
\usepackage{ulem} 

\usepackage{algorithmic}
\usepackage{algorithm}

\usepackage{amssymb}
\usepackage{mathtools}
\usepackage{enumitem} 

\newcommand{\ie}{i.e.\,}

\usepackage{titletoc}

\begin{document}

\title{Stability and Generalization of Push-Sum Based Decentralized Optimization over Directed Graphs}

\author{\name Yifei Liang \email liangyf65@mail2.sysu.edu.cn \\
       \addr  School of Cyber Science and Technology\\ 
       Sun Yat-sen University\\
       Shenzhen, Guangdong 518107, China
       \AND
       \name Yan Sun \email sun9899@uni.sydney.edu.au \\
       \addr The University of Sydney\\
       Sydney, NSW 2006, Australia
       \AND
       \name Xiaochun Cao \email caoxiaochun@mail.sysu.edu.cn \\
       \addr School of Cyber Science and Technology\\ 
       Sun Yat-sen University\\
       Shenzhen, Guangdong 518107, China
       \AND
       \name Li Shen \email shenli6@mail.sysu.edu.cn \\
       \addr School of Cyber Science and Technology\\ 
       Sun Yat-sen University\\
       Shenzhen, Guangdong 518107, China}
\editor{My editor}

\maketitle

\begin{abstract}
Push-Sum-based decentralized learning enables optimization over directed communication networks, where information exchange may be asymmetric. 
While convergence properties of such methods are well understood, their finite-iteration stability and generalization behavior remain unclear due to structural bias induced by column-stochastic mixing and asymmetric error propagation.
In this work, we develop a unified uniform-stability framework for the Stochastic Gradient Push (SGP) algorithm that captures the effect of directed topology. 
A key technical ingredient is an imbalance-aware consistency bound for Push-Sum, which controls consensus deviation through two quantities: the stationary distribution imbalance parameter $\delta$ and the spectral gap $(1-\lambda)$ governing mixing speed. 
This decomposition enables us to disentangle statistical effects from topology-induced bias.
We establish finite-iteration stability and optimization guarantees for both convex objectives and non-convex objectives satisfying the Polyak--\L{}ojasiewicz condition. 
For convex problems, SGP attains excess generalization error of order 
$\tilde{\mathcal{O}}\!\left(\frac{1}{\sqrt{mn}}+\frac{\gamma}{\delta(1-\lambda)}+\gamma\right)$ 
under step-size schedules, and we characterize the corresponding optimal early stopping time that minimizes this bound. 
For P\L{} objectives, we obtain convex-like optimization and generalization rates with dominant dependence proportional to 
$\kappa\!\left(1+\frac{1}{\delta(1-\lambda)}\right)$, 
revealing a multiplicative coupling between problem conditioning and directed communication topology.
Our analysis clarifies when Push-Sum correction is necessary compared with standard decentralized SGD and quantifies how imbalance and mixing jointly shape the best attainable learning performance. 
Experiments on logistic regression and image classification benchmarks under common network topologies validate the theoretical findings.
\end{abstract}


\input{Latex/introduction}

\input{Latex/related_work}

\input{Latex/preliminaries}

\input{Latex/theoretical_analysis}

\input{Latex/experiment}

\input{Latex/conclution}

\bibliography{reference}

\newpage
\allowdisplaybreaks 
{
\begin{center}
    \LARGE \textbf{APPENDIX}
\end{center}
}
\startcontents[sections]  
\printcontents[sections]{}{1}{\setcounter{tocdepth}{3}}  
\vskip 0.6in

\newpage
\appendix

\input{Latex/appendix_a}

\input{Latex/appendix_b}

\input{Latex/appendix_c}

\end{document}

%% file: Latex/introduction.tex
\section{Introduction}

Decentralized Learning (DL) has become a standard paradigm for large-scale machine learning due to its advantages in privacy preservation~\citep{cyffers-2022}, training efficiency~\citep{lian-2017}, and system robustness~\citep{neglia-2019}.  
By distributing computation and data across multiple nodes, DL avoids the need for centralized data aggregation and naturally supports collaborative training in resource-constrained or privacy-sensitive environments. 
Classical decentralized algorithms, such as Decentralized SGD (D-SGD)~\citep{lian-2017,koloskova-2020}, typically assume undirected communication graphs, where information exchange between nodes is symmetric and averaging operations are unbiased.
In many practical systems, however, communication is inherently directed, and information flow may be one-way due to heterogeneous transmission power, asymmetric bandwidth, or packet loss (Figure~\ref{fig:Topology_compare}). 
The Push-Sum protocol~\citep{kempe-2003} was introduced to achieve average consensus over directed graphs and later extended to distributed optimization under asymmetric communication~\citep{tsianos-2012,nedic-2016,benezit-2010}. 
Building on this idea, \citet{assran-2019} proposed Stochastic Gradient Push (SGP), which combines Push-Sum normalization with parallel SGD and enables decentralized learning over directed networks.

\begin{figure}[t]
    \centering
    \includegraphics[width=\columnwidth]{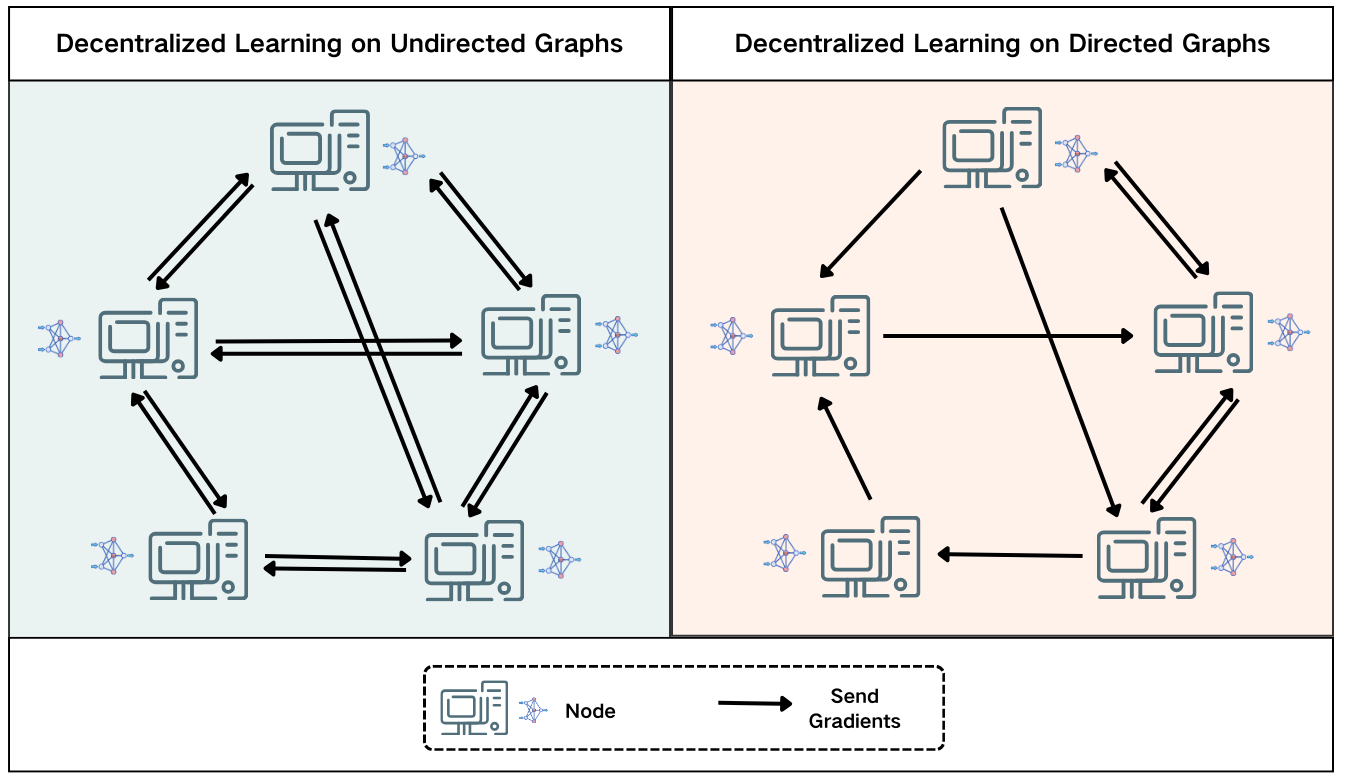}
    \caption{Comparison of symmetric and asymmetric topologies.}
    \label{fig:Topology_compare}
    \vspace{-1.2\baselineskip}
\end{figure}

While the convergence properties of Push-Sum-based algorithms have been studied~\citep{nedic-2016,assran-2019}, their stability and generalization behavior at finite iterations remain largely unexplored in theory, especially in realistic learning settings.
Understanding generalization in directed decentralized learning is particularly challenging because communication asymmetry introduces structural bias that cannot be treated as a small perturbation.
The main technical difficulty arises from the use of column-stochastic mixing matrices. 
Unlike doubly-stochastic averaging, column-stochastic communication does not preserve symmetry and instead induces a non-uniform stationary distribution. 
As a consequence, consensus errors may persist throughout training and interact with stochastic gradient noise in a nontrivial way over time. 
These effects lead to additional instability mechanisms that are absent in centralized SGD~\citep{hardt-2016} and are not captured by existing decentralized stability analyses developed for undirected D-SGD~\citep{sun-2021,lebars-2023}. 
In particular, it remains unclear how imbalance and slow mixing jointly shape the excess risk and whether Push-Sum correction fundamentally changes the generalization behavior compared with standard decentralized methods.
This motivates the following questions: \textbf{(i)} under what conditions is Push-Sum correction necessary compared with standard decentralized SGD, and \textbf{(ii)} how does directed network topology influence generalization error and excess risk?

To address these questions, we develop a unified analysis of the stability and optimization behavior of SGP over directed communication topologies.
Our objective is to quantify precisely how directed mixing affects generalization and excess risk at finite iterations.
To provide the necessary structural context, we first distinguish balanced and imbalanced communication graphs. 
By a classical result~\citep{olfati-2004}, 
in balanced graphs, aggregation is unbiased and the dynamics reduce to the standard undirected setting; in imbalanced directed graphs, only column-stochastic mixing is possible, which induces a non-uniform stationary distribution and structural bias.
Motivated by this distinction, we characterize directed mixing matrices through their stationary distributions and introduce an explicit measure of topological imbalance, denoted by $\delta$ (Definition~\ref{def:delta}). 
This parameter measures the deviation of the stationary distribution from uniformity and provides a concrete criterion for when unbiased aggregation is achievable. 
In particular, when $\delta=\frac{1}{m} $, the network is balanced and aggregation is effectively unbiased; when $\delta<\frac{1}{m}$, Push-Sum normalization is necessary to eliminate structural bias. 
This characterization directly addresses question \textbf{(i)} by identifying the regime in which Push-Sum correction is required.
An important technical ingredient in our analysis is a refined consistency bound for Push-Sum (Lemma~\ref{lem:push_sum_consistency}), which controls the deviation between local iterates and their network average under directed mixing. 
Building on this decomposition, we disentangle two distinct topological effects: the spectral gap $(1-\lambda)$, which governs the mixing speed, and the imbalance factor $\delta$, which captures asymmetry in aggregation (Section~\ref{sec:topology_properties}). 
By separating these quantities, we derive stability bounds (Section~\ref{sec:thm}) that explicitly show how directed topology enters the excess risk through the factors $1/(\delta(1-\lambda))$. 
This provides a quantitative answer to question \textbf{(ii)}, clarifying how imbalance and slow mixing jointly affect generalization.
Lastly, we establish finite-iteration optimization guarantees for both convex and non-convex (including Polyak--\L{}ojasiewicz) objectives under constant and diminishing step-size schedules. 
Combining stability and optimization results yields explicit bounds on the excess generalization error and reveals a trade-off between optimization accuracy and algorithmic stability. 
In particular, we characterize the optimal early stopping time that minimizes the excess risk and derive the corresponding minimal achievable generalization error. 
Our analysis shows that directed topology influences not only convergence behavior but also the best attainable excess risk through the imbalance parameter $\delta$ and the spectral gap $(1-\lambda)$. 
Together, these results provide a unified understanding of when directed communication degrades learning performance and when its effect reduces to that of standard undirected decentralized SGD.

Our work builds on uniform stability theory, which has been widely used to relate algorithmic sensitivity to generalization performance~\citep{devroye-1979,mcallester-1999}(Section~\ref{sec:gen_sta_def}). 
For centralized SGD, \citet{hardt-2016} derived stability guarantees in convex settings. 
These results were later extended to decentralized scenarios, including D-SGD~\citep{sun-2021} and asynchronous variants~\citep{deng-2023}.
However, existing decentralized stability analyses typically assume symmetric communication and rely on doubly-stochastic mixing matrices. 
They therefore do not isolate the role of stationary distribution imbalance in directed networks.
By explicitly introducing the imbalance parameter $\delta$ and separating it from the spectral gap $(1-\lambda)$, our framework generalizes these results to directed communication and identifies the additional instability induced by asymmetric mixing. 
This distinction explains why directed decentralized learning may exhibit fundamentally different generalization behavior from its undirected counterpart.

\textbf{Our Contributions}

\begin{table}[t]
\caption{Comparison of stability ($\epsilon_{\mathrm{stab}}$) and optimal error ($\epsilon_{\mathrm{opt}}$) for distributed learning algorithms. These theoretical results are obtained after $T$ iterations across $m$ distributed nodes processing $n$ training samples, where $C_\lambda\asymp 1/(1-\lambda)$ (spectral gap const), $\delta$ (asymmetry const), $d$ (input dimension), $C_{w_0}$ (initial point const), $v$ (learning rate coefficient), $L$ (smoothness).}
\label{tab:contribution}
\centering
\small
\renewcommand{\arraystretch}{1.2}
\setlength{\tabcolsep}{2pt}

\begin{tabular}{l c c c c}
\hline\hline
\textbf{Algorithm} & \textbf{Setting} & \textbf{Learning Rate} & \textbf{$\epsilon_{\mathrm{stab}}$} & \textbf{$\epsilon_{\mathrm{opt}}$} \\
\hline
\makecell[l]{\textbf{C-SGD}\\~\citep{sun-2023}}
& Non-Convex
& $\gamma_t = \mathcal{O}(\frac{1}{t})$
& $\mathcal{O}\!\left( \frac{1}{mn}T^{\frac{vL}{1 + vL}}\right)$
& $\mathcal{O}\!\left( \frac{1}{nT}\right)$ \\
\hline

\multirow{3}{*}{\makecell[l]{\textbf{D-SGD}\\ ~\citep{sun-2021}}} 
    & \multirow{2}{*}{Convex} 
    & $\gamma_t = \gamma$ 
    & $\mathcal{O}\!\left((\frac{1}{mn}+C_{\lambda})T\right)$ 
    & $\mathcal{O}\!\left((1 + C_{\lambda})\frac{1}{T}\right)$ \\    
    & & $\gamma_t = \mathcal{O}(\frac{1}{t})$ 
    & $\mathcal{O}\!\left((\frac{1}{mn}+C_{\lambda})\ln T\right)$ 
    & $\mathcal{O}\!\left((1 + C_{\lambda})\frac{1}{\ln T}\right)$ \\ 
    \cline{2-5}    
    & Non-Convex 
    & $\gamma_t = \mathcal{O}(\frac{1}{t})$ 
    & $\mathcal{O}\!\left((C_{\lambda} + \frac{1}{mn})T^{\frac{vL}{1 + vL}}\right)$ 
    & -- \\ 
\hline

\multirow{4}{*}{\makecell[l]{\textbf{SGP} (Ours)}} 
    & \multirow{2}{*}{Convex} 
    & $\gamma_t = \gamma$ 
    & $\mathcal{O}\!\left((\frac{1}{mn}+\frac{C_{\lambda}}{\delta})T\right)$ 
    & $\mathcal{O}\!\left((1+ \frac{C_{\lambda}C_{w_0}}{\delta})\frac{1}{T}\right)$ \\  
    & & $\gamma_t = \mathcal{O}(\frac{1}{t})$ 
    & $\mathcal{O}\!\left(\frac{1}{mn}\ln T\right)$
    & $\mathcal{O}\!\left((1 + \frac{C_{\lambda}C_{w_0}}{\delta})\frac{1}{\ln T}\right)$ \\ 
    \cline{2-5}
    & \multirow{2}{*}{Non-Convex}
    & $\gamma_t = \gamma$ 
    & $\mathcal{O}\!\left(\left(\frac{1}{mn}+\frac{C_{\lambda}(C_{w_0}+1)}{\delta}\right)\exp(L\gamma T)\right)$ 
    & $\mathcal{O}\!\left((1 + \frac{C_{\lambda}C_{w_0}}{\delta})\frac{1}{T}\right)$ \\
    & & $\gamma_t = \mathcal{O}(\frac{1}{t})$ 
    & $\mathcal{O}\!\left(\frac{C_{w_0}+1}{\delta mn}T^{\frac{1+vL}{2+vL}}+\frac{C_{\lambda}}{\delta}T^{\frac{vL}{2+vL}}\right)$ 
    & $\mathcal{O}\!\left((1+ \frac{C_{\lambda}C_{w_0}}{\delta})\frac{1}{\ln T}\right)$ \\ 
\hline\hline
\end{tabular}

\vspace{-0.2cm}
\end{table}

\begin{itemize}
  \item \textbf{Topology characterization and Push-Sum criterion.}  
  We analyze directed communication from a Markov chain perspective and introduce an explicit imbalance parameter $\delta$, which characterizes when unbiased aggregation is achievable and when Push-Sum correction is necessary in directed networks, especially under practical asymmetric connectivity.

   \item \textbf{Stability, optimization, and excess risk analysis.}  
  We establish unified finite-iteration bounds on uniform stability and optimization error for SGP over directed graphs in convex settings. 
  By combining these results, we derive guarantees on the excess generalization error and identify the optimal early stopping time that minimizes it. 
  Our analysis disentangles the roles of the spectral gap $(1-\lambda)$ and the imbalance factor $\delta$, and shows that the minimal excess risk decomposes into a statistical term and a topology-dependent bias term. 
  A detailed comparison of rates is provided in Table~\ref{tab:contribution}.

    \item \textbf{Non-convex and P\L{} analysis.}  
    For general non-convex objectives, we characterize how directed mixing amplifies instability under constant step sizes and yields polynomial growth under diminishing schedules. 
    Under the Polyak--\L{}ojasiewicz condition, we obtain convex-like optimization rates with constants that depend explicitly on $\kappa/(\delta(1-\lambda))$, highlighting the coupling between problem conditioning and network topology.

    \item \textbf{Empirical validation.}  
    We validate our theoretical predictions on logistic regression (a9a) and image classification (CIFAR-10 with LeNet) in Section~\ref{sec:exp}, illustrating the influence of topology and step-size schedules in practice, across diverse directed network settings.
\end{itemize}

%% file: Latex/related_work.tex
\section{Related work}

\textbf{Decentralized Learning over Directed Graphs.} Decentralized learning over directed graphs builds upon foundational work in stochastic approximation~\citep{robbins-1951} and distributed online prediction~\citep{agarwal-2011, dekel-2012}. Algorithms such as D-SGD~\citep{koloskova-2020} achieve linear speedup on symmetric networks~\citep{lian-2017} but fail on directed graphs because asymmetric weight matrices break the doubly stochastic property required for consensus. Early theoretical support came from consensus analysis on switching topologies~\citep{olfati-2004}, asynchronous optimization theory~\citep{tsitsiklis-1986}, and the study of nonhomogeneous Markov chains. The Push-Sum protocol~\citep{kempe-2003} resolved the asymmetry challenge by introducing auxiliary variables alongside column-stochastic matrices, enabling exact average consensus without knowledge of network size or out-degrees. This idea was extended to weighted gossip with general stochastic matrices~\citep{benezit-2010}, to convex optimization via distributed dual averaging~\citep{tsianos-2012}, and to time-varying directed graphs by Nedić and Olshevsky~\citep{nedic-2015, nedic-2016-gaussian}, yielding a convergence rate of $O(\ln t / \sqrt{t})$. Stochastic Gradient Push~\citep{nedic-2017} improved this to $O(\ln t / t)$ for strongly convex objectives. To escape the sublinear convergence inherent in Push-Sum based methods, gradient tracking emerged: EXTRA~\citep{shi-2015} achieved linear convergence on undirected graphs by tracking the average gradient, while DEXTRA~\citep{xi-2017} and ExtraPush~\citep{zeng-2017} adapted this mechanism to directed networks using constant stepsizes. The Push-Pull framework~\citep{pu-2020} unified these ideas through a dual-matrix architecture, employing row-stochastic matrices to pull iterates and column-stochastic matrices to push gradients, thereby separating consensus and optimization dynamics. This design achieves linear speedup for non-convex stochastic problems and remains stable under unidirectional communication without additional nonlinear corrections. In deep learning, Stochastic Gradient Push~\citep{assran-2019} combines Push-Sum with stochastic gradients to ensure consensus while preserving SGD's convergence rate; its asynchronous extension~\citep{assran-2020} tolerates communication delays. Recent progress includes quantized communication~\citep{taheri-2020}, personalized federated learning via directed partial gradient push~\citep{liu-2024} or asymmetric topologies~\citep{li-2023}, and B-ary tree structures for heterogeneous data~\citep{you-2024}. Ongoing research explores adaptive edge weighting, sporadic gradient tracking, generalized smoothness conditions, and topology learning informed by semantic structure, enhancing decentralized optimization in dynamic environments.

\textbf{Stability and Generalization.} Algorithmic stability offers a principled way to bound the generalization error of learning algorithms~\citep{bousquet-2002, elisseeff-2005, shalev-shwartz-2010}, building on earlier frameworks such as VC dimension~\citep{blumer-1989}, Valiant's PAC learning model, and Rademacher complexity~\citep{bartlett-2002}. \citet{bousquet-2002} showed that hypothesis stability suffices for generalization, and ~\citet{hardt-2016} applied this to stochastic gradient descent, proving that uniform stability degrades with more iterations under constant stepsizes, which loosens generalization bounds over time and suggests that faster training or early stopping can improve generalization. Later work introduced data-dependent stability~\citep{kuzborskij-2018}, on-average stability with convergence-aware analysis~\citep{charles-2018, lei-2020}, refined bounds under weaker assumptions~\citep{bassily-2020}, and extensions to nonsmooth or adversarial settings~\citep{xiao-2022, deng-2024}. In distributed learning, these ideas extend to multi-agent systems, yielding generalization comparable to centralized SGD under doubly stochastic mixing~\citep{sun-2021} and topology-dependent bounds that incorporate spectral gaps~\citep{zhu-2022}. Le Bars et al.~\citep{lebars-2023} exploited the row-stochastic property of mixing matrices, which preserves the global average of iterates, to transfer centralized stability arguments to decentralized settings without explicit dependence on graph structure. Push-Sum and related methods for directed graphs, however, rely on column-stochastic matrices. The average-preserving invariance no longer holds because of asymmetry and dynamic imbalance correction through auxiliary variables. Consequently, stability analyses designed for row-stochastic networks do not apply: local iterates accumulate bias and weighted averaging breaks uniform stability across agents. This gap calls for new analytical tools tailored to column-stochastic networks. Current efforts aim to develop robust generalization guarantees that account for evolving weights and imbalance in fully directed and unbalanced communication settings.

%% file: Latex/preliminaries.tex
\section{Preliminaries}
\label{sec:preliminaries}

This section establishes the mathematical framework of our analysis. We formulate the distributed optimization problem in Subsection~\ref{sec:problem_setup}, describe the key topological properties of the communication graph in Subsection~\ref{sec:topology_properties}, and present the Stochastic Gradient Push (SGP) algorithm and its consensus dynamics in Subsection~\ref{sec:sgp_alg}. Finally, Subsection~\ref{sec:gen_sta_def} introduces the stability and generalization metrics, distinguishing the dynamics used for analysis from the final output model used for evaluation.

\subsection{Problem Formulation}
\label{sec:problem_setup}

Consider a decentralized system consisting of $m$ nodes, indexed by $\mathcal{V} = \{1, \ldots, m\}$. Each node $i$ operates on an input space $\mathcal{X}_i \subseteq \mathbb{R}^{d}$ and an output space $\mathcal{Y}_i \subseteq \mathbb{R}$. The sample space is denoted by $\bm{S}_i = \mathcal{X}_i \times \mathcal{Y}_i$, where data samples are drawn independently and identically (i.i.d.) from a distribution $\mathcal{D}_i$. The sample space is the union $\bm{S} = \bigcup_{i=1}^m \bm{S}_i$, with each node possessing a dataset of size $n$.

The global goal is to collaboratively learn an optimal parameter $\bm{w}^* \in \mathbb{R}^d$ that minimizes the expected global risk $F(\bm{w})$, given as the mean of local expected risks:
\begin{equation}
\begin{aligned}
    \min_{\bm{w} \in \mathbb{R}^d} F(\bm{w}) &= \frac{1}{m} \sum_{i=1}^m F_i(\bm{w}), \\
    F_i(\bm{w}) &= \mathbb{E}_{\bm{\xi} \sim \mathcal{D}_i} [f(\bm{w}; \bm{\xi})].
\end{aligned}
\end{equation}

Here, $f(\bm{w}; \bm{\xi})$ denotes the loss function computed for a sample $\bm{\xi}$. Because the distributions $\mathcal{D}_i$ are unknown in practice, we rely on the Empirical Risk Minimization (ERM) framework to approximate the expected risk using observed data. With $\bm{S}_i = \{\bm{\xi}_{i,1}, \ldots, \bm{\xi}_{i,n}\}$ as the local dataset, the global empirical risk $F_{\bm{S}}(\bm{w})$ and its minimizer $\bm{w}_{\bm{S}}^*$ are defined as:
\begin{equation}
\begin{aligned}
    F_{\bm{S}}(\bm{w}) :&= \frac{1}{mn} \sum_{i=1}^m  \sum_{\zeta=1}^n f(\bm{w}; \bm{\xi}_{i,\zeta})= \frac{1}{m} {F_{\bm{S}_i}(\bm{w})}, \\
    \bm{w}_{\bm{S}}^* :&= \arg \min_{\bm{w}} F_{\bm{S}}(\bm{w}).
    \end{aligned}
\end{equation}

\subsection{Communication Topology and Structural Properties}
\label{sec:topology_properties}

The agents communicate over a strongly connected directed graph $\mathcal{G} = (\mathcal{V}, \mathcal{E})$. A directed edge $(j, i) \in \mathcal{E}$ indicates that node $i$ receives information from node $j$. The In- and Out-neighbor sets of node $i$ are $\mathcal{N}_i^{\mathrm{in}} := \{j \mid (j, i) \in \mathcal{E}\} \cup \{i\}$ and $\mathcal{N}_i^{\mathrm{out}}: = \{j \mid (i, j) \in \mathcal{E}\} \cup \{i\}$, with the out-degree  $d_i: = |\mathcal{N}_i^{\mathrm{out}}|$. The mixing matrix $\bm{P} \in \mathbb{R}^{m \times m}$ is defined as:
\begin{equation}
    [\bm{P}]_{ij} =
    \begin{cases}
        1 / (d_j+1), & \text{if } j \in \mathcal{N}_i^{\mathrm{in}}; \\
        0, & \text{otherwise.}
    \end{cases}
\end{equation}
which ensures column-stochasticity, i.e., $\bm{1}^\top \bm{P} = \mathbf{1}^\top$.

The performance of distributed learning algorithms is influenced by the underlying communication topology. As discussed above, general directed graphs admit only column-stochastic mixing matrices. An exception class of directed graphs defined as follows.

\begin{definition}[Balanced Graph~\citep{olfati-2004}]
\label{def:balanced_graph}
A directed graph $\mathcal{G}$ is balanced if and only if the in-degree equals the out-degree for every node $i \in \mathcal{V}$, \ie, $|\mathcal{N}_i^{\mathrm{in}}| = |\mathcal{N}_i^{\mathrm{out}}|$ holds for all nodes $i$ in the vertex set $\mathcal{V}$.
\end{definition}

Such graphs admit the following fundamental characterization.

\begin{lemma}[Property of Balanced Graph~\citep{olfati-2004}]
\label{thm:doubly_stochastic}
There exists a non-negative matrix $\bm{P}$ compatible with $\mathcal{G}$ that is doubly stochastic (satisfying $\bm{P}\mathbf{1} = \mathbf{1}$ and $\mathbf{1}^\top \bm{P} = \mathbf{1}^\top$) if and only if graph $\mathcal{G}$ is balanced.
\end{lemma}

From a Markov chain perspective, the information exchange induced by the communication matrix $\bm{P}$ can be described as a random walk over the graph topology:
\[
    \bm{w}^{(t+1)} = \bm{P}\bm{w}^{(t)},
\]
where $\bm{w}^{(t)}$ denotes the model parameter held by each agent at iteration $t$. The matrix $\bm{P}$ is assumed to be nonnegative, column-stochastic, and primitive, consistent with the mixing matrices introduced above. Its long-term behavior is characterized by the \textit{Perron--Frobenius theorem}~\citep{horn-2012,meyer-2023}:
\[
    \lim_{t \to \infty} \bm{P}^{t} = \boldsymbol{\pi}\mathbf{1}^\top,
\]
where $\boldsymbol{\pi} \in \mathbb{R}^m_{++}$ is the unique stationary distribution, satisfying $\bm{P}\boldsymbol{\pi}=\boldsymbol{\pi}$ and $\|\boldsymbol{\pi}\|_1=1$, and characterizing the limiting influence that each agent contributes to the aggregation process.

The communication graph structure determines $\boldsymbol{\pi}$, giving two representative regimes.

\textbf{Balanced Graph:} If the underlying graph admits a doubly stochastic mixing matrix, then the stationary distribution becomes uniform, \ie, $\boldsymbol{\pi} = \frac{1}{m}\mathbf{1}$. In this regime, traditional D-SGD maintains unbiased optimization, as each node contributes equally to global average.

\textbf{Unbalanced Graph:} When the in-degree $|\mathcal{N}_i^{\mathrm{in}}|$ and out-degree $|\mathcal{N}_i^{\mathrm{out}}|$ of node $i$ differ, matrix $\bm{P}$ is column-stochastic but not row-stochastic ($\bm{P}\mathbf{1} \neq \mathbf{1}$), yielding a non-uniform stationary distribution $\boldsymbol{\pi}$. This imbalance introduces systematic bias in DL aggregation.

Understanding the distinction between balanced and unbalanced graphs is essential for characterizing how network topology influences distributed optimization algorithms. Figure~\ref{fig:Balanced_compare} illustrates the classification of several commonly used decentralized network topologies.

\begin{figure}[h]
    \centering
    \includegraphics[width=\columnwidth]{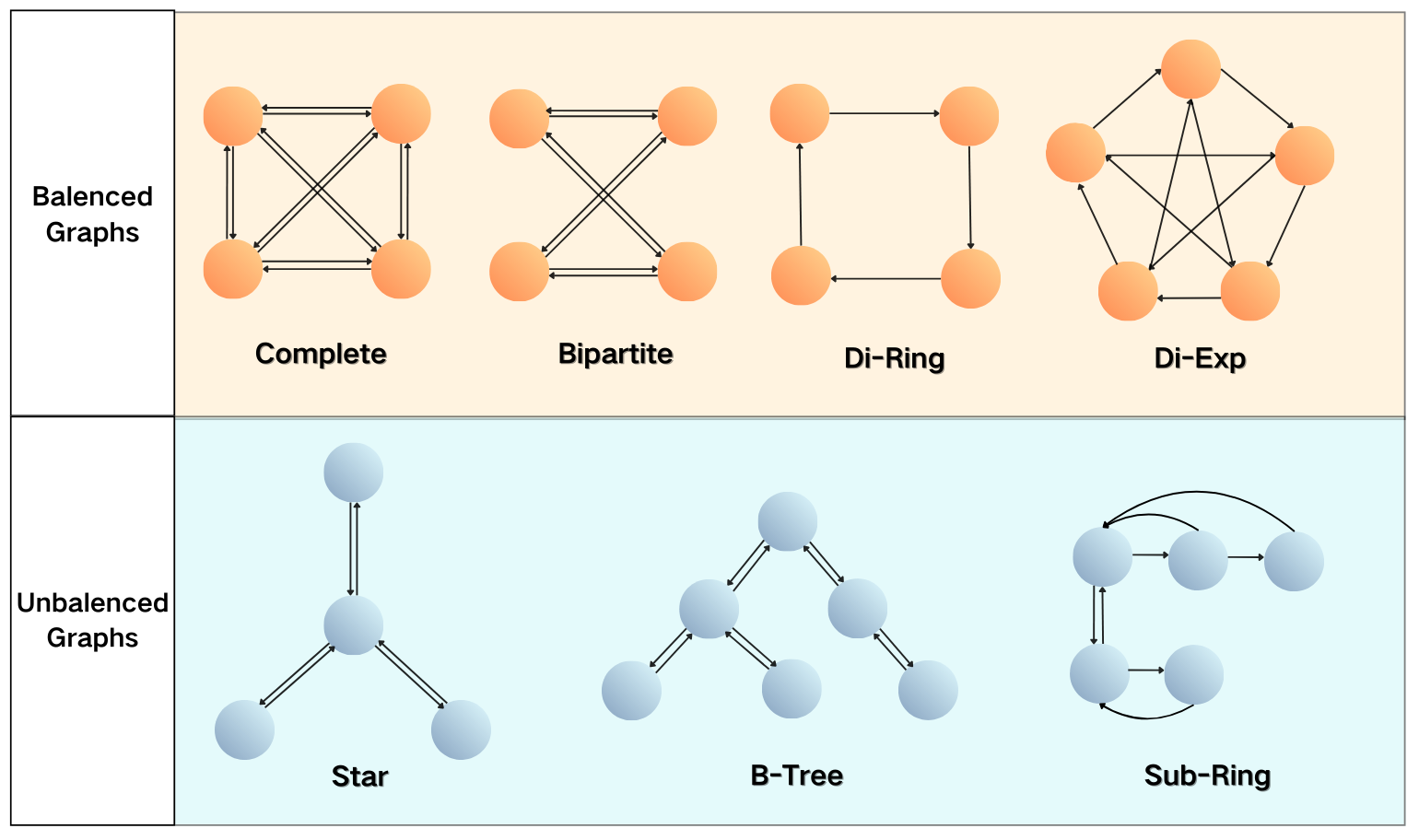}
    \caption{Classification of Balanced and Unbalanced Graphs.}
    \label{fig:Balanced_compare} 
\end{figure}

To characterize topology and balance, we introduce two fundamental parameters.

\begin{definition}[Spectral Gap~\citep{montenegro-2006}]
\label{def:spectral_gap}
Let $\sigma(\bm{P})$ denote the spectrum of the primitive matrix $\bm{P}$. The second largest eigenvalue modulus (SLEM) is defined as
\begin{equation}
    \lambda := \max \{ |\mu| : \mu \in \sigma(\bm{P}), \mu \neq 1 \}.
\end{equation}
It holds that $0 < \lambda < 1$.
\end{definition}

\begin{remark}
The parameter $\lambda$ characterizes the convergence rate of a Markov chain to its stationary distribution. For an irreducible and aperiodic chain with transition matrix $\bm{P}$, the total variation distance after $k$ steps decays as $O(\lambda^k)$. Hence, the spectral gap $1 - \lambda$ measures the mixing speed: smaller $\lambda$ (larger gap) implies faster mixing. This property is crucial in gossip algorithms, decentralized averaging, and consensus protocols, where $\lambda$ governs the exponential rate at which information homogenizes across nodes. The same analysis applies to doubly stochastic matrices, which preserve the uniform stationary distribution and are widely used in distributed systems for fairness and symmetry. Sharp bounds on $\lambda$ for such matrices and their implications for algorithmic convergence are discussed in~\citep{sun-2021,zhu-2022}.
\end{remark}

\begin{definition}[Topological Imbalance]
\label{def:delta}
By the Perron--Frobenius theorem, there exists a unique stationary distribution vector $\boldsymbol{\pi} \in \mathbb{R}^m_{++}$ such that $\bm{P}\boldsymbol{\pi} = \boldsymbol{\pi}$ and $\|\boldsymbol{\pi}\|_1 = 1$. The topological imbalance parameter $\delta$ is defined as
\begin{equation}
    \delta := \min_{1 \le i \le m} [\boldsymbol{\pi}]_i.
\end{equation}
\end{definition}

\begin{remark}
The topological imbalance parameter $\delta \in (0, 1/m]$ measures the agents' influence on the global state under repeated application of the communication matrix $\bm{P}$. The entries of the stationary distribution $\boldsymbol{\pi}_i$ represent each agent's relative asymptotic authority: agents with larger $\pi_i$ dominate the consensus value, while those with smaller $\pi_i$ contribute less and propagate information more slowly.
For doubly stochastic matrices, the distribution is uniform and $\delta = 1/m$, achieving perfect balance. In highly unbalanced topologies (e.g., directed graphs with large out-degree disparities), $\delta \to 0$. Smaller $\delta$ implies greater difficulty in reaching consensus, and error bounds that typically grow in $1/\delta$
\end{remark}

\subsection{Algorithm: Stochastic Gradient Push (SGP)}
\label{sec:sgp_alg}

In unbalanced topologies, standard D-SGD is ineffective because the stationary distribution $\boldsymbol{\pi}$ is non-uniform, introducing systematic bias. To mitigate this, earlier works~\citep{kempe-2003,tsianos-2012} employ the Push-Sum protocol for approximate averaging.

SGP~\citep{assran-2019} builds on Push-Sum by running two parallel message-passing processes. Node $i$ keeps a proxy vector $\bm{w}_i^{(t)}$ that transports parameter information across the network, and a scalar weight $u_i^{(t)}$ (with $u_i^{(0)}=1$) that tracks how much mixing influence the node gradually accumulates. Because directed graphs amplify messages unevenly, the proxy $\bm{w}_i^{(t)}$ becomes biased on its own. The accompanying weight $u_i^{(t)}$ captures this distortion, and the corrected estimate
$ \bm{z}_i^{(t)} = \bm{w}_i^{(t)} / u_i^{(t)}$
recovers an unbiased representation of the parameter.

\begin{algorithm}[ht]
    \caption{Stochastic Gradient Push (SGP)}
    \label{alg:sgp}
    \begin{algorithmic}[1]
        \REQUIRE {Initialize step size sequence $\{\gamma_t\}$, $\bm{w}_i^{(0)} = \bm{z}_i^{(0)} \in \mathbb{R}^d$, and $u_i^{(0)} = 1$ for all $i \in \mathcal{V}$}
        \FOR{$t = 0,1,2,\cdots, T-1$}
            \FOR{each node $i \in \mathcal{V}$ in parallel}
                \STATE  Sample $\bm \xi_{i}^{(t)} \sim \mathcal{D}$ from local distribution.
                \STATE  Compute stochastic gradient at $\bm{z}_i^{(t)}$: $\bm{g}_i^{(t)}=\nabla f(\bm{z}_i^{(t)}; \bm \xi_{i}^{(t)}), \bm{w}_i^{(t + \frac{1}{2})} = \bm{w}_i^{(t)} - \gamma_t \bm{g}_i^{(t)}.$
                \STATE  Send $\big([\bm{P}]_{ij}\bm{w}_i^{(t+\frac{1}{2})},[\bm{P}]_{ij} {u}_i^{(t)}\big)$ to $\mathcal{N}_i^{\mathrm{out}}(t)$
                 \STATE Receive $\big([\bm{P}]_{ij} \bm{w}_j^{(t + \frac{1}{2})}, [\bm{P}]_{ij} {u}_j^{(t)}\big)$ from $\mathcal{N}_i^{\mathrm{in}}(t)$
                \STATE Update local variables:
                \begin{align}
                    \bm{w}_i^{(t+1)} &= \sum_{j \in \mathcal{N}_i^{\mathrm{in}}} [\bm{P}]_{ij} \left( \bm{w}_j^{(t)} - \gamma_t \bm{g}_j^{(t)} \right) \label{eq:alg_w_update}\\
                    u_i^{(t+1)} &= \sum_{j \in \mathcal{N}_i^{\mathrm{in}}} [\bm{P}]_{ij} u_j^{(t)} \label{eq:alg_u_update}\\
                    \bm{z}_i^{(t+1)} &= \bm{w}_i^{(t+1)} / u_i^{(t+1)} \label{eq:alg_z_update}
                \end{align}
            \ENDFOR
        \ENDFOR
        \ENSURE Output $\overline{\bm{w}}^{(T)}.$
    \end{algorithmic}
\end{algorithm}

We can conclude that the SGP update at each iteration is given by
\begin{equation}
\bm{W}^{(t+1)} = \bm{P}\!\left(\bm{W}^{(t)} - \gamma_t \nabla \bm{f}(\bm{Z}^{(t)}; \bm{S}^{(t)})\right).
\end{equation}
Let $\overline{\bm{w}}^{(t)} := \frac{1}{m}\sum_{i=1}^m \bm{w}_i^{(t)}$ denote the global average of the parameter proxies, which serves as the \textit{consensus model}. Then the induced evolution of this average is as follows.
\begin{proposition}
\label{prop:update}
The uniform average update of SGP satisfies
\begin{equation}
    \overline{\bm{w}}^{(t+1)} = \overline{\bm{w}}^{(t)} 
    - \frac{\gamma_t}{m}\,\bm{1}^\top\nabla \bm{f}(\bm{Z}^{(t)}; \bm{S}^{(t)}).
\end{equation}
\end{proposition}

\begin{proof}
See Appendix~\ref{pro:update} for detailed proof.
\end{proof}

\begin{remark}
Proposition~\ref{prop:update} shows that the evolution of the consensus model $\overline{\bm{w}}^{(t)}$ depends solely on the uniformly averaged stochastic gradients across all agents and is completely independent of the specific structure of the communication matrix $\bm{P}$, despite potentially asymmetric information flow.
\end{remark}

The effectiveness of the algorithm depends on how the local de-biased variables $\bm{z}_i^{(t)}$ follow the consensus model $\overline{\bm{w}}^{(t)}$. The following lemma provides a quantitative bound.


\begin{lemma}[Consistency of Push-Sum]
\label{lem:push_sum_consistency}
Assume the stochastic gradients are uniformly bounded, \ie there exists $G>0$ such that
$\|\nabla f(\bm{z}_i^{(t)};\bm{\xi}_i^{(t)})\|\le G$.
Define $C_{w_0}:=\frac{1}{m}\sum_{i=1}^m\|\bm{w}_i^{(0)}\|$.
Then there exist constants $C>0$ and $\lambda\in(0,1)$ such that for all $t\ge 1$ and all $i\in[m]$,
\begin{equation}\label{eq:pushsum_bound_final}
\| \bm{z}_i^{(t)} - \overline{\bm{w}}^{(t)} \|
\le \frac{C}{\delta} \left( \lambda^t\, C_{w_0}
+  \sum_{s=0}^{t-1} \lambda^{t-s}\gamma_s\, G \right).
\end{equation}
\end{lemma}

\begin{proof}
See Appendix~\ref{pro:push_sum_consistency} for detailed proof.
\end{proof}

\begin{remark}
Lemma~\ref{lem:push_sum_consistency} bounds the consensus error $\|\bm{z}_i^{(t)  } - \overline{\bm{w}}^{(t)}\|$ by a transient term $\lambda^t C_{w_0}$ from initial mismatch (decaying at rate $\lambda$) and a gradient accumulation term smoothed by the same decay.
The $1/\delta$ factor captures the cost of imbalance: modest penalty when $\delta \approx 1/m$ (balanced networks), but significant error amplification when $\delta \to 0$ (weaker influence of low-authority agents). This highlights the core challenge of decentralized algorithms on asymmetric topologies, and shows why Push-Sum remains essential for reliable performance.
\end{remark}

\subsection{Stability and Generalization Measures}
\label{sec:gen_sta_def}

\begin{definition}[Generalization and Optimization Error]
\label{def_gen&opt}
Given a dataset $\bm{S}$ and randomized algorithm $\mathcal{A}\!:\!\bm{S}\! \to\! \bm{W}$, we define:
\begin{enumerate}[label=(\roman*),leftmargin=1cm, itemindent=0cm]
    \item \textit{Generalization error} is $\epsilon_{\mathrm{gen}} = \mathbb{E}_{\bm{S}}[F(\mathcal{A}(\bm{S})) - F_{\bm{S}}(\mathcal{A}(\bm{S}))]$ ,\ie the expected statistical discrepancy between population and empirical risk distributions.
    \item \textit{Excess generalization error} is $\epsilon_{\mathrm{exc}} = \mathbb{E}_{\bm{S}}[F(\mathcal{A}(\bm{S})) - F(\bm{w}^{*})]$ ,\ie the expected performance gap between population risk and the global true minimizer.
    \item \textit{Optimization error} is $\epsilon_{\mathrm{opt}} = \mathbb{E}_{\bm{S}}[F_{\bm{S}}(\mathcal{A}(\bm{S})) - F_{\bm{S}}(\bm{w}_{\bm{S}}^*)]$ ,\ie the expected convergence gap between population risk and the empirical risk minimizer solution. 
\end{enumerate}
\end{definition}

Furthermore, $\epsilon_{\mathrm{exc}}$ can be decomposed as follows:
\[
\begin{gathered}
\mathbb{E}_{\bm{S},\mathcal{A}}\left[F(\mathcal{A}(\bm{S})) - F(\bm{w}^{*})\right] = \underbrace{\mathbb{E}_{\bm{S},\mathcal{A}}\left[F(\mathcal{A}(\bm{S})) - F_{\bm{S}}(\mathcal{A}(\bm{S}))\right]}_{\epsilon_{\mathrm{gen}}} \\
+ \underbrace{\mathbb{E}_{\bm{S},\mathcal{A}}\left[F_{\bm{S}}(\mathcal{A}(\bm{S})) - F_{\bm{S}}(\bm{w}_{\bm{S}}^*)\right]}_{\epsilon_{\mathrm{opt}}} + \underbrace{\mathbb{E}_{\bm{S},\mathcal{A}}\left[F_{\bm{S}}(\bm{w}_{\bm{S}}^*) - F(\bm{w}^*)\right]}_{\leq 0},
\end{gathered}
\]

To evaluate the optimization error $\epsilon_{\mathrm{opt}}$, we first specify the algorithmic output $\mathcal{A}(\bm{S})$. Since the last iterate $\overline{\bm{w}}^{(T)}$ is unstable under stochastic noise in non-convex settings, we follow classical convergence analyses that rely on averaged iterates~\citep{ghadimi-2013}:
\begin{equation}
\label{eq:output_average}
    \bm{w}_{\text{avg}}^{(T)} := 
    \frac{\sum_{t=1}^{T} \gamma_t \overline{\bm{w}}^{(t)}}{\sum_{t=1}^{T} \gamma_t}.
\end{equation}
This construction accounts for time-varying stepsizes and yields a stable surrogate for theoretical analysis. Hence, our excess generalization error is decomposed as :
\begin{equation}
\epsilon_{\mathrm{exc}} \le 
\epsilon_{\mathrm{ave\text{-}stab}}
+ \epsilon_{\mathrm{opt}}.  
\end{equation}

To bound $\epsilon_{\mathrm{gen}}$, we employ the concept of uniform stability:
\begin{definition}[Uniform stability~\citep{bousquet-2002}]
\label{def_uniform_stability}
A stochastic algorithm $\mathcal{A}$ is $\epsilon_{\mathrm{stab}}$-uniformly stable if for any pair of datasets $\bm{S}$ and $\bm{S}^{\prime}$ that differ in at most one training example, the following uniform bound holds:
\begin{equation}
\sup_{\bm{z}} \mathbb{E}_{\mathcal{A}}\left[f(\mathcal{A}(\bm{S}); \bm{z}) - f(\mathcal{A}(\bm{S}^{\prime}); \bm{z})\right] \leq \epsilon_{\mathrm{stab}}.
\end{equation}
\end{definition}

\begin{lemma}
[Generalization for Convex Objectives~\citep{hardt-2016}]
\label{thm:convex-relation}
Let the sto-chastic algorithm $\mathcal{A}$ be $\epsilon_{\mathrm{stab}}$-uniformly stable. Then, the generalization error satisfies:
\begin{equation}
\mathbb{E}_{\bm{S},\mathcal{A}}\left[F(\mathcal{A}(\bm{S})) - F_{\bm{S}}(\mathcal{A}(\bm{S}))\right] \leq \epsilon_{\mathrm{stab}}.
\end{equation}
\end{lemma}

\begin{lemma}
[Generalization for Non-Convex Objectives~\citep{sun-2021}]
\label{thm:nonconvex-relation}
Suppose loss function is gradient bounded under constant $G$. Consider  $\overline{\bm{w}}^{(t)}$ and $\overline{\bm{w}}^{\prime(t)}$ denote the outputs of the decentralized algorithm trained on $\bm{S}$ and $\bm{S}^{\prime}$ respectively at step $t$. Let $\Delta_t \triangleq \| \overline{\bm{w}}^{(t)} - \overline{\bm{w}}^{\prime(t)} \|$.
Then, for any time step $t_0 \in \{0, 1, \ldots, T\}$ and any $\xi_i \sim \mathcal{D}$, under the random update and permutation rules, the generalization error is bounded by:
\begin{equation}
\mathbb{E}_{\bm{S},\mathcal{A}}\left[F(\mathcal{A}(\bm{S})) - F_{\bm{S}}(\mathcal{A}(\bm{S}))\right] \leq \frac{t_0G}{mn}  + G \cdot \mathbb{E}\left[\Delta_T \mid \Delta_{t_0}=0\right].
\end{equation}
\end{lemma}

\begin{remark}
The Theorem~\ref{thm:convex-relation} and Theorem~\ref{thm:nonconvex-relation} implies that once we have access to the uniform stability error, we can derive the generalization gap as an accompanying result.
\end{remark}

\begin{lemma}[Expansion~\citep{hardt-2016}]
\label{lem:hardt}
Fix an update sequence \( \Phi_1, \ldots, \Phi_T \) and another sequence \( \Phi_{1}^{\prime}, \ldots, \Phi_{T}^{\prime} \). Let \( \bm{w}_{0} = \bm{w}_{0}^{\prime} \) be a starting point, \( \bm{w}_{t} \) and \( \bm{w}_{t}^{\prime} \) are defined as:
\[
\bm{w}_{t+1} = \Phi_t(\bm{w}_t) \quad \text{and} \quad \bm{w}_{t+1}^{\prime} = \Phi_t^{\prime}(\bm{w}_t^{\prime}).
\]
For non-negative step sizes \( \gamma_t \geq 0 \) and loss function \( f \), define \( \Phi_{f,\gamma_t} \) as
\[
\Phi_{f,\gamma_t}(\bm{w}, \bm{\xi}) = \bm{w} - \gamma_t \nabla f(\bm{w}),
\]
 Assume \( f \) is \( L \)-smooth. Then, the following properties hold:
\begin{enumerate}[label=(\roman*),leftmargin=1cm, itemindent=0cm]
    \item The update \( \Phi_{f, \gamma_t} \) is \( (1 + L \gamma_t) \)-expansive.
    \item If \( f \) is convex, then for any \( \gamma_t \leq 2 / L \), the update \( \Phi_{f,\gamma_t} \) is \( 1 \)-expansive.
\end{enumerate}
\end{lemma}

Lemma~\ref{lem:hardt} characterizes the evolution of the distance between two optimization trajectories after a single gradient step and provides a fundamental tool for stability-based generalization analysis~\citep{hardt-2016}. 
\textbf{(i)} In the non-convex setting, the stability gap may increase after each update, which can in principle lead to exponential growth of the trajectory discrepancy if not properly controlled. 
\textbf{(ii)} Under convexity and a suitable choice of step size, the gradient update becomes non-expansive, thereby preventing exponential amplification of errors and enabling tighter stability guarantees.

%% file: Latex/theoretical_analysis.tex
\section{Theoretical Analysis}
\label{sec:thm}

In this section, we present a comprehensive theoretical analysis of SGP. We begin by outlining the necessary assumptions in Subsection~\ref{sec:ass}. In Subsection~\ref{sec:con}, we analyze the uniform stability and optimization error for the convex setting, and combine them to derive the excess generalization error, We then extend these results to the non-convex case under the P\L{} condition in Subsection~\ref{sec:non-con}. Finally, in Subsection~\ref{sec:discussion-topology}, we discuss how network topology properties govern the learning performance. Detailed proofs are provided in Appendix~\ref{sec:proof}.

\subsection{Basic Assumptions}
\label{sec:ass}

The analysis requires the assumption stated below:
\begin{assumption}($G$-Lipschitz) The function $f( \bm{x} ; \bm{z} )$ is differentiable with respect to $\bm{x}$ and $G$-Lipschitz for every $\bm{z}$, \ie  $\exists G \geq 0$ such that $| f( \bm{y} ; \bm{z} ) - f( \bm{x} ; \bm{z} ) | \leq {G} \| \bm{y} - \bm{x} \|$.
As a consequence, the gradient is uniformly bounded:
$\|\nabla f(\bm{x};\bm{z})\| \le G.$
\label{ass:lipschitz}
\end{assumption}

\begin{assumption} ($L$-Smooth) The differentiable function $f( \bm{x} ; \bm{z} )$ is $L$-Smooth for every $\bm{z}$ means there exists a constant $L > 0$ such that $\| \nabla f( \bm{y} ; \bm{z} ) - \nabla f( \bm{x} ; \bm{z} )\| \le L \| \bm{y} - \bm{x}\|$.
\label{ass:smooth}
\end{assumption}

\begin{assumption} (Bounded Space) The parameter space is bounded by a closed ball $B(O,r)$ centered at the origin with radius $r > 0$.
\label{ass:bounded}
\end{assumption}

\begin{remark}
We briefly comment on the above assumptions.
\textbf{(i) Lipschitz continuity and smoothness.}
Assumptions~\ref{ass:lipschitz} and~\ref{ass:smooth} are standard in the study of algorithmic stability for gradient-based learning algorithms.
They allow one to control the sensitivity of the loss and the optimization dynamics with respect to small perturbations of the training data.
Similar conditions are commonly imposed in stability analyses of SGD and related methods~\citep{bousquet-2002,hardt-2016,sun-2021}.
\textbf{(ii) Bounded parameter space.}
Assumption~\ref{ass:bounded} ensures that the iterates remain in a compact domain and enables the use of projection arguments in the convergence and stability analysis.
This boundedness condition is also frequently adopted in prior work~\citep{hardt-2016,sun-2021}.
Overall, these assumptions are consistent with the standard setup in existing stability-based generalization results and are not stronger than those typically required in the literature.
\end{remark}

\subsection{Results on Convex Case}
\label{sec:con}
First, we give the definition of a convex function as follows:

\begin{definition}[Convex]
\label{def:convex}
The loss function $f(\bm{x}; \bm{z})$ is said to be convex with respect to $\bm{x}$ if for any $\bm{x}, \bm{y}$ in its domain, it satisfies
\begin{equation}
f(\bm{x}; \bm{z}) \leq f(\bm{y}; \bm{z}) + \langle \nabla f(\bm{y}; \bm{z}), \bm{x}-\bm{y} \rangle.
\end{equation}
\end{definition}

Under convexity, the optimization and generalization results for the convex case follow.

\begin{theorem}[Uniform Stability]
\label{thm:stability-convex}
Assume the loss function $f$ is convex with assumptions~\ref{ass:lipschitz}--~\ref{ass:smooth} hold. 
Then, when $\gamma_t \leq 2/L$, the uniform stability of SGP satisfies:
\begin{equation}
\begin{aligned}
\epsilon_{\mathrm{stab}}
\leq \frac{2CGL C_{w_0}}{\delta }\sum_{t=0}^{T-1}\gamma_t {\lambda}^t
+\frac{2CG^2L}{\delta(1-{\lambda})}\sum_{t=0}^{T-1}{\gamma_t}^2+ \frac{2 G^2}{mn}\sum_{t=0}^{T-1}\gamma_t. 
\end{aligned}
\end{equation}
\end{theorem}
\textbf{Proof Sketch}
To analyze the uniform stability of SGP, we start by bounding the expected divergence between the consensus models \(\overline{\bm{w}}_T\) and \(\overline{\bm{w}}'_T\). By the \(G\)-Lipschitz assumption (Assumption~\ref{ass:lipschitz}), the stability \(\epsilon_{\mathrm{stab}}\) is related to the expected divergence \(\mathbb{E}[\Delta_T]\) as follows:
\[
\epsilon_{\mathrm{stab}} \leq G \cdot \mathbb{E}\|\overline{\bm{w}}_T - \overline{\bm{w}}'_T\| = G \mathbb{E}[\Delta_T],
\]
where \(\Delta_T := \|\overline{\bm{w}}_T - \overline{\bm{w}}'_T\|\) represents the divergence between the models after \(T\) iterations.

Next, we analyze the evolution of \(\Delta_t\) at each iteration. The gradient update step is non-expansive due to the convexity and \(L\)-smoothness of the loss function (Lemma~\ref{lem:hardt}).

When the nodes sample the same data (probability \(1 - \frac{1}{n}\)), the divergence evolves as:
\[
\mathbb{E}[\Delta_{t+1}] \leq \mathbb{E}[\Delta_t] + \frac{2L\gamma_t}{m} \sum_{i=1}^m \mathbb{E}\|\overline{\bm{w}}^{(t)} - \bm{z}_i^{(t)}\| .
\]
The second term can be bounded by using Lemma~\ref{lem:push_sum_consistency}.

When a differing sample is selected (with probability \(\frac{1}{n}\)), the divergence between the models increases due to the discrepancies between the gradient updates of the nodes. Therefore, the total divergence is a combination of both the consensus error and the gradient mismatch due to the differing samples. We express this evolution as:
\[
\mathbb{E}[\Delta_{t+1}] \leq \mathbb{E}[\Delta_t] + \frac{2L\gamma_t}{m} \sum_{i=1}^m \mathbb{E}\|\overline{\bm{w}}^{(t)} - \bm{z}_i^{(t)}\| + \frac{2G\gamma_t}{m}.
\]
Here, the final term accounts for the perturbation introduced by the differing sample.

To obtain the final stability bound, we combine the results from both cases by taking the expectation over the sampling process. Using the fact that, with probability \(1 - \frac{1}{n}\), the nodes sample the same data, and with probability \(\frac{1}{n}\), one node samples a different data point, we obtain the following recurrence for the expected divergence:
\[
\mathbb{E}[\Delta_{t+1}] \leq \mathbb{E}[\Delta_t] + \frac{2L\gamma_t}{\delta} \left( \lambda^t C_{w_0} + G \sum_{s=0}^{t-1} \lambda^{t-s} \gamma_s \right) + \frac{2G\gamma_t}{mn}.
\]
We sum this recurrence from \(t=0\) to \(T-1\) and simplify the terms to arrive at the  bound:
\[
\mathbb{E}[\Delta_T] \leq \frac{2CL C_{w_0}}{\delta} \sum_{t=0}^{T-1} \gamma_t \lambda^t + \frac{2CGL}{\delta(1-\lambda)} \sum_{t=0}^{T-1} \gamma_t^2 + \frac{2G}{mn} \sum_{t=0}^{T-1} \gamma_t.
\]

Finally, by the \(G\)-Lipschitzness of the loss function, we obtain the final stability bound. Detailed derivations are provided in Appendix~\ref{pro:stability-convex}.

\begin{remark}
Theorem~\ref{thm:stability-convex} reveals how decentralization over directed graphs affect stability beyond standard centralized SGD. 
The bound can be decomposed into three components:
\noindent\textbf{(i) Initialization and network mixing.}
The first term characterizes the residual effect of the initial disagreement, which decays geometrically at rate $\lambda^t$. 
Hence, a smaller spectral radius $\lambda$ implies faster consensus, ensuring that the influence of early-round perturbations diminishes rapidly.
\noindent\textbf{(ii) Directed communication penalty.}
The second term reflects the structural difficulty of learning over directed graphs. 
The factor $1/\delta$ quantifies the amplification caused by imbalance in the stationary distribution, while the factor $1/(1-\lambda)$ captures the slowdown due to imperfect connectivity and limited information propagation.
\noindent\textbf{(ii) Stochastic sampling effect.}
The final term, of order $\sum_{t=0}^{T-1}\gamma_t/(mn)$, corresponds to the intrinsic sampling noise in stochastic optimization. 
Notably, it matches the stability scaling of centralized SGD \citep{sun-2023}, indicating that the additional generalization cost of decentralization is entirely governed by the network-dependent terms above.
\end{remark}

The learning rate plays a crucial role in ensuring the stability and generalization behavior of optimization algorithms. We now discuss two commonly used learning rate schemes below:

\begin{corollary}[Uniform Stability on Common Learning Rate]
\label{cor:stability-convex}
Assume the loss function $f$ is convex and assumptions~\ref{ass:lipschitz}--\ref{ass:smooth} hold.

For a constant learning rate $\gamma_t=\gamma$ satisfying $\gamma\leq 2/L$, we have:
\begin{equation}
\begin{aligned}
\epsilon_{\mathrm{stab}} 
\leq \frac{2CGL\gamma C_{w_0}}{\delta (1-\lambda)}
+\left(\frac{2CG^2L\gamma^2}{\delta(1-\lambda)}+ \frac{2G^2\gamma}{mn}\right)T.
\end{aligned}
\end{equation}

For a diminishing learning rate $\gamma_t=\frac{v}{t+1}$ satisfying $v\leq 2/L$, we have:
\begin{equation}
\begin{aligned}
\epsilon_{\mathrm{stab}}
\leq \frac{2G^2 v}{mn}\ln T
+ \frac{2vCGL C_{w_0} + 4CG^2L v^2}{\delta (1-\lambda)}
+ \frac{2G^2 v}{mn}.
\end{aligned}
\end{equation}
\end{corollary}

\begin{proof}
See Appendix~\ref{pro:stability-convex-cor} for proofs.
\end{proof}

\begin{remark}
Corollary~\ref{cor:stability-convex} makes explicit how the learning-rate schedule determines the long-horizon behavior of stability.
\textbf{(i) Growth rate under different schedules.}
With a constant step size $\gamma$, the cumulative stochastic perturbation accumulates linearly, leading to $\epsilon_{\mathrm{stab}} = \mathcal{O}(T)$. 
In contrast, the diminishing schedule $\gamma_t = v/(t+1)$ suppresses this accumulation and yields the milder growth rate $\mathcal{O}(\ln T)$. 
This distinction reflects the classical bias--variance trade-off: constant step sizes preserve optimization speed but incur persistent instability, whereas decaying schedules gradually attenuate sensitivity to data perturbations. 
Moreover, these rates are known to be essentially unimprovable in general convex stochastic optimization, as matching lower bounds exist even in the centralized setting.
\textbf{(ii) Network-dependent degradation.}
The bounds further separate optimization effects from communication-induced penalties. 
When the graph is balanced, the dependence on $(1-\lambda)^{-1}$ matches the asymptotic scaling of undirected D-SGD~\citep{sun-2021}. 
For general directed graphs, however, the additional factor $1/\delta$ quantifies the imbalance of the stationary distribution and leads to a strictly larger stability constant. 
Hence, the generalization gap between SGP and its undirected counterpart is entirely attributable to the asymmetry of information flow.
\end{remark}

\begin{theorem}[Optimization Error]
\label{thm:optimal-convex}
Assume the loss function $f$ is convex and assumptions~\ref{ass:lipschitz}--~\ref{ass:bounded} hold. 
Then SGP satisfies:
\begin{equation}
\begin{aligned}
 \epsilon_{\mathrm{opt}}
 &\leq
 \frac{\|\overline{\bm{w}}^{(0)}-\bm{w}_{\bm{S}}^*\|^2}
      {2\sum_{t=0}^{T-1}\gamma_t}
 +\frac{2rCL C_{w_0}}
      {\delta \sum_{t=0}^{T-1}\gamma_t}
      \sum_{t=0}^{T-1}\gamma_t {\lambda}^{t} 
 +\left(\frac{2rCLG}{\delta (1-{\lambda})}
 +\frac{G^2}{2}\right)
 \frac{\sum_{t=0}^{T-1}\gamma_t^2}
      {\sum_{t=0}^{T-1}\gamma_t}.
 \end{aligned}
\end{equation}
\end{theorem}

\begin{proof}
Detailed proofs can be found in Appendix~\ref{pro:optimal-convex}.
\end{proof}

\begin{corollary}[Optimization Error on Common Learning Rate]
\label{cor:opt-convex}
When the loss $f$ is convex and assumptions~\ref{ass:lipschitz}--~\ref{ass:bounded} hold, 
for constant learning rate $\gamma_t=\gamma$ satisfying $\gamma\leq2/L$, SGP satisfies:
\begin{equation}
\begin{aligned}
\epsilon_{\mathrm{opt}} &\leq\left(\frac{\|\overline{\bm{w}}^{(0)}-\bm{w}_{\bm{S}}^*\|^2}{2\gamma}+\frac{2rCL C_{w_0}}{\delta (1-{\lambda})}\right)\frac{1}{T}+\frac{2CGrL\gamma}{\delta (1-{\lambda})}+\frac{G^2\gamma}{2}.
\end{aligned}
\end{equation}
For diminishing learning rate $\gamma_t=\frac{v}{t+1}$ satisfying $v\leq2/L$, SGP satisfies:
\begin{equation}
\begin{aligned}
\epsilon_{\mathrm{opt}} &\leq \left(\frac{\|\overline{\bm{w}}^{(0)} - \bm{w}_{\bm{S}}^*\|^2}{v } 
    + \frac{4rCL C_{w_0}}{\delta (1-{\lambda})}  + \frac{8vrCG L}{\delta (1-{\lambda})} + 2vG^2\right)\frac{1}{{\ln{T}}}.
\end{aligned}
\end{equation}
\end{corollary}

\begin{proof}
Detailed proofs can be found in Appendix~\ref{pro:opt-convex}.
\end{proof}

\begin{remark}
Theorem~\ref{thm:optimal-convex} and Corollary~\ref{cor:opt-convex} characterize the optimization behavior of SGP in convex settings.
\textbf{(i) Step-size dependence.}
Under constant step size, SGP achieves the classical $\mathcal{O}(1/T)$ convergence rate up to a residual term proportional to $\gamma$, while diminishing step sizes yield the milder $\mathcal{O}(1/\ln T)$ rate. 
These rates are consistent with standard stochastic approximation results and match the order of centralized SGD in convex problems \citep{hardt-2016}. 
\textbf{(ii) Consistency with undirected decentralized optimization.}
When the communication graph is balanced, the dependence on $(1-\lambda)^{-1}$ coincides with the spectral-gap dependence appearing in undirected D-SGD analyses \citep{lian-2017}. 
In this regime, the bound reduces to the classical decentralized optimization rate up to constants, indicating that SGP preserves the asymptotic order established for symmetric networks.
\textbf{(iii) Directed-network specific effects.}
For general directed graphs, the optimization error exhibits additional amplification through $1/\delta$, reflecting imbalance in the stationary distribution as in push-sum based methods \citep{nedic-2016}. 
Moreover, the exponentially weighted term $\sum_t \gamma_t \lambda^t$ explicitly captures transient consensus bias induced by asymmetric information flow, which is typically implicit in existing directed analyses, including SGP \citep{assran-2019}. 
Thus, any degradation relative to classical decentralized optimization stems from directed mixing effects rather than from the stochastic optimization mechanism.
\end{remark}

This trade-off between stability and optimization necessitates careful selection of the iteration steps $T$ via early stopping to minimize the excess generalization error.

\begin{corollary}
[Excess Generalization Error on Common Learning Rate]
\label{cor:exc-convex}
When the loss function $f$ is convex and Assumptions \ref{ass:lipschitz}--\ref{ass:bounded} hold, SGP satisfies:

For a constant learning rate $\gamma_t=\gamma$ with $\gamma\leq 2/L$, there exists an early-stopping time
\begin{equation}
T^{\star}
=
\tilde{\Theta}\!\left(
\frac{\sqrt{mn}}{\gamma}\,
\sqrt{\frac{\delta(1-\lambda)+\gamma}{\delta(1-\lambda)+mn\gamma}}
\right)
\approx
\tilde{\Theta}\!\left(
\frac{1}{\gamma}\sqrt{\frac{mn}{1+\frac{mn\gamma}{\delta(1-\lambda)}}}
\right),
\end{equation}
and the corresponding excess generalization error satisfies
\begin{equation}
\epsilon_{\mathrm{exc}}^{\star}
=
\tilde{\mathcal{O}}\!\left(
\frac{1}{\sqrt{mn}}
+\sqrt{\frac{\gamma}{\delta(1-\lambda)}}
+\frac{\gamma}{\delta(1-\lambda)}
+\gamma
\right).
\end{equation}

For a diminishing learning rate $\gamma_t=\frac{v}{t+1}$ with $v\leq 2/L$, there exists
\begin{equation}
T^{\star}
=
\tilde{\Theta}\!\left(
\exp\!\left(
\frac{\sqrt{mn}}{v}\sqrt{1+\frac{v^2}{\delta(1-\lambda)}}
\right)
\right),
\end{equation}
and at this time the excess generalization error scales as
\begin{equation}
\epsilon_{\mathrm{exc}}^{\star}
=
\tilde{\mathcal{O}}\!\left(
\frac{1}{\sqrt{mn}}
+\frac{v}{\delta(1-\lambda)}
+\frac{v}{\sqrt{mn\,\delta(1-\lambda)}}
+v^2
\right).
\end{equation}
Here $\tilde{\Theta}(\cdot)$ and $\tilde{\mathcal{O}}(\cdot)$ suppress logarithmic factors and universal constants depending on $G,L,C,r,C_{w_0}$ (and initialization-dependent constants).
\end{corollary}

\begin{proof}
See Appendix~\ref{pro:exc-convex} for detailed proofs.
\end{proof}

\begin{remark}
Corollary~\ref{cor:exc-convex} highlights statistical and network-induced effects in the excess generalization error.
\textbf{(i) Statistical term.}
In both step-size regimes, the dominant term $\tilde{\mathcal{O}}(1/\sqrt{mn})$ coincides with the optimal rate of centralized SGD using the total sample size $mn$. This shows that, up to logarithmic factors, SGP preserves the statistical efficiency of centralized learning and achieves linear speedup with respect to the total data volume, despite operating over an asymmetric communication topology.
\textbf{(ii) Topological bias.}
In addition to the statistical term, the bound contains network-dependent contributions proportional to $\frac{\gamma}{\delta(1-\lambda)}$ (or $\frac{v}{\delta(1-\lambda)}$ in the diminishing step-size case). These terms reflect a consensus bias induced by imperfect mixing and asymmetry in the communication graph. Unlike the statistical error, they do not vanish with increasing sample size $n$, and instead depend explicitly on the imbalance parameter $\delta$ and the spectral gap $1-\lambda$. Consequently, the achievable accuracy is fundamentally constrained by the network topology: poorly connected or highly imbalanced graphs lead to a non-negligible residual error even when $mn$ is large.
\end{remark}

\subsection{Results on Non-convex Case}
\label{sec:non-con}

Since optimization problems in ML are usually non-convex, such as deep neural network training, analyzing non-convex settings is inherently more complicated but necessary.

\begin{theorem}
\label{thm:stability-nonconvex}
[Uniform Stability]
Assume the loss function $f$ is non-convex and assumptions~\ref{ass:lipschitz}--~\ref{ass:smooth} hold, $t_0 \in \{0, 1, . . . , n\}$ assuming the first time step in which SGP chooses the different example.
Then, the uniform stability of SGP satisfies:
\small
{
\begin{equation}
\begin{aligned}
\epsilon_{\mathrm{stab}} \leq{} 
&\min_{t_{0}}\Bigg\{ \frac{t_0}{mn}
+ G\sum_{t=t_0}^{T-1} \prod_{k=t+1}^{T-1}\left(1+L \gamma_k - \frac{L \gamma_k}{mn}\right) \times\left(
  \frac{2CL\gamma_t {\lambda}^t}{\delta m} C_{w_0}
  + \frac{2GCL{\gamma_t}^2}{\delta(1-{\lambda})}
  + \frac{2G\gamma_t}{mn} 
\right)\Bigg\}.
\end{aligned}
\end{equation}
}
\end{theorem}

\textbf{Proof Sketch}
We analyze the uniform stability of SGP by bounding the divergence between the two consensus trajectories 
$\overline{\bm w}^{(t)}$ and $\overline{\bm w}^{\prime(t)}$ generated on neighboring datasets.
Let $\Delta_t := \|\overline{\bm w}^{(t)}-\overline{\bm w}^{\prime(t)}\|$.
Unlike the convex case, the gradient step is no longer non-expansive.
Lemma~\ref{lem:hardt} implies that in the non-convex setting the averaged update may expand distances by a factor
$(1+L\gamma_t)$.
At iteration $t$, separating the gradient difference at the consensus model from the consensus errors yields
\[
\Delta_{t+1}
\le
(1+L\gamma_t)\Delta_t
+\frac{L\gamma_t}{m}\sum_{i=1}^m
\|\overline{\bm w}^{(t)}-\bm z_i^{(t)}\|
+\frac{L\gamma_t}{m}\sum_{i=1}^m
\|\overline{\bm w}^{\prime(t)}-\bm z_i^{\prime(t)}\|
+\frac{2G\gamma_t}{mn},
\]
where the last term accounts for the possible gradient mismatch when the differing sample.

The consensus error terms are controlled by Lemma~\ref{lem:push_sum_consistency}, which shows that for every $t$
\[
\frac{1}{m}\sum_{i=1}^m
\|\overline{\bm w}^{(t)}-\bm z_i^{(t)}\|
\lesssim
\frac{C}{\delta}
\Big(
\lambda^t C_{w_0}
+
G\sum_{s=0}^{t-1}\lambda^{t-s}\gamma_s
\Big).
\]
Substituting this bound gives a one-step recursion of the form
\[
\mathbb E[\Delta_{t+1}]
\le
\bigl(1+L\gamma_t\bigr)\mathbb E[\Delta_t]
+
\frac{2CL\gamma_t}{\delta}
\Big(
\lambda^t C_{w_0}
+
G\sum_{s=0}^{t-1}\lambda^{t-s}\gamma_s
\Big)
+
\frac{2G\gamma_t}{mn}.
\]

To handle the multiplicative expansion, we condition on the first time $t_0$ at which the two runs decouple.
Since $\Delta_{t_0}=0$, unrolling the above recursion from $t_0$ to $T-1$ shows that the accumulated perturbations are weighted by the expansion factors $\prod_{k}(1+L\gamma_k)$.

Finally, removing the conditioning introduces the standard decoupling term $t_0/(mn)$,
\[
\epsilon_{\mathrm{stab}}
\le
\frac{t_0}{mn}
+
G\,\mathbb E[\Delta_T\mid \Delta_{t_0}=0],
\]
where the second term is controlled by the accumulated network and stochastic errors.
Choosing $t_0$ balances the decoupling probability and the subsequent expansion, leading to the stated bound. Detailed proof is located in Appendix~\ref{pro:stability-nonconvex}.

\begin{remark}
Theorem~\ref{thm:stability-nonconvex} characterizes the sensitivity of SGP under non-convex objectives.
\textbf{(i) Amplification in non-convex case.}
The product term $\prod_{k=t+1}^{T-1} \left(1 + L \gamma_k - \frac{L \gamma_k}{mn}\right)$
acts as an amplification factor that propagates perturbations introduced after time $t_0$. 
In contrast to the convex case, where contractive behavior can be established through monotonic descent toward a minimizer, non-convex dynamics do not generally admit such contraction. 
As a result, small discrepancies may accumulate multiplicatively across iterations. 
The bound makes this effect explicit and quantifies how stability depends on the expansivity induced by the step-size schedule.
\textbf{(ii) Interaction with directed topology.}
The network-dependent factors $\frac{1}{\delta}$ and $\frac{1}{1-\lambda}$ appear inside the amplified summation, indicating that communication imbalance and slow mixing increase the magnitude of propagated perturbations. 
Unlike in convex settings where topology contributes additively to the stability bound, here it interacts with the amplification mechanism, thereby influencing stability through the trajectory. 
This highlights that directed mixing properties play a more pronounced role in non-convex stability.
\end{remark}

 A well-designed step size can limit the rapid growth of SGP’s generalization error. The following corollary considers two common learning rate schedules:

\begin{corollary}[Uniform Stability on Common Learning Rate]
\label{cor:stability-nonconvex}
When $f$ is non-convex and Assumptions~\ref{ass:lipschitz}--\ref{ass:smooth} hold, for a constant learning rate $\gamma_t=\gamma$, we have
\begin{equation}
\begin{aligned}
\epsilon_{\mathrm{stab}}
\leq
\left(\frac{2GCL\gamma C_{w_0}}{\delta (1-\lambda)}
+\frac{4CG^2L\gamma^2}{\delta(1-\lambda)}
+\frac{4G^2\gamma}{mn}\right)
\left(1+L \gamma-\frac{L \gamma}{mn}\right)^{T}.
\end{aligned}
\end{equation}
For a diminishing learning rate $\gamma_t=\frac{v}{t+1}$, we have
\begin{equation}
\begin{aligned}
\epsilon_{\mathrm{stab}}
\leq
\frac{4CG\,v^{\frac{1}{2+vL}}\,C_{w_0}}{\delta}\,T^{\frac{1+vL}{2+vL}}
+\frac{v^{\frac{1}{2+vL}}+4G^2}{mn}\,T^{\frac{1+vL}{2+vL}}
+\frac{2CG^2\,v^{\frac{1}{2+vL}}}{\delta(1-\lambda)}\,T^{\frac{vL}{2+vL}}.
\end{aligned}
\end{equation}
\end{corollary}

\begin{proof}
See Appendix \ref{pro:stability-nonconvex-cor} for detailed proof.
\end{proof}

\begin{remark}
Corollary~\ref{cor:stability-nonconvex} clarifies how the learning-rate schedule governs stability in non-convex decentralized optimization and contrasts with the convex case.
\textbf{(i) Contrast with convex stability.}
In convex settings, stability grows at most linearly or logarithmically in $T$, depending on the step-size schedule, due to the contractive structure of the objective. 
By contrast, under a constant step size $\gamma$, the non-convex bound scales as 
$\left(1+L\gamma-\frac{L\gamma}{mn}\right)^T \approx e^{L\gamma T}$,
indicating exponential amplification of perturbations. 
This difference reflects the absence of global contraction in non-convex dynamics, where gradient updates may expand discrepancies rather than attenuate them.
\textbf{(ii) Effect of diminishing step sizes.}
Adopting a diminishing schedule $\gamma_t \propto 1/t$ reduces the exponential growth to a polynomial rate $\mathcal{O}(T^p)$, where $p=\frac{1+vL}{2+vL}<1$. 
Thus, instability accumulates sublinearly in $T$, and sensitivity to early perturbations decreases as the step size vanishes. 
Such behavior aligns with classical non-convex SGD analyses, where diminishing learning rates are required to control long-term variance.
\textbf{(iii) Comparison with existing decentralized analyses.}
While exponential instability under constant step size is also observed in centralized non-convex SGD, the bound further reveals the interaction with directed topology through the factors $\frac{1}{\delta}$ and $\frac{1}{1-\lambda}$. 
Unlike convex decentralized results, where topology contributes additively to the stability constant, here it appears within the amplified term, magnifying perturbations throughout the trajectory. 
This distinguishes directed SGP from centralized and undirected settings and shows that topological imbalance directly affects non-convex stability rates.
\end{remark}

The Polyak–Łojasiewicz (P\L{}) condition has been widely used in non-convex optimization~\citep{deng-2023}, as it ensures convergence in function value to the global minimum whereas general non-convex settings only guarantee convergence to a stationary point. 

\begin{definition}[PŁ - Condition]
\label{def:PL}
Let $\bm{w}^* = \operatorname*{argmin}_{\bm{w}} f(\bm{w})$. A function $f$ satisfies the PŁ - Condition with parameter $\alpha > 0$ if for all $\bm{w}$,
\[
2\alpha [f(\bm{w}) - f(\bm{w}^*)] \leq \|\nabla f(\bm{w})\|^2.
\]
When $f$ is additionally $L$-smooth, the ratio $\kappa := L/\alpha$ is called the \textit{condition number} of $f$.
\end{definition}

\begin{remark}
\textbf{(i) Applicability of PŁ - Condition.} The PŁ condition is a standard relaxation of strong convexity that guarantees linear convergence of gradient descent and its stochastic variants. It holds for many practical non-convex objectives, including over-parameterized neural networks (especially with ReLU activations) and certain matrix factorization and phase retrieval problems~\citep{song-2021,chen-2023b}.
\textbf{(ii) Condition Number.} The quantity $\kappa = L/\alpha$ serves as an effective condition number, analogous to the ratio of smoothness to strong convexity parameters in the convex setting~\citep{karimi2016linear}. Larger $\kappa$ reflects greater ill-conditioning: slower function decay away from the minimum and extended flat regions. Under PŁ, gradient methods typically converge linearly at rate $1 - O(1/\kappa)$, yielding iteration complexity $O(\kappa \log(1/\varepsilon))$ to reach $\varepsilon$-accuracy. Thus, $\kappa$ directly controls convergence speed and quantifies landscape difficulty.
\end{remark}

We can derive the optimization error in the non-convex case:

\begin{theorem}[Optimization Error]
\label{thm:opt-nonconvex}
Assume $f$ is non-convex and satisfies the P\L{} condition, and assumptions~\ref{ass:lipschitz}--\ref{ass:bounded} hold.
Then the optimization error of SGP satisfies:
\begin{equation}
\begin{aligned}
\epsilon_{\mathrm{opt}}
&\leq
\frac{Gr}{\alpha\sum_{t=0}^{T-1}\gamma_t}
+\frac{CG\kappa C_{w_0}}{2\delta\,\sum_{t=0}^{T-1}\gamma_t}\sum_{t=0}^{T-1}\gamma_t\lambda^t
+\left(\frac{CG^2\kappa}{2\delta(1-\lambda)}+\frac{G^2\kappa}{4}\right)
\frac{\sum_{t=0}^{T-1}\gamma_t^2}{\sum_{t=0}^{T-1}\gamma_t}.
\end{aligned}
\end{equation}
\end{theorem}

\begin{proof}
See Appendix \ref{pro:optimal-nonconvex} for detailed proofs.
\end{proof}

\begin{corollary}[Optimization Error on Common Learning Rate]
\label{cor:opt-nonconvex}
 When the loss function $f$ is non-convex and satisfies the PŁ-condition, and assumptions \ref{ass:lipschitz}--\ref{ass:bounded} hold, for a constant learning rate $\gamma_t=\gamma$, we have:
\begin{equation}
\begin{aligned}
\epsilon_{\mathrm{opt}} &\leq \left(\frac{Gr}{\alpha\gamma}+\frac{CG\kappa C_{w_0}}{2\delta(1-{\lambda})}\right)\frac{1}{T}
+\left(\frac{CG^2\kappa}{2\delta (1-{\lambda})}+\frac{G^2\kappa}{4}\right)\gamma.
\end{aligned}
\end{equation}
For a diminishing learning rate $\gamma_t=\frac{v}{t+1}$, we have:
\begin{equation}
\begin{aligned}
\epsilon_{\mathrm{opt}} &\leq \left(\frac{CG\kappa C_{w_0}+2vCG^2L \kappa}{\delta  (1-{\lambda}) }
+\frac{2rG+v^2G^2L}{v\alpha}\right)\frac{1}{\ln{T}}.
\end{aligned}
\end{equation}
\end{corollary}
\begin{proof}
See Appendix \ref{pro:opt-nonconvex} for detailed proofs.
\end{proof}

\begin{remark}
Theorem~\ref{thm:opt-nonconvex} and Corollary~\ref{cor:opt-nonconvex} characterize the convergence behavior of SGP under the Polyak--\L{}ojasiewicz (P\L{}) condition.
\textbf{(i) Topology-dependent error floor.}
In the constant step-size regime, Corollary~\ref{cor:opt-nonconvex} shows that the stationary error is dominated by $\kappa/(\delta(1-\lambda))$, revealing a coupling between problem conditioning and directed network topology. 
Here, $\kappa=L/\alpha$ measures the conditioning of the objective, while $1/(\delta(1-\lambda))$ captures the degradation induced by imbalance and slow mixing. 
Thus, directed communication can amplify the optimization difficulty through the convergence constants.
\textbf{(ii) Convex-like rates under P\L{}.}
Under the P\L{} condition, SGP attains $\mathcal{O}(1/T)$ convergence with a constant step size and $\mathcal{O}(1/\ln T)$ with a diminishing step size, matching the rates obtained in convex decentralized optimization. 
This indicates that the P\L{} condition yields convex-like optimization behavior, although the constants remain more sensitive to directed settings.
\end{remark}

\begin{corollary}[Excess Generalization Error with Common Learning Rates]
\label{cor:exc-nonconvex}

Suppose the loss function $f$ is non-convex and satisfies the PŁ condition, and that assumptions~\ref{ass:lipschitz}--\ref{ass:bounded} hold.

For a constant learning rate $\gamma_t = \gamma$, there exists an early-stopping time $T^{\star}$ such that
\begin{equation}
T^{\star}
= \tilde{\Theta}\left(
    \frac{mn}{L\,\gamma\,(mn-1)}
    \log\frac{1}{\gamma}
\right),
\end{equation}
and the corresponding excess generalization error satisfies
\begin{equation}
\epsilon_{\mathrm{exc}}^{\star}
= \tilde{\mathcal{O}}\left(
    \kappa
    G^2
    \left(1+\frac{1}{\delta(1-\lambda)}\right)
    \left(1+\frac{1}{mn}\right)
    \gamma
\right).
\end{equation}

For a diminishing learning rate $\gamma_t = \frac{v}{t+1}$, there exists an early-stopping time of the form
\begin{equation}
T^{\star}
= \tilde{\Theta}\left(
    \exp\left(
        \Theta\left(\frac{mn}{(mn-1)\,v\,L}\right)
    \right)
\right),
\end{equation}
and at this time the excess generalization error scales as
\begin{equation}
\epsilon_{\mathrm{exc}}^{\star}
= \tilde{\mathcal{O}}\left(
    \frac{\kappa(mn-1)}{mn}
    \left[
        rG
        + \frac{CGL C_{w_0}}{\delta(1-\lambda)}
        + vG^2\left(1+\frac{CL}{\delta(1-\lambda)}\right)
    \right]
\right),
\end{equation}
where the notation $\tilde{\Theta}(\cdot)$ and $\tilde{\mathcal{O}}(\cdot)$ suppresses logarithmic factors in $T$ and $1/\gamma$ as well as universal constants.
\end{corollary}

\begin{proof}
See Appendix~\ref{pro:exc-nonconvex} for proofs.
\end{proof}

\begin{remark}
Corollary~\ref{cor:exc-nonconvex} clarifies how the PŁ condition shapes the generalization behavior of SGP in the non-convex regime.

\textbf{(i) Diminishing step sizes.}
Under a constant step size, the optimal stopping time depends only logarithmically on $1/\gamma$, resulting in a comparatively limited stability window. In contrast, a diminishing schedule of the form $\gamma_t=v/(t+1)$ leads to an exponentially large optimal stopping horizon,
$T^{\star}\approx \exp\!\bigl(\Theta(1/v)\bigr),$
thereby enlarging the range of iterations over which the excess risk remains controlled. This suggests that diminishing step sizes provide greater tolerance to longer training trajectories.
\textbf{(ii) Geometric--topological coupling.}
The excess risk bound exhibits a multiplicative structure consistent with optimization under the PŁ inequality. In particular, the dominant dependence can be summarized as
$\epsilon_{\mathrm{exc}}^{\star}
\propto
\kappa
\left(1+\frac{1}{\delta(1-\lambda)}\right)
\times \mathcal{E}_{\mathrm{noise}},$
indicating that accuracy is jointly governed by the problem conditioning (via $\kappa=L/\alpha$) and the network topology (via $\delta$ and the spectral gap $1-\lambda$). Hence, weak connectivity or imbalance in the communication graph effectively amplifies the conditioning, and consequently impacts both optimization and generalization.
\end{remark}

\subsection{Discussion on the Impact of Topology}
\label{sec:discussion-topology}

\begin{table}[!t]
\renewcommand{\arraystretch}{1.2}
\setlength{\tabcolsep}{2pt}
\caption{Comparison of spectral properties ($\lambda$) and topological imbalance ($\delta$) across graph topologies with $m$ nodes. $C_{\lambda} \triangleq 1/(1 - \lambda)$ is the mixing time complexity.}
\label{tab:topology_properties}
\centering
\begin{tabular}{lccc}
\toprule
\textbf{Topology} & \textbf{Spectral Gap} & \textbf{Mixing Cost} & \textbf{Imbalance} \\
 & $(1-\lambda)$ & $C_{\lambda}$ & $\delta$ \\
\midrule
Fully Connected & $1$ & $\mathcal{O}(1)$ & $\mathcal{O}(1/m)$ \\
Di-Exp & $\mathcal{O}(1)$ & $\mathcal{O}(\log_2 m)$ & $\mathcal{O}(1/m)$ \\
Bipartite & $\mathcal{O}(1/m)$ & $\mathcal{O}(m)$ & $\mathcal{O}(1/m)$ \\
B-tree & $\mathcal{O}(1/m)$ & $\mathcal{O}(m)$ & $\mathcal{O}(1/2m)$ \\
Di-Ring & $\mathcal{O}(1/m^2)$ & $\mathcal{O}(m^2)$ & $\mathcal{O}(1/m)$ \\
Sub-Ring & $\mathcal{O}(1/m^2)$ & $\mathcal{O}(m^2)$ & $\mathcal{O}(1)^{\dagger}$ \\
Star Graph & $1/2$ & $2$ & $\mathcal{O}(1)^{\dagger}$ \\
\bottomrule
\end{tabular}
\begin{tablenotes}
\footnotesize
\item $^{\dagger}$The value of $\delta$ for Star Graph depends on the weight matrix; here we assume a standard setup.
\end{tablenotes}
\end{table}

Our theoretical analysis shows that the stability and optimization performance of SGP are governed by a unified scaling factor that depends on two key topological aspects of the communication graph: structural imbalance and network connectivity.

Structural imbalance is quantified by the parameter $\delta$, which reflects the influence of the stationary distribution of the communication process. When $\delta_i$ is small, nodes have limited outgoing connectivity relative to incoming links. Although such nodes receive information from the network, their local gradient information contributes weakly to the global aggregation. As a result, data and updates stored at these nodes are not propagated across the network. From an optimization perspective, this behavior is equivalent to a reduction in the effective sample size. Our results indicate that highly imbalanced topologies lead to increased generalization error, highlighting the importance of maintaining structural balance so that all nodes exert influence on the learning dynamics.

Network connectivity is primarily captured by the spectral gap $1-\lambda$, whose inverse $1/(1-\lambda)$ characterizes the mixing time of the underlying Markov chain. This quantity determines how quickly local updates diffuse throughout the network. When the spectral gap is small, corresponding to $\lambda$ close to $1$, information propagation becomes slow. Although SGP achieves asymptotic convergence rates comparable to centralized methods, poor connectivity substantially increases the number of iterations required to reach this regime. As illustrated in Table~\ref{tab:topology_properties}, sparse topologies such as directed rings incur a quadratic mixing cost $\mathcal{O}(m^2)$, which limits their scalability. In contrast, exponential graphs maintain a logarithmic mixing cost $\mathcal{O}(\log m)$ and remain effective as the network size grows.

Taken together, an inherent trade-off exists in network design. Improving connectivity and maintaining structural balance are conflicting objectives. For example, star topologies mix rapidly but introduce imbalance due to the central node, whereas ring topologies are balanced but mix inefficiently. Since the topological factor enters the error bounds multiplicatively, unfavorable topology cannot be offset by increasing data size. Consequently, in large-scale SGP deployments, network designs that jointly improve connectivity while preserving reasonable balance are more effective. Bounded-degree exponential graphs provide such a compromise.

%% file: Latex/experiment.tex
\section{Experiment}
\label{sec:exp}
In this section, we demonstrate empirical studies to validate our theoretical analysis. We conduct general classification experiments on the classical Logistic Regression~\ref{sec:Logistic} and LeNet~\ref{sec:LeNet} respectively, and evaluate the impacts of the key factors to stability and optimization errors.

\subsection{Logistic Regression on a9a Dataset}
\label{sec:Logistic}
\begin{figure*}[t]
\centering
\begin{minipage}{\textwidth}
\centering
    \subfloat[$\epsilon_{\text{gen}}$ of Different LR.]{ 
 	\includegraphics[width=0.32\textwidth]{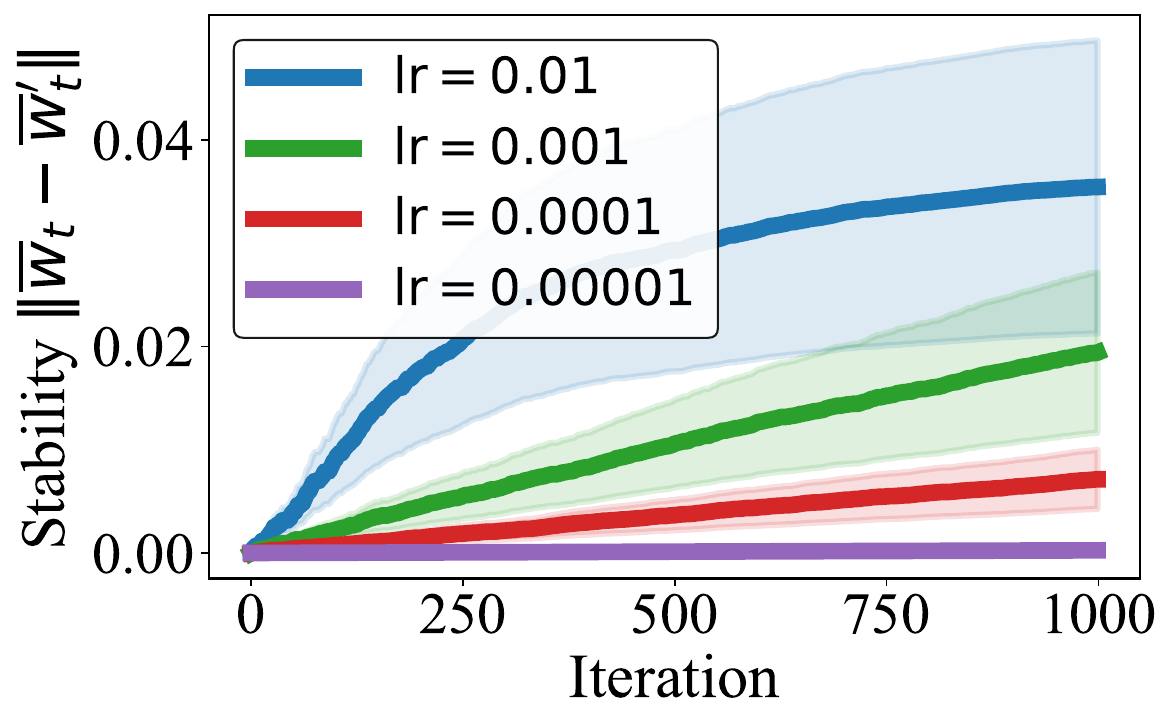}}\hfill
    \subfloat[$\epsilon_{\text{gen}}$ of Different Client Size.]{ 
	\includegraphics[width=0.32\textwidth]{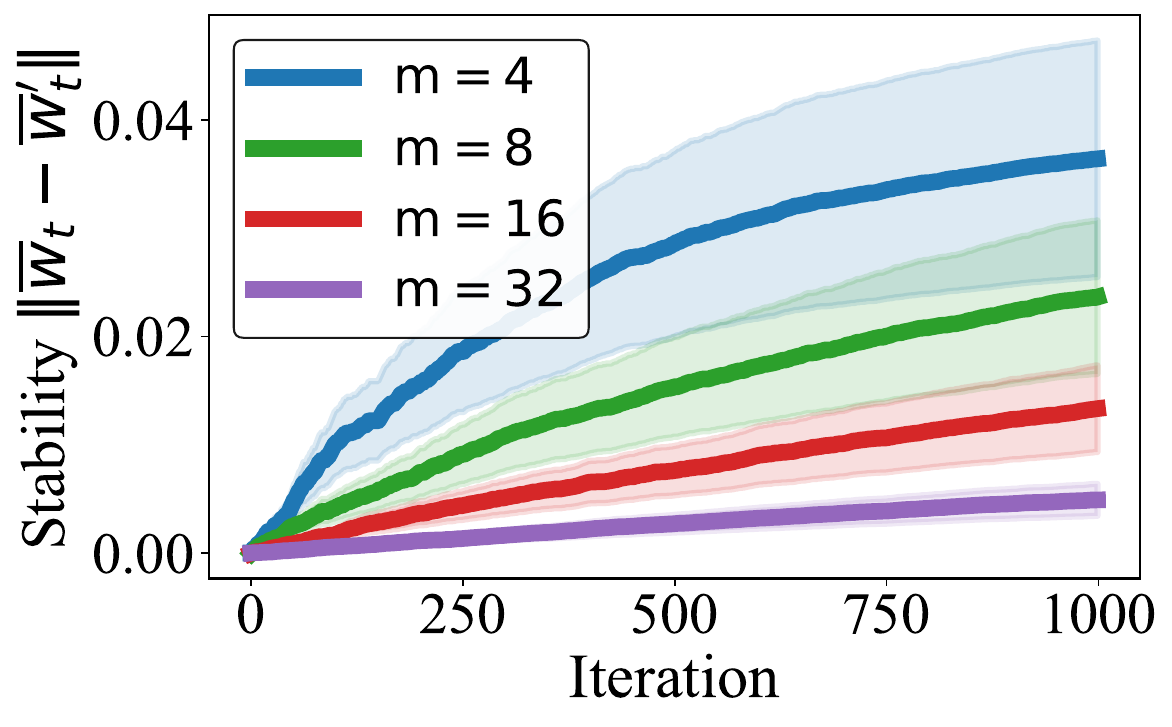}}\hfill
    \subfloat[$\epsilon_{\text{gen}}$ of Different Topology.]{
	\includegraphics[width=0.32\textwidth]{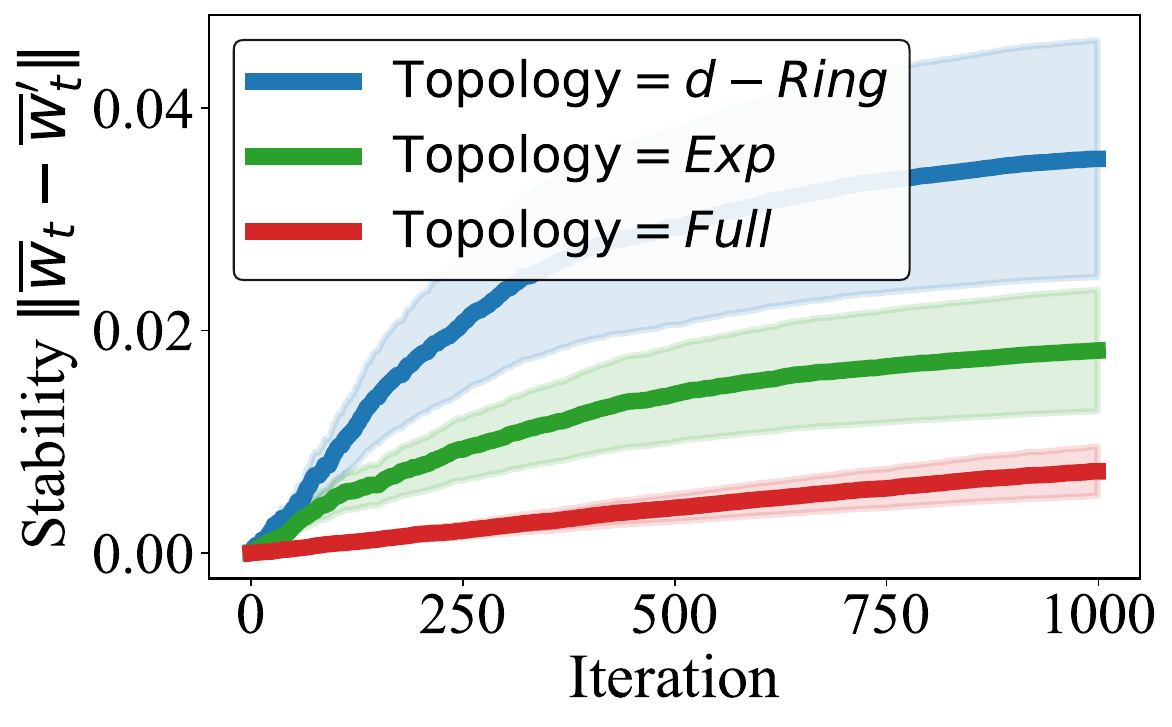}}
\end{minipage}
\begin{minipage}{\textwidth}
\centering
    \subfloat[$\epsilon_{\text{opt}}$ of Different LR.]{ 
	\includegraphics[width=0.32\textwidth]{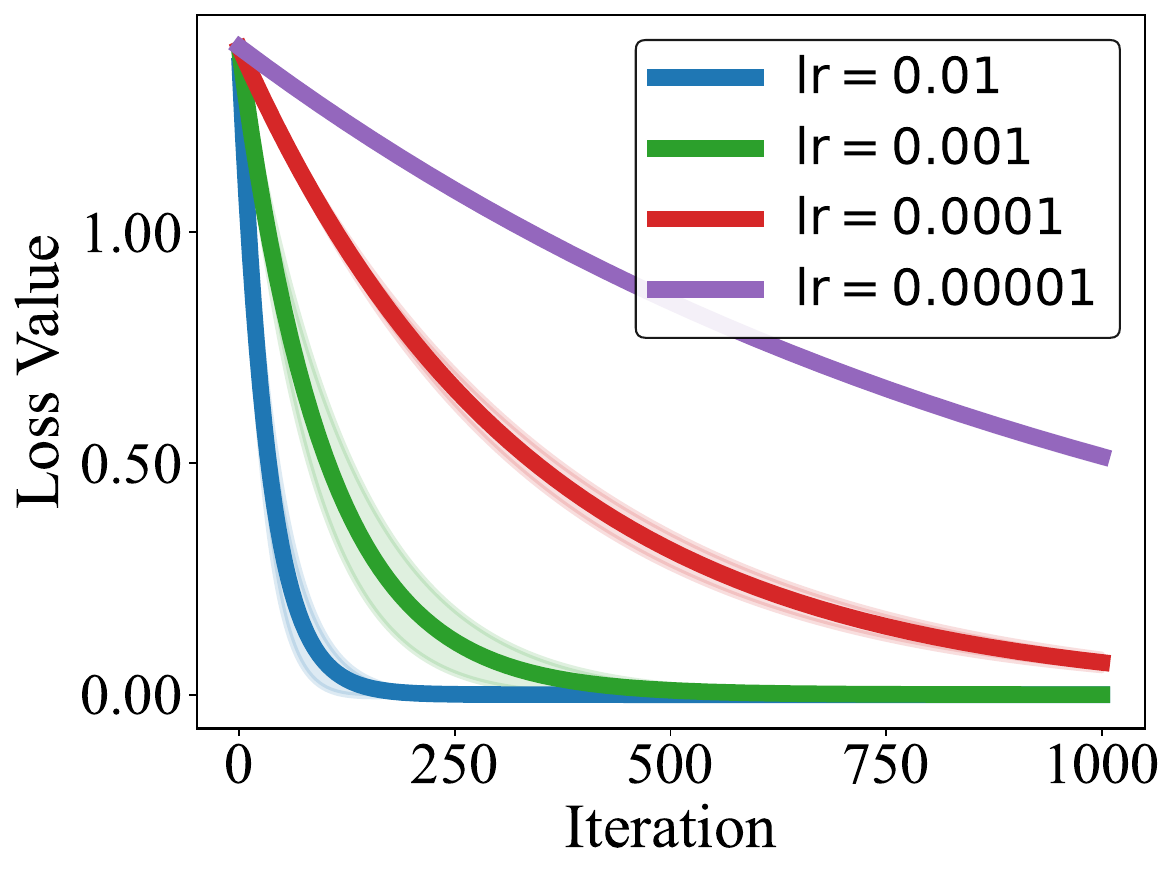}}\hfill
    \subfloat[$\epsilon_{\text{opt}}$ of Different Client Size.]{ 
	\includegraphics[width=0.32\textwidth]{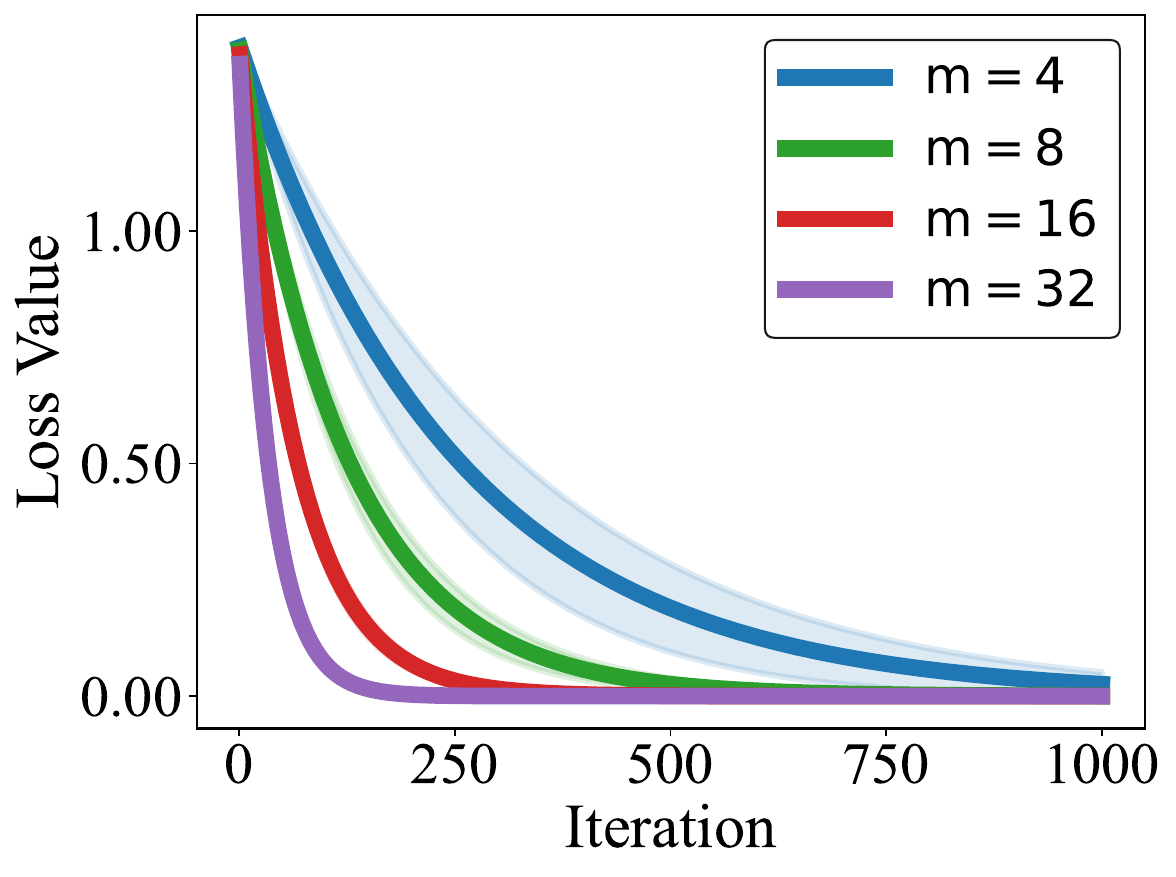}}\hfill
    \subfloat[$\epsilon_{\text{opt}}$ of Different Topology.]{ 
	\includegraphics[width=0.32\textwidth]{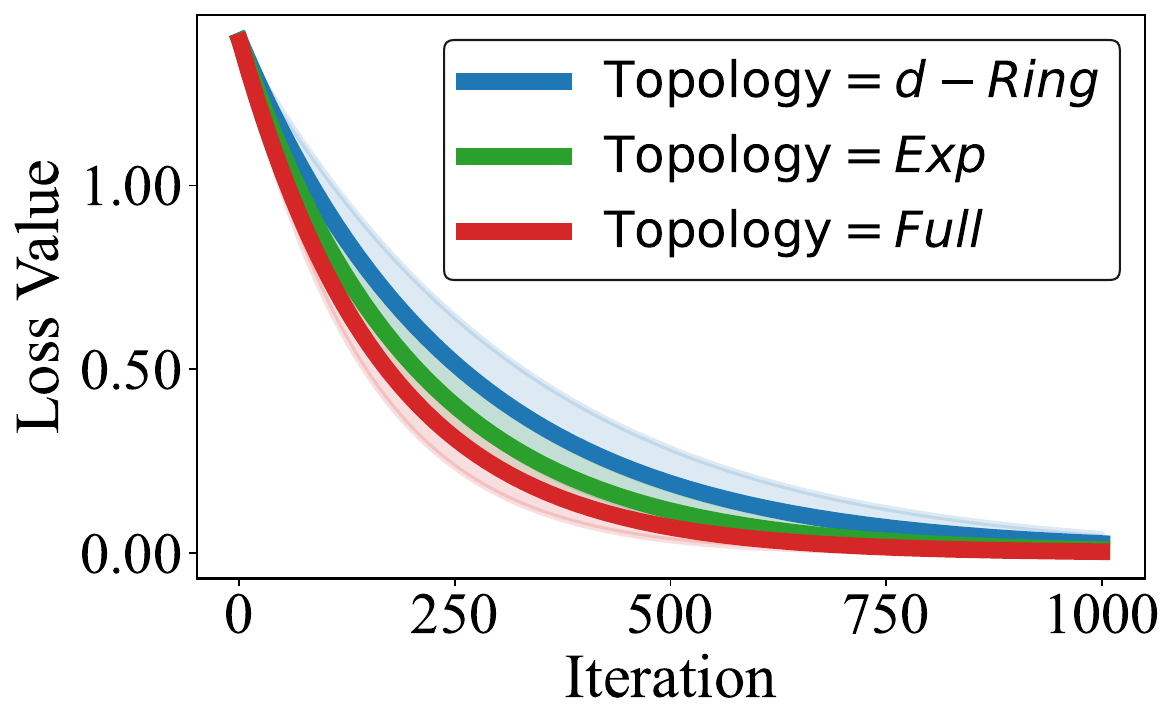}}
\end{minipage}
\vspace{-0.2cm} 
\caption{Impacts of generalization and optimization errors on convex objective.}
\vspace{-0.5cm} 
\label{exp:Logistic}
\end{figure*}

In the validation of convex objectives, we adopt the classical logistic regression problem~\citep{hosmer-2013} to evaluate generalization and optimization performance during training. 
We consider the following regularized loss:
$$
    f(x) = \frac{1}{n}\sum_{i=1}^{n}\log\left(1+\exp\left(-b_i a_i^\top x\right)\right) + \frac{\mu}{2}\Vert x\Vert^2,
$$
where $a_i\in\mathbb{R}^d$ and $b_i\in\{-1,+1\}$ are data samples and $n$ is the dataset size. 
Experiments are conducted on the a9a dataset~\citep{chang-2011} with $d=123$. 
We set the regularization parameter $\mu=10^{-4}$. For decentralized training, we randomly split 32k samples into 32 clients. 
To study the effect of constant learning rates, we select $\gamma \in \{0.01, 0.001, 0.0001, 0.00001\}$. 
To study the impact of client size, we vary $m \in \{4, 8, 16, 32\}$. 
To examine the role of topology, we consider several directed graphs, including $\textit{d-Ring}$, $\textit{Exp}$, and $\textit{Full}$.

The experimental results in Figure~\ref{exp:Logistic} are consistent with our theoretical analysis in the convex setting.
First, Figures~(a) and~(d) show a clear learning-rate effect: increasing $\gamma$ leads to a larger and faster-growing stability error
$\|\overline{\bm w}^{(t)}-\overline{\bm w}^{\prime(t)}\|$, while simultaneously accelerating optimization (the training loss decreases
more quickly within the same iteration budget). This behavior matches our bounds: the stability recursion accumulates perturbations
through step-size sums such as $\sum_t \gamma_t$ and $\sum_t \gamma_t^2$, so larger $\gamma$ amplifies the sensitivity to a single-sample
change, whereas the optimization bound improves when the step size is larger (up to the admissible range). The plots thus provide an
empirical illustration of the stability--optimization trade-off captured by our theory.
Second, Figures~(b) and~(e) indicate that increasing the client size $m$ reduces both the generalization (stability) error and the
optimization error. This trend is aligned with the role of the effective sample size $mn$ in our convex excess-risk analysis: the
probability that the two runs differ at an iteration scales as $1/n$, while the impact of that mismatch on the network average is
further diluted by the $1/m$ averaging across clients. As $m$ grows, stochastic fluctuations are reduced and the influence of the
single replaced sample becomes weaker, leading to uniformly smaller errors in the experiments.
Third, Figures~(c) and~(f) demonstrate a systematic dependence on the communication topology. Densely connected networks (e.g., Full)
achieve smaller stability errors and faster loss decrease than sparse topologies (e.g., Ring). This is exactly the dependence predicted
by our theory through the mixing parameters: faster mixing (larger spectral gap $1-\lambda$) and better balance (larger $\delta$)
shrink the consensus-induced error amplification, which appears in our bounds through factors of the form $1/(\delta(1-\lambda))$
and through the decaying-memory term $\sum_{s=0}^{t-1}\lambda^{t-s}\gamma_s$. Under slower-mixing or more imbalanced topologies, the
consensus error persists longer and is repeatedly injected into the gradient step, yielding more pronounced residual errors, the experimental trends with respect to $m$, and the topology are in qualitative agreement with the structure of
our convex bounds: larger step sizes improve optimization but worsen stability, larger client populations improve both, and better topological balance reduces the error contributions induced by decentralized communication.

\begin{figure*}[t]
\centering
\begin{minipage}[b]{\textwidth}
\centering
    \subfloat[$\epsilon_{\text{gen}}$ of Different LR.]{ 
	\includegraphics[width=0.32\textwidth]{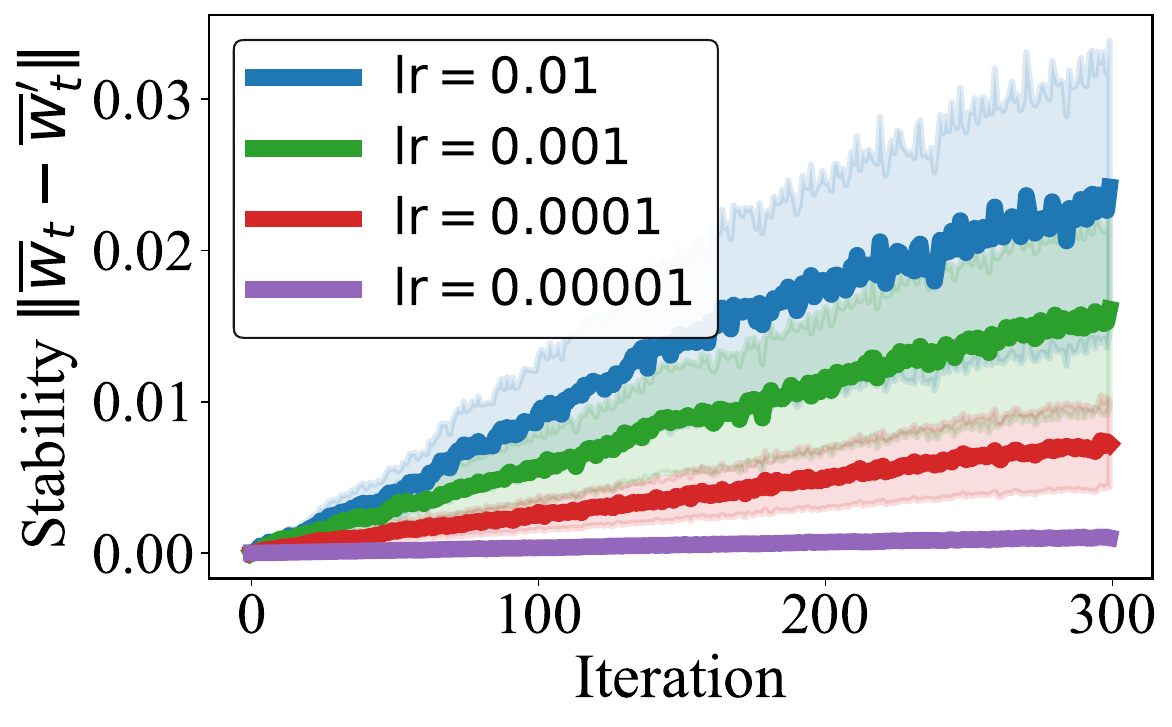}}\hfill
    \subfloat[$\epsilon_{\text{gen}}$ of Different Client Size.]{ 
	\includegraphics[width=0.32\textwidth]{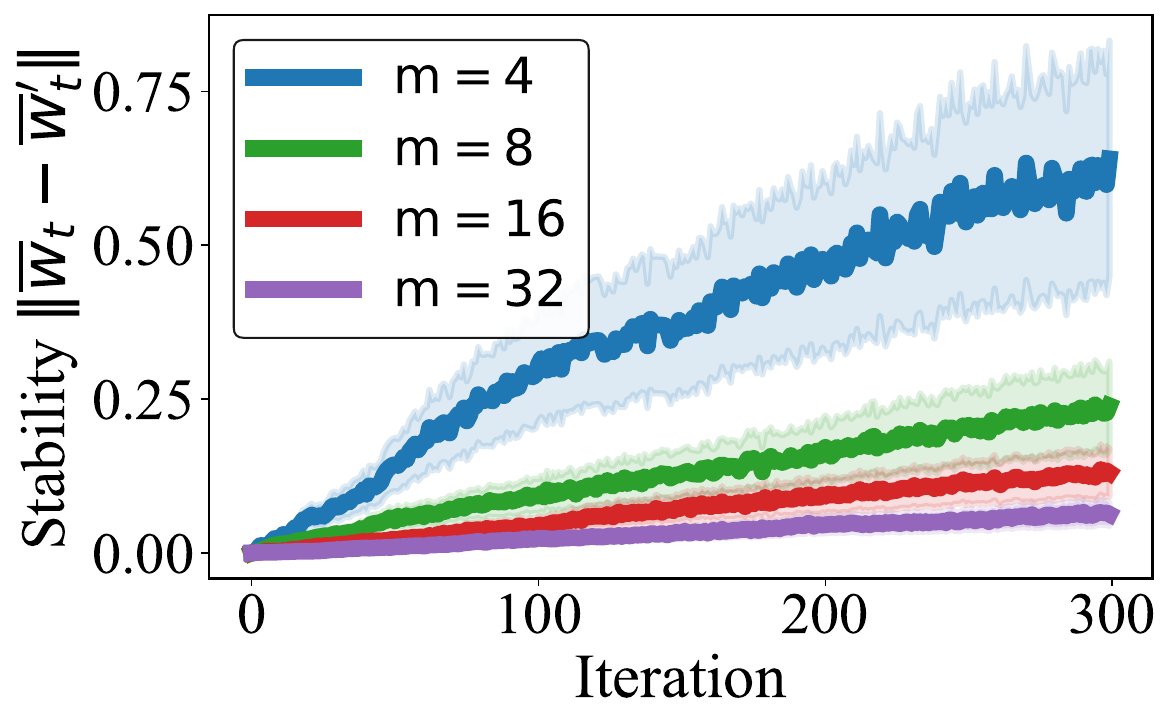}}\hfill
    \subfloat[$\epsilon_{\text{gen}}$ of Different Topology.]{ 
	\includegraphics[width=0.32\textwidth]{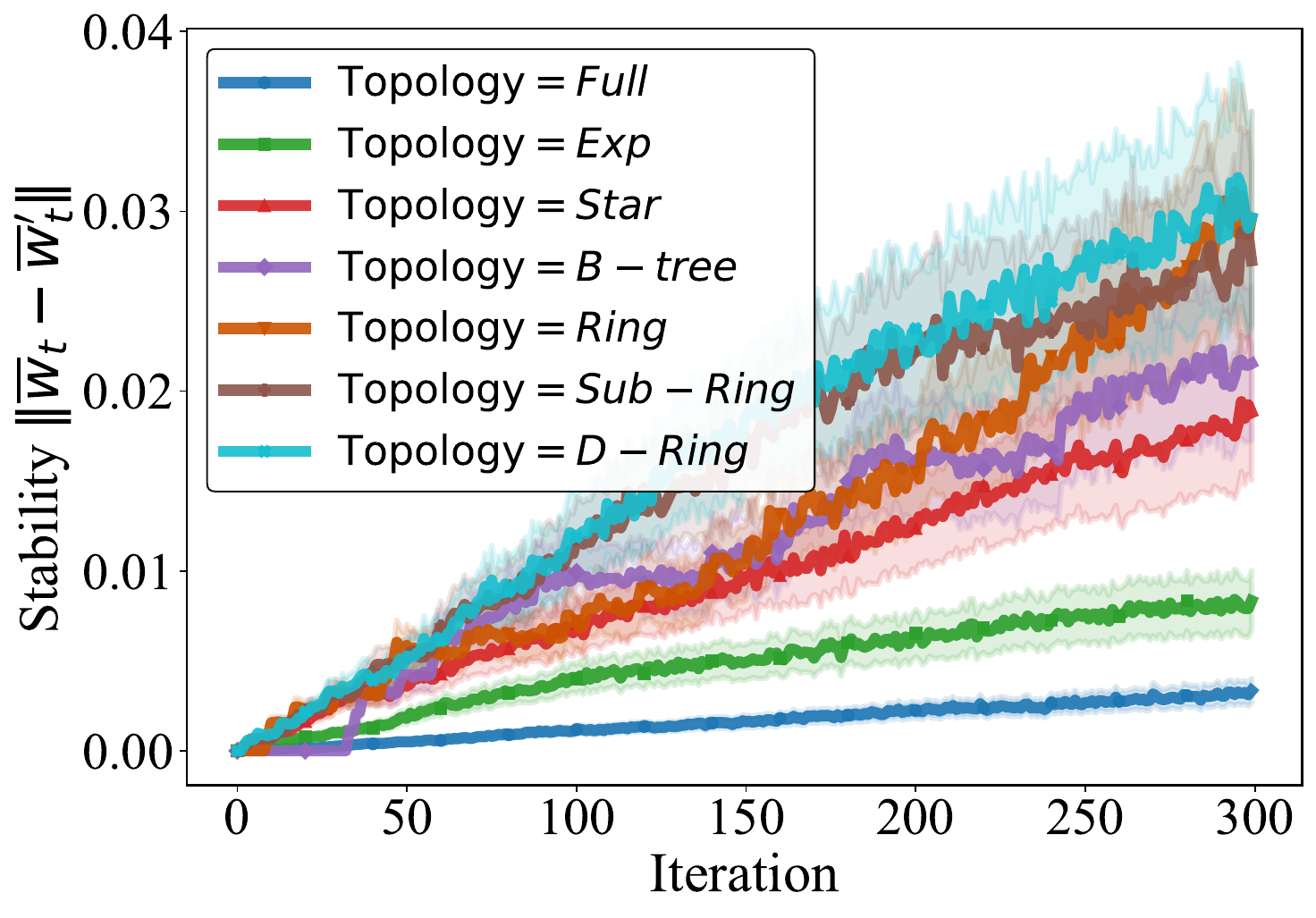}}
\end{minipage}
\begin{minipage}[t]{\textwidth}
\centering
    \subfloat[$\epsilon_{\text{opt}}$ of Different LR.]{ 
	\includegraphics[width=0.32\textwidth]{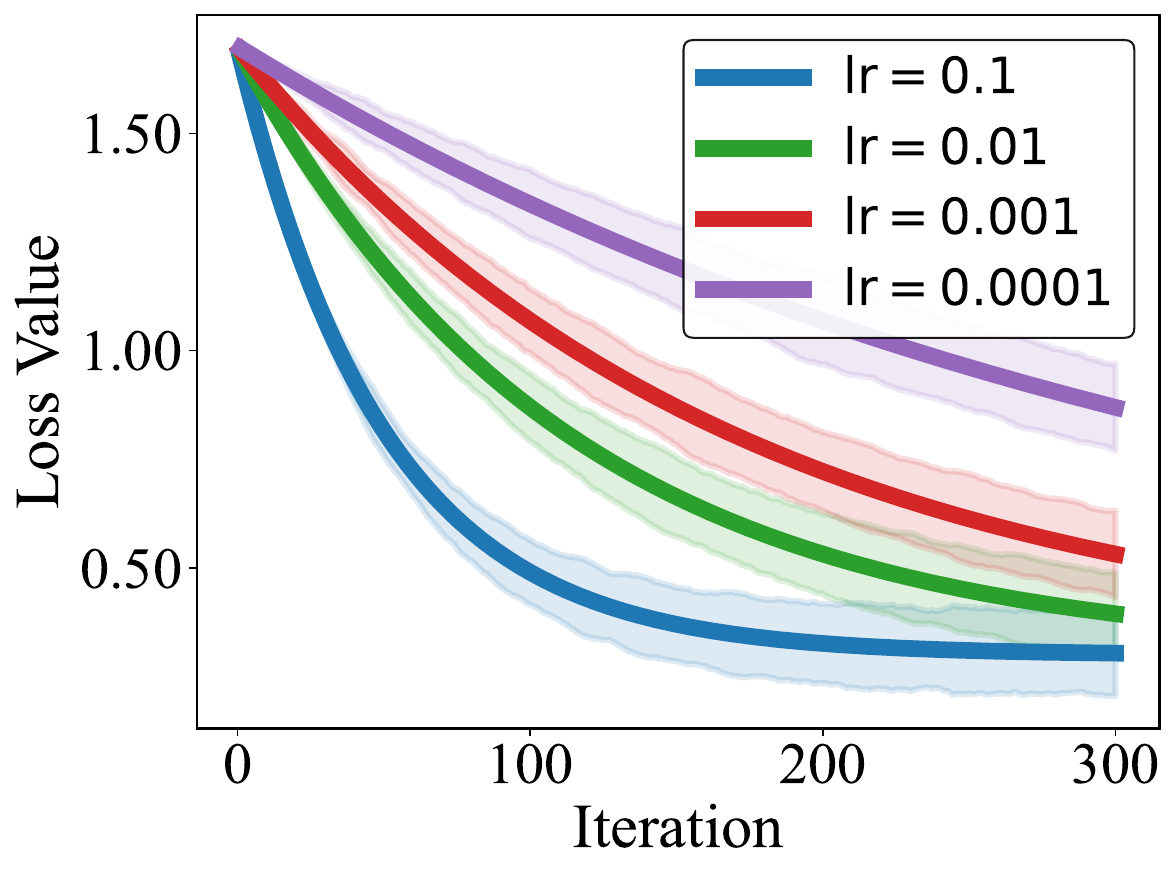}}\hfill
    \subfloat[$\epsilon_{\text{opt}}$ of Different Client Size.]{ 
	\includegraphics[width=0.32\textwidth]{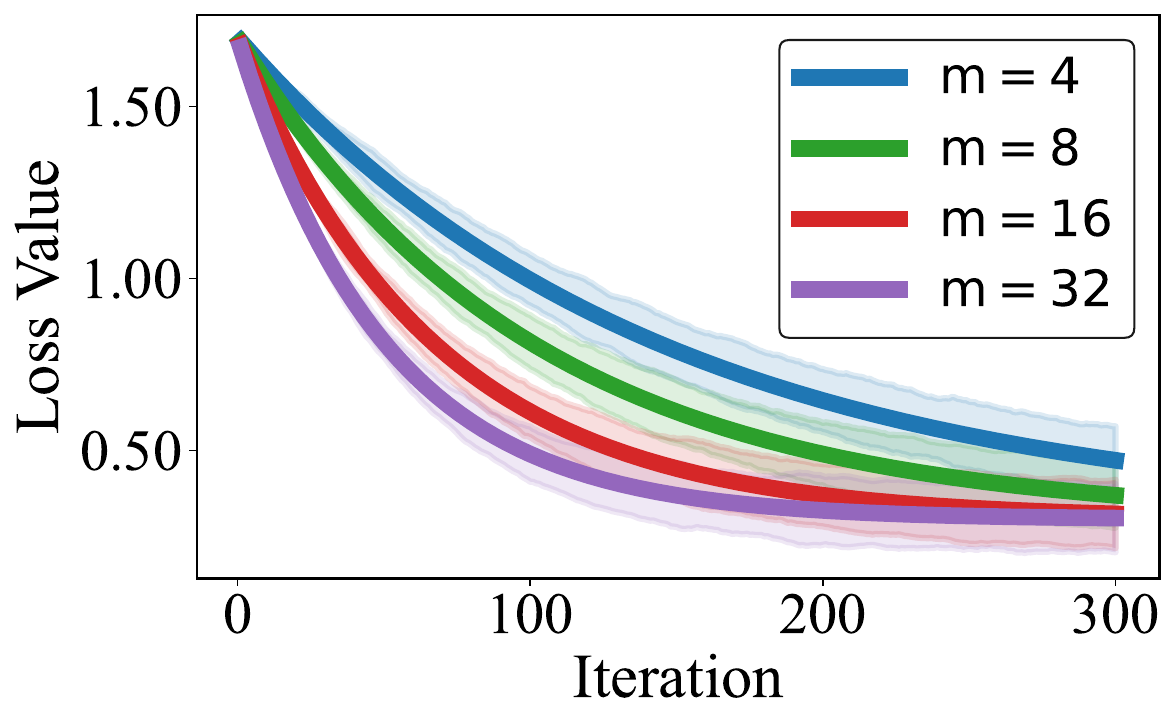}}\hfill
    \subfloat[$\epsilon_{\text{opt}}$ of Different Topology.]{ 
	\includegraphics[width=0.32\textwidth]{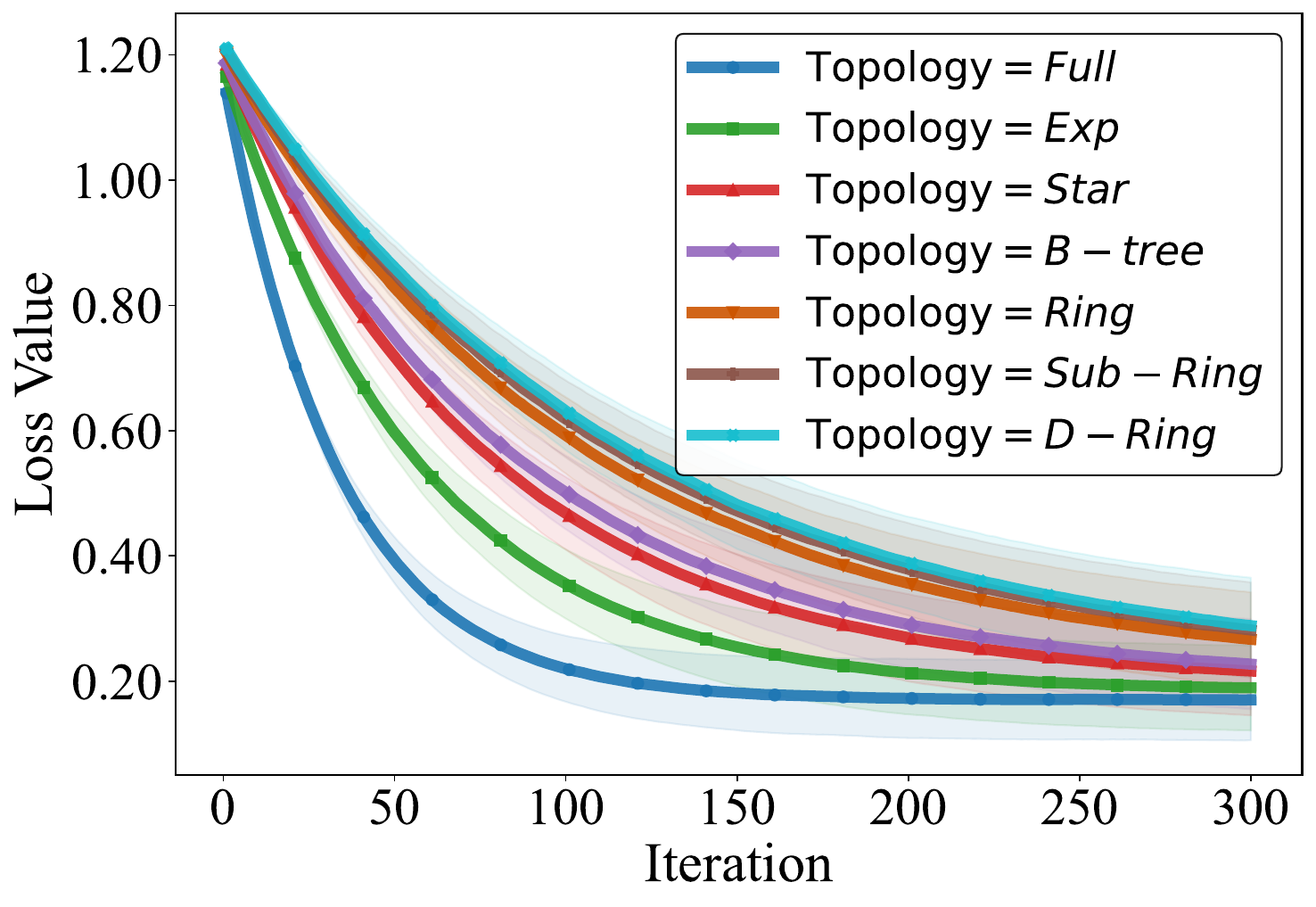}}
\end{minipage}
\vspace{-0.1cm}
\caption{Impacts of generalization and optimization errors on non-convex objective.}
\vspace{-0.5cm} 
\label{exp:LetNet}
\end{figure*}

\subsection{LeNet on CIFAR-10 Dataset}
\label{sec:LeNet}

We further evaluate the non-convex setting using LeNet-5~\citep{lecun-1998} on the CIFAR-10 dataset~\citep{krizhevsky-2009}. 
The 50,000 training samples are evenly divided among 100 clients, each holding 500 samples. 
We fix a batch size of 50 and adopt stage-wise learning rate decay. 
To study the effect of learning rate under stable training, we select $\gamma \in \{0.001, 0.002, 0.003, 0.004\}$ with $m=32$. 
To investigate the role of client size, we vary $m \in \{4, 8, 16, 32\}$. 
In the non-convex experiments, we consider a broader family of communication graphs than in the convex case, including
\textit{Full}, \textit{Exp}, \textit{Star}, \textit{B-tree}, \textit{Ring}, \textit{Sub-Ring}, and \textit{D-Ring}.
We train for 300 iterations and report stability and loss curves.

As shown in Figure~\ref{exp:LetNet}, the empirical behavior is consistent with our non-convex analysis.
First, panels (a) and (d) illustrate a clear learning-rate effect.
Larger step sizes lead to faster growth of the stability measure while also accelerating loss reduction in the early stage.
This observation reflects the trajectory-level amplification described in Theorem~\ref{thm:stability-nonconvex}:
in the absence of convex contraction, perturbations may accumulate multiplicatively along the optimization path.
Consequently, increasing the learning rate improves short-term optimization but amplifies sensitivity to data perturbations.
Second, panels (b) and (e) show that increasing the number of clients $m$ improves both stability and optimization.
With larger $m$, the stability curves grow more slowly and the loss decreases more rapidly.
This is consistent with the role of the effective sample size $mn$ in our bounds:
a larger network reduces the influence of any single data point and mitigates stochastic variability across iterations.
Third, panels (c) and (f) demonstrate that topology has a pronounced impact in the non-convex regime.
Well-connected graphs such as \textit{Full} and \textit{Exp} exhibit the smallest stability growth and the fastest convergence,
whereas sparse or highly directional structures such as \textit{Ring} and \textit{D-Ring} show substantially larger stability accumulation and slower loss decay.
Intermediate graphs, including \textit{Star}, \textit{B-tree}, and \textit{Sub-Ring}, fall between these two extremes.
This empirical ordering aligns with the spectral properties summarized in Table~\ref{tab:topology_properties}.
Graphs with larger spectral gap $(1-\lambda)$ mix information more rapidly,
and graphs with larger balance parameter $\delta$ distribute influence more evenly across nodes.
Since our non-convex bounds depend on the factor $1/(\delta(1-\lambda))$,
poor connectivity or imbalance increases the magnitude of perturbation propagation throughout training.
The experiments therefore support the theoretical conclusion that, in the non-convex setting,
network topology does not merely affect constants, but interacts with optimization dynamics.

%% file: Latex/conclution.tex
\section{Conclusion}

In this paper, we studied the stability and generalization behavior of the Stochastic Gradient Push (SGP) algorithm over directed communication networks. By leveraging the framework of uniform stability, we established explicit generalization and excess risk bounds for both convex objectives and non-convex objectives satisfying the P\L{} condition. Our analysis highlights the role of column-stochastic communication in shaping learning dynamics and provides a precise characterization of how topological imbalance and spectral gap jointly influence stability and optimization performance. These results offer a theoretical understanding of when Push-Sum correction is necessary and how directed network structure affects decentralized learning beyond asymptotic convergence.

Several directions remain open for future investigation. 
First, extending the current analysis to more general non-convex settings beyond the P\L{} condition would further broaden the applicability of the theory. 
Second, it would be of interest to study adaptive or time-varying communication topologies, where the imbalance and spectral properties evolve over time. 
Finally, incorporating additional practical considerations such as communication compression, quantization, or partial participation into the stability-based framework may provide deeper insight into the generalization behavior of decentralized learning systems in realistic environments and large-scale deployments.

%% file: Latex/appendix_a.tex
\section{Notations and Abbreviations}
\label{sec:nomenclature}

\renewcommand{\arraystretch}{1.12}
\footnotesize

\begin{longtable}{c p{0.78\textwidth}}
\caption{Notations and Abbreviations}
\label{tab:nomenclature} \\

\toprule
\textbf{Symbol} & \textbf{Meaning} \\
\midrule
\endfirsthead

\toprule
\textbf{Symbol} & \textbf{Meaning} \\
\midrule
\endhead

\midrule
\multicolumn{2}{r}{\emph{Continued on next page}} \\
\endfoot

\bottomrule
\endlastfoot

$\mathcal{X}_i$ & Input space of node $i$ \\

$\mathcal{Y}_i$ & Output space of node $i$ \\

$\mathcal{S}_i$ & $\mathcal{X}_i \times \mathcal{Y}_i$ \\

$\bm{S}$ & Training dataset \\

$\bm{S}'$ & Dataset differing from $\bm{S}$ by one sample \\

$\bm{\xi}_i^{(t)}$ & Sample used by node $i$ at round $t$ \\

$\mathcal{A}$ & Learning algorithm \\

$\bm{w}_i^{(t)}$ & Model parameters of node $i$ at round $t$ \\

$\bm{W}^{(t)}$ & Parameter matrix at round $t$ \\

$\overline{\bm{w}}^{(t)}$ & Network average $\frac{1}{m}\sum_{i=1}^m \bm{w}_i^{(t)}$ \\

$\bm{w}_{\mathrm{avg}}^{(T)}$ & Weighted averaged iterate $\frac{\sum_{t=0}^{T-1}\gamma_t \overline{\bm{w}}^{(t)}}{\sum_{t=0}^{T-1}\gamma_t}$ \\

$\bm{z}_i^{(t)}$ & Debiased parameter of node $i$ at round $t$ \\

$\bm{Z}^{(t)}$ & Debiased parameter matrix at round $t$ \\

$\bm{u}^{(t)}$ & Push-sum weight vector at round $t$ \\

$\bm{\pi}$ & Stationary distribution vector of communication matrix $\bm{P}$ \\

$\bm{P}^{(t)}$ & Communication matrix at round $t$ \\

$p_{i,j}^{(t)}$ & Entry of communication matrix $\bm{P}^{(t)}$ \\

$\bm{H}$ & Residual matrix $\bm{P}-\bm{\pi}\bm{1}^\top$ \\

$\lambda$ & Spectral radius of $\bm{H}$ ($\lambda<1$) \\

$C_H$ & Constant satisfying $\|\bm{H}^t\|_\infty \le C_H \lambda^t$ \\

$m$ & Number of nodes \\

$n$ & Number of samples \\

$T$ & Total number of iterations \\

$\gamma_t$ & Learning rate (step size) at round $t$ \\

$v$ & Learning rate constant in diminishing stepsizes \\

$L$ & Smoothness (gradient Lipschitz) constant \\

$\alpha$ & PŁ-condition parameter \\

$\kappa$ & Condition number $L/\alpha$ \\

$r$ & Radius of bounded parameter space \\

$G$ & Uniform bound on stochastic gradients \\

$C_{w_0}$ & Initialization-dependent constant \\

$\Delta_t$ & Distance $\|\overline{\bm{w}}^{(t)} - \overline{\bm{w}}^{\prime(t)}\|$ \\

$\rho$ & Stability growth factor $1+L\gamma-\frac{L\gamma}{mn}$ \\

$\mu$ & Logarithm of $\rho$, i.e., $\mu=\ln\rho$ \\

$f(\bm{w};\xi)$ & Loss function \\

$F_{\bm{S}}(\bm{w})$ & Empirical risk over dataset $\bm{S}$ \\

$F(\bm{w})$ & Population risk \\

$\bm{w}_{\bm{S}}^*$ & Minimizer of empirical risk $F_{\bm{S}}$ \\

$\bm{w}^*$ & Minimizer of population risk $F$ \\

$\epsilon_{\mathrm{stab}}$ & Uniform stability bound \\

$\epsilon_{\mathrm{opt}}$ & Optimization error bound \\

$\epsilon_{\mathrm{ave\text{-}stab}}$ & Averaged stability bound \\

$\epsilon_{\mathrm{exc}}$ & Excess generalization error bound \\

$\mathbb{E}[\cdot]$ & Expectation operator \\

$\nabla$ & Gradient operator \\

ERM & Empirical Risk Minimization \\

DL & Decentralized Learning \\

SGP & Stochastic Gradient Push \\

SGD & Stochastic Gradient Descent \\

\end{longtable}

\normalsize

\normalsize

%% file: Latex/appendix_b.tex
\section{Technical Propositions and Lemmas}
\label{sec:lemma}

\subsection{Proof of Proposition~\ref{prop:update}}
\label{pro:update}
\begin{proof}
By the SGP update rule,
\begin{equation}
\label{eq:sgp_update}
\bm{W}^{(t+1)}=\bm{P}^{(t)}\!\left(\bm{W}^{(t)}-\gamma_t \nabla \bm{f}(\bm{Z}^{(t)};\bm{S}^{(t)})\right).
\end{equation}

Consider the network average
\(
\overline{\bm{w}}^{(t)}=\frac{1}{m}\bm{1}^\top \bm{W}^{(t)}.
\)
Since $\bm 1^\top \bm P^{(t)}=\bm 1^\top$, then
\begin{align}
\overline{\bm{w}}^{(t+1)}
&=\frac{1}{m}\bm{1}^\top \bm{W}^{(t+1)} \nonumber\\
&=\frac{1}{m}\bm{1}^\top \bm{P}^{(t)}\!\left(\bm{W}^{(t)}-\gamma_t \nabla \bm{f}(\bm{Z}^{(t)};\bm{S}^{(t)})\right) \nonumber\\
&=\frac{1}{m}\bm{1}^\top\!\left(\bm{W}^{(t)}-\gamma_t \nabla \bm{f}(\bm{Z}^{(t)};\bm{S}^{(t)})\right)
\quad (\text{since } \bm{1}^\top \bm{P}^{(t)}=\bm{1}^\top) \nonumber\\
&=\overline{\bm{w}}^{(t)}-\frac{\gamma_t}{m}\bm{1}^\top \nabla \bm{f}(\bm{Z}^{(t)};\bm{S}^{(t)}) \nonumber\\
&=\overline{\bm{w}}^{(t)}-\frac{\gamma_t}{m}\sum_{i=1}^m \nabla f(\bm{z}_i^{(t)};\bm{\xi}_i^{(t)}).
\label{eq:consensus_property}
\end{align}
This proves Proposition~\ref{prop:update}.
\end{proof}

\subsection{Proof of Lemma~\ref{lem:push_sum_consistency}}
\label{pro:push_sum_consistency}

\begin{proof}
Since $\bm{P}$ is a nonnegative, column-stochastic, and primitive matrix, the Perron--Frobenius theorem
guarantees the existence of a unique positive right eigenvector $\bm{\pi}\in\mathbb{R}^m_{++}$ such that $\bm{P}\bm{\pi}=\bm{\pi}$ and $\|\bm{\pi}\|_1=1$.
The corresponding left eigenvector is $\bm{1}^\top$, satisfying $\bm{1}^\top\bm{P}=\bm{1}^\top$.
Define
\begin{equation}
    \bm{H}:=\bm{P}-\bm{\pi}\bm{1}^\top.
\end{equation}
Then $\bm{H}\bm{\pi}=\bm{0}$ and $\bm{1}^\top\bm{H}=\bm{0}^\top$, and hence
$\bm{P}^t=\bm{\pi}\bm{1}^\top+\bm{H}^t$ for any $t\ge 1$.

Let $\lambda$ be the spectral radius of $\bm{H}$, which satisfies $\lambda<1$ due to primitivity.
Thus there exists a constant $C_H>0$ such that, for the induced $\infty$-norm,
\begin{equation}\label{eq:H_bound}
\|\bm{H}^t\|_\infty \le C_H\lambda^t,\quad \forall t\ge 0.
\end{equation}
In particular, $(\bm{P}^t)_{ij}=\pi_i+(\bm{H}^t)_{ij}$ and $|(\bm{H}^t)_{ij}|\le \|\bm{H}^t\|_\infty$.

We examine the evolution of the push-sum weight vector $\bm{u}^{(t)}=\bm{P}\bm{u}^{(t-1)},\bm{u}^{(0)}=\bm{1}.$
Therefore,
\begin{equation}
\bm{u}^{(t)}=\bm{P}^t\bm{1}=(\bm{\pi}\bm{1}^\top+\bm{H}^t)\bm{1}
=m\bm{\pi}+\bm{H}^t\bm{1}.
\end{equation}
For node $i$, $u_i^{(t)}=m\pi_i+(\bm{H}^t\bm{1})_i$. Let $\delta=\min_i\pi_i>0$.
Using \eqref{eq:H_bound} and $\|\bm{1}\|_\infty=1$,
\[
|(\bm{H}^t\bm{1})_i|\le \|\bm{H}^t\|_\infty\|\bm{1}\|_\infty \le C_H\lambda^t.
\]
Hence there exists $T$ such that for all $t>T$,
$|(\bm{H}^t\bm{1})_i|\le \tfrac12 m\delta$, implying $u_i^{(t)}\ge \tfrac12 m\delta$.
For the finite interval $0\le t\le T$, primitivity of $\bm{P}$ and positivity of $\bm{u}^{(0)}$
ensure $u_i^{(t)}>0$. Thus, by possibly enlarging constants, we may take
\begin{equation}\label{eq:u_lower}
u_i^{(t)}\ge \frac12 m\delta,\quad \forall t\ge 0,\ \forall i\in[m].
\end{equation}

Next, the numerator iterates satisfy
\begin{equation}
\bm{W}^{(t)}=\bm{P}\Big(\bm{W}^{(t-1)}-\gamma_{t-1}\nabla\bm{f}(\bm{Z}^{(t-1)};\bm{S}^{(t-1)})\Big).
\end{equation}
Unrolling the recursion gives, for node $i$,
\begin{equation}\label{eq:w_expansion}
\bm{w}_i^{(t)}
= \sum_{j=1}^m (\bm{P}^t)_{ij}\bm{w}_j^{(0)}
- \sum_{s=0}^{t-1}\sum_{j=1}^m (\bm{P}^{t-s})_{ij}\gamma_s \nabla f(\bm{z}_j^{(s)};\bm{\xi}_j^{(s)}).
\end{equation}
The global average evolves as
\begin{equation}\label{eq:w_bar_expansion}
\overline{\bm{w}}^{(t)}
=\frac{1}{m}\sum_{j=1}^m \bm{w}_j^{(0)}
-\frac{1}{m}\sum_{s=0}^{t-1}\sum_{j=1}^m \gamma_s \nabla f(\bm{z}_j^{(s)};\bm{\xi}_j^{(s)}).
\end{equation}

We bound $\|\bm{z}_i^{(t)}-\overline{\bm{w}}^{(t)}\|$.
Since $\bm{z}_i^{(t)}=\bm{w}_i^{(t)}/u_i^{(t)}$,
$\bm{z}_i^{(t)}-\overline{\bm{w}}^{(t)}
=\frac{\bm{w}_i^{(t)}-u_i^{(t)}\overline{\bm{w}}^{(t)}}{u_i^{(t)}}.$
Substituting $(\bm{P}^n)_{ij}=\pi_i+(\bm{H}^n)_{ij}$ into \eqref{eq:w_expansion} yields
\begin{align}
\bm{w}_i^{(t)}
&=\pi_i\left(\sum_{j=1}^m\bm{w}_j^{(0)}-\sum_{s=0}^{t-1}\sum_{j=1}^m\gamma_s\nabla f(\bm{z}_j^{(s)};\bm{\xi}_j^{(s)})\right)
+\sum_{j=1}^m(\bm{H}^t)_{ij}\bm{w}_j^{(0)} \nonumber\\
&\quad -\sum_{s=0}^{t-1}\sum_{j=1}^m(\bm{H}^{t-s})_{ij}\gamma_s\nabla f(\bm{z}_j^{(s)};\bm{\xi}_j^{(s)}).
\label{eq:w_split}
\end{align}
By \eqref{eq:w_bar_expansion}, the term in parentheses equals $m\overline{\bm{w}}^{(t)}$.
Moreover,
\[
u_i^{(t)}\overline{\bm{w}}^{(t)}
=(m\pi_i+(\bm{H}^t\bm{1})_i)\overline{\bm{w}}^{(t)}
=m\pi_i\overline{\bm{w}}^{(t)}+(\bm{H}^t\bm{1})_i\overline{\bm{w}}^{(t)}.
\]
Therefore the dominant term cancels and we obtain
\begin{equation}\label{eq:delta_final}
\bm{w}_i^{(t)}-u_i^{(t)}\overline{\bm{w}}^{(t)}
=\underbrace{\sum_{j=1}^m(\bm{H}^t)_{ij}\bm{w}_j^{(0)}}_{\text{Term I}}
-\underbrace{\sum_{s=0}^{t-1}\sum_{j=1}^m(\bm{H}^{t-s})_{ij}\gamma_s\nabla f(\bm{z}_j^{(s)};\bm{\xi}_j^{(s)})}_{\text{Term II}}
-\underbrace{(\bm{H}^t\bm{1})_i\overline{\bm{w}}^{(t)}}_{\text{Term III}}.
\end{equation}

For Term I, using \eqref{eq:H_bound} and the definition of $C_{w_0}$,
\begin{align}
\|\text{Term I}\|
&\le \sum_{j=1}^m |(\bm{H}^t)_{ij}|\|\bm{w}_j^{(0)}\|
\le \|\bm{H}^t\|_\infty \sum_{j=1}^m \|\bm{w}_j^{(0)}\| \nonumber\\
&\le C_H\lambda^t \sum_{j=1}^m \|\bm{w}_j^{(0)}\|
= m C_H\lambda^t\, C_{w_0}.
\label{eq:termI}
\end{align}

For Term II,
by the triangle inequality and the uniform bound $\|\nabla f(\bm{z}_j^{(s)};\bm{\xi}_j^{(s)})\|\le G$,
\begin{align}
\|\text{Term II}\|
&\le \sum_{s=0}^{t-1}\sum_{j=1}^m |(\bm{H}^{t-s})_{ij}|\gamma_s
\big\|\nabla f(\bm{z}_j^{(s)};\bm{\xi}_j^{(s)})\big\| \nonumber\\
&\le G\sum_{s=0}^{t-1}\gamma_s \sum_{j=1}^m |(\bm{H}^{t-s})_{ij}|
\le G\sum_{s=0}^{t-1}\gamma_s \|\bm{H}^{t-s}\|_\infty \nonumber\\
&\le G C_H\sum_{s=0}^{t-1}\lambda^{t-s}\gamma_s .
\label{eq:termII}
\end{align}

For Term III,
using $|(\bm{H}^t\bm{1})_i|\le \|\bm{H}^t\|_\infty\|\bm{1}\|_\infty\le C_H\lambda^t$
and \eqref{eq:w_bar_expansion} together with the uniform gradient bound,
\begin{align}
\|\text{Term III}\|
&= |(\bm{H}^t\bm{1})_i|\cdot \|\overline{\bm{w}}^{(t)}\|
\le C_H\lambda^t \left\|\frac{1}{m}\sum_{j=1}^m\bm{w}_j^{(0)}
-\frac{1}{m}\sum_{s=0}^{t-1}\sum_{j=1}^m \gamma_s \nabla f(\bm{z}_j^{(s)};\bm{\xi}_j^{(s)})\right\| \nonumber\\
&\le C_H\lambda^t\left(\frac{1}{m}\sum_{j=1}^m\|\bm{w}_j^{(0)}\|
+\frac{1}{m}\sum_{s=0}^{t-1}\sum_{j=1}^m\gamma_s\big\|\nabla f(\bm{z}_j^{(s)};\bm{\xi}_j^{(s)})\big\|\right) \nonumber\\
&\le C_H\lambda^t\left(C_{w_0}+ G\sum_{s=0}^{t-1}\gamma_s\right).
\label{eq:termIII}
\end{align}

Combining \eqref{eq:termI}--\eqref{eq:termIII} and absorbing numerical factors into a constant $C'>0$,
we obtain
\begin{equation}\label{eq:Delta_bound}
\|\bm{w}_i^{(t)}-u_i^{(t)}\overline{\bm{w}}^{(t)}\|
\le mC'\left(\lambda^t C_{w_0}+ G\sum_{s=0}^{t-1}\lambda^{t-s}\gamma_s\right),
\end{equation}
where we used that $\sum_{s=0}^{t-1}\gamma_s \le \sum_{s=0}^{t-1}\lambda^{t-s}\gamma_s/( \min_{1\le k\le t}\lambda^k)$
and absorbed into $C'$.

Finally, using \eqref{eq:u_lower},
\begin{align}
\|\bm{z}_i^{(t)}-\overline{\bm{w}}^{(t)}\|
&=\frac{1}{u_i^{(t)}}\|\bm{w}_i^{(t)}-u_i^{(t)}\overline{\bm{w}}^{(t)}\|
\le \frac{2}{m\delta}\|\bm{w}_i^{(t)}-u_i^{(t)}\overline{\bm{w}}^{(t)}\| \nonumber\\
&\le \frac{C}{\delta}\left(\lambda^t C_{w_0}+ G\sum_{s=0}^{t-1}\lambda^{t-s}\gamma_s\right),
\end{align}
where $C>0$ absorbs $C'$ and numerical constants. This completes the proof.
\end{proof}

%% file: Latex/appendix_c.tex
\section{Proof of Theorem and Corollary}
\label{sec:proof}

\subsection{Proof of Theorem~\ref{thm:stability-convex}}
\label{pro:stability-convex}
\begin{proof}
Suppose the two sample sets $\bm{S}$ and $\bm{S}^{\prime}$ differ in only one sample among the $n$ samples.
Assume that at each iteration $t$, a global index is sampled uniformly from $\{1,\dots,n\}$ and broadcast to all nodes.
Then with probability $1-\frac{1}{n}$, the sampled indices at all nodes are identical for both runs; with probability $\frac{1}{n}$, the sampled point differs.

Let $\Delta_t := \|\overline{\bm{w}}^{(t)}-\overline{\bm{w}}^{\prime(t)}\|$ and assume $\overline{\bm{w}}^{(0)}=\overline{\bm{w}}^{\prime(0)}$, i.e., $\Delta_0=0$.

\paragraph{Case 1: Identical Samples.}

Using Proposition~\ref{prop:update} and adding/subtracting
$\nabla f(\overline{\bm{w}}^{(t)};\bm{\xi}_i^{(t)})$ and
$\nabla f(\overline{\bm{w}}^{\prime(t)};\bm{\xi}_i^{(t)})$, we obtain
\begin{align}
\Delta_{t+1}
&= \Big\|\overline{\bm{w}}^{(t)}-\overline{\bm{w}}^{\prime(t)}
-\frac{\gamma_t}{m}\sum_{i=1}^{m}\big(\nabla f(\bm{z}_i^{(t)};\bm{\xi}_i^{(t)})
-\nabla f(\bm{z}_i^{\prime(t)};\bm{\xi}_i^{(t)})\big)\Big\| \nonumber\\
&\le
\Big\|\overline{\bm{w}}^{(t)}-\overline{\bm{w}}^{\prime(t)}
-\frac{\gamma_t}{m}\sum_{i=1}^{m}\big(\nabla f(\overline{\bm{w}}^{(t)};\bm{\xi}_i^{(t)})
-\nabla f(\overline{\bm{w}}^{\prime(t)};\bm{\xi}_i^{(t)})\big)\Big\| \nonumber\\
&\quad
+\Big\|\frac{\gamma_t}{m}\sum_{i=1}^{m}\big(\nabla f(\overline{\bm{w}}^{(t)};\bm{\xi}_i^{(t)})
-\nabla f(\bm{z}_i^{(t)};\bm{\xi}_i^{(t)})\big)\Big\| \nonumber\\
&\quad
+\Big\|\frac{\gamma_t}{m}\sum_{i=1}^{m}\big(\nabla f(\overline{\bm{w}}^{\prime(t)};\bm{\xi}_i^{(t)})
-\nabla f(\bm{z}_i^{\prime(t)};\bm{\xi}_i^{(t)})\big)\Big\| \nonumber\\
&\le
\Delta_t
+\frac{L\gamma_t}{m}\sum_{i=1}^{m}\|\overline{\bm{w}}^{(t)}-\bm{z}_i^{(t)}\|
+\frac{L\gamma_t}{m}\sum_{i=1}^{m}\|\overline{\bm{w}}^{\prime(t)}-\bm{z}_i^{\prime(t)}\|.
\end{align}

Applying Lemma~\ref{lem:push_sum_consistency},
\begin{align}
\Delta_{t+1}
\le \Delta_t
+\frac{2CL\gamma_t}{\delta}
\Big(\lambda^t C_{w_0}
+G\sum_{s=0}^{t-1}\lambda^{t-s}\gamma_s\Big).
\label{eq:case1_final_complete}
\end{align}

\paragraph{Case 2: Different Sample.}

Without loss of generality, assume node $1$ uses the replaced sample.
Proceeding as above and using $\|\nabla f(\cdot)\|\le G$,
\begin{align}
\Delta_{t+1}
&\le
\Delta_t
+\frac{L\gamma_t}{m}\sum_{i=2}^{m}\|\overline{\bm{w}}^{(t)}-\bm{z}_i^{(t)}\|
+\frac{L\gamma_t}{m}\sum_{i=2}^{m}\|\overline{\bm{w}}^{\prime(t)}-\bm{z}_i^{\prime(t)}\|
+\frac{2G\gamma_t}{m}.
\end{align}
Since $\frac{1}{m}\sum_{i=2}^m\|\cdot\|\le \frac{1}{m}\sum_{i=1}^m\|\cdot\|$, applying the same lemma gives
\begin{align}
\Delta_{t+1}
\le
\Delta_t
+\frac{2CL\gamma_t}{\delta}
\Big(\lambda^t C_{w_0}
+G\sum_{s=0}^{t-1}\lambda^{t-s}\gamma_s\Big)
+\frac{2G\gamma_t}{m}.
\label{eq:case2_final_complete}
\end{align}

\paragraph{Expectation recursion.}

Taking expectation and using the two cases,
\begin{align}
\mathbb{E}[\Delta_{t+1}]
\le
\mathbb{E}[\Delta_t]
+\frac{2CL\gamma_t}{\delta}
\Big(\lambda^t C_{w_0}
+G\sum_{s=0}^{t-1}\lambda^{t-s}\gamma_s\Big)
+\frac{2G\gamma_t}{mn}.
\end{align}

Summing from $t=0$ to $T-1$ yields
\begin{align}
\mathbb{E}[\Delta_T]
\le
\frac{2CL C_{w_0}}{\delta}\sum_{t=0}^{T-1}\gamma_t\lambda^t
+\frac{2CLG}{\delta}
\sum_{t=0}^{T-1}\gamma_t
\sum_{s=0}^{t-1}\lambda^{t-s}\gamma_s
+\frac{2G}{mn}\sum_{t=0}^{T-1}\gamma_t.
\end{align}

Assume the stepsizes are nonincreasing.
Then for $s\le t-1$, $\gamma_s\ge \gamma_t$, and hence
\begin{align}
\sum_{t=0}^{T-1}\gamma_t\sum_{s=0}^{t-1}\lambda^{t-s}\gamma_s
&\le
\sum_{t=0}^{T-1}\gamma_t^2
\sum_{s=0}^{t-1}\lambda^{t-s}
=
\sum_{t=0}^{T-1}\gamma_t^2
\sum_{k=1}^{t}\lambda^k
\le
\frac{1}{1-\lambda}
\sum_{t=0}^{T-1}\gamma_t^2.
\end{align}

Therefore,
\begin{align}
\mathbb{E}[\Delta_T]
\le
\frac{2CL C_{w_0}}{\delta}\sum_{t=0}^{T-1}\gamma_t\lambda^t
+\frac{2CLG}{\delta(1-\lambda)}\sum_{t=0}^{T-1}\gamma_t^2
+\frac{2G}{mn}\sum_{t=0}^{T-1}\gamma_t.
\end{align}

Finally, by $G$-Lipschitzness,
\begin{align}
\epsilon_{\mathrm{stab}}
&=
\mathbb{E}|f(\overline{\bm{w}}_T;\bm{z})-f(\overline{\bm{w}}^{\prime}_T;\bm{z})|
\le
G\,\mathbb{E}[\Delta_T] \nonumber\\
&\le
\frac{2CGL C_{w_0}}{\delta}\sum_{t=0}^{T-1}\gamma_t\lambda^t
+\frac{2CG^2L}{\delta(1-\lambda)}\sum_{t=0}^{T-1}\gamma_t^2
+\frac{2G^2}{mn}\sum_{t=0}^{T-1}\gamma_t.
\end{align}

This completes the proof of Theorem~\ref{thm:stability-convex}.
\end{proof}

\subsection{Proof of Corollary~\ref{cor:stability-convex}}
\label{pro:stability-convex-cor}
\begin{proof}\paragraph{Case 1: Constant Learning Rate.}

For constant learning rates $\gamma_t=\gamma\leq2/L$, by Theorem~\ref{thm:stability-convex}, we have
\begin{align}
\epsilon_{\mathrm{stab}}
&\leq\frac{2CGL C_{w_0}}{\delta }\sum_{t=0}^{T-1}\gamma {\lambda}^t
+\frac{2CG^2L}{\delta(1-{\lambda})}\sum_{t=0}^{T-1}{\gamma}^2
+ \frac{2 G^2 }{mn}\sum_{t=0}^{T-1}\gamma \nonumber \\
&\overset{(a)}\leq\frac{2CGL\gamma C_{w_0}}{\delta (1-{\lambda})}
+\left(\frac{2CG^2L{\gamma}^2}{\delta(1-{\lambda})}
+ \frac{2 G^2 {\gamma}}{mn}\right)T.
\end{align}

\medskip
\paragraph{Case 2: Decreasing Learning Rate.}
For diminishing learning rates $\gamma_t=\frac{v}{t+1}$ with $v\leq2/L$, Theorem~\ref{thm:stability-convex} yields
\begin{align}
\epsilon_{\mathrm{stab}}
&\leq \frac{2CGL C_{w_0}}{\delta }\sum_{t=0}^{T-1}\frac{v {\lambda}^t}{t+1} 
+\frac{2CG^2L}{\delta(1-{\lambda})}\sum_{t=0}^{T-1}\Big(\frac{v}{t+1}\Big)^2
+ \frac{2 G^2  }{mn}\sum_{t=0}^{T-1}\frac{v}{t+1} \nonumber \\
&\overset{(a)(b)(c)}\leq 
\frac{2vCGL C_{w_0}}{\delta (1-{\lambda})}
+\frac{4CG^2L v^2}{\delta(1-{\lambda})}
+ \frac{2 G^2  v}{mn}(1+\ln{T}) \nonumber \\
&= \frac{2 G^2 v}{mn}\ln{T}
+ \frac{2vCGL C_{w_0} + 4CG^2L v^2}{\delta (1-{\lambda})}
+ \frac{2 G^2 v}{mn},
\end{align}
where we have used:
\[
\tag{a}
\sum_{t=0}^{T-1}\lambda^t \le \frac{1}{1-\lambda}, 
\qquad
\sum_{t=0}^{T-1}\frac{\lambda^t}{t+1}
\le \sum_{t=0}^{T-1}\lambda^t
\le \frac{1}{1-\lambda},
\]
\[
\tag{b}
\sum_{t=0}^{T-1}\frac{1}{t+1}
=1+\sum_{t=1}^{T-1}\frac{1}{t+1}
\leq 1 + \int_{1}^{T}\frac{1}{x}\,dx
=1+\ln{T},
\]
\[
\tag{c}
\sum_{t=0}^{T-1}\frac{1}{(t+1)^2}
\leq 1+\int_{1}^{T}\frac{1}{x^2}\,dx
=2-\frac{1}{T}
\leq 2.
\]

This completes the proof of Corollary~\ref{cor:stability-convex}.
\end{proof}

\subsection{Proof of Theorem~\ref{thm:optimal-convex}}
\label{pro:optimal-convex}

\begin{proof} 
Using the convexity of $F_{\bm{S}}$, we have
\begin{equation}
\tag{d}
F_{\bm{S}}(\overline{\bm{w}}^{(t)}) - F_{\bm{S}}(\bm{w}_{\bm{S}}^*)
\leq \left\langle \nabla F_{\bm{S}}(\overline{\bm{w}}^{(t)}),\,
\overline{\bm{w}}^{(t)} - \bm{w}_{\bm{S}}^* \right\rangle .
\end{equation}

We derive
\begin{align}
 \mathbb{E} \big\|\overline{\bm{w}}^{(t+1)}-\bm{w}_{\bm{S}}^*\big\|^2
 &= \mathbb{E} \Big\|\overline{\bm{w}}^{(t)} - \frac{\gamma_t}{m}\sum_{i=1}^{m}\nabla f(\bm{z}_{i}^{(t)}; \bm{\xi}_{i}^{(t)})
 -\bm{w}_{\bm{S}}^*\Big\|^2\nonumber \\
 &\leq  \mathbb{E} \big\|\overline{\bm{w}}^{(t)}-\bm{w}_{\bm{S}}^*\big\|^2
     +\frac{2\gamma_t}{m}\mathbb{E}\Big\langle-\sum_{i=1}^{m}\nabla f(\bm{z}_{i}^{(t)}; \bm{\xi}_{i}^{(t)}),\,
       \overline{\bm{w}}^{(t)}-\bm{w}_{\bm{S}}^*\Big\rangle \nonumber \\
 &\quad
     +\frac{\gamma_t^2}{m^2}\mathbb{E} \Big\|\sum_{i=1}^{m}\nabla f(\bm{z}_{i}^{(t)}; \bm{\xi}_{i}^{(t)})\Big\|^2\nonumber \\
&\overset{\text{Ass}~\ref{ass:lipschitz}}{\leq}  
    \mathbb{E} \big\|\overline{\bm{w}}^{(t)}-\bm{w}_{\bm{S}}^*\big\|^2
    +\frac{2\gamma_t}{m}\mathbb{E}\Big\langle-\sum_{i=1}^{m}\nabla f(\bm{z}_{i}^{(t)}; \bm{\xi}_{i}^{(t)}),\,
       \overline{\bm{w}}^{(t)}-\bm{w}_{\bm{S}}^*\Big\rangle
    +\gamma_t^2G^2\nonumber \\
&=  \mathbb{E} \big\|\overline{\bm{w}}^{(t)}-\bm{w}_{\bm{S}}^*\big\|^2
   +\sum_{i=1}^{m}\frac{2\gamma_t}{m}\mathbb{E}\Big\langle-\nabla f(\overline{\bm{w}}^{(t)}; \bm {\xi}_{i}^{(t)}),\,
      \overline{\bm{w}}^{(t)}-\bm{w}_{\bm{S}}^*\Big\rangle\nonumber \\
&\quad
   +\sum_{i=1}^{m}\frac{2\gamma_t}{m}\mathbb{E}\Big\langle
      \nabla f(\overline{\bm{w}}^{(t)}; \bm {\xi}_{i}^{(t)})
      -\nabla f(\bm{z}_{i}^{(t)}; \bm {\xi}_{i}^{(t)}),\,
      \overline{\bm{w}}^{(t)}-\bm{w}_{\bm{S}}^*\Big\rangle
   +\gamma_t^2G^2\nonumber \\
&\overset{(d)}{\leq}  
   \mathbb{E} \big\|\overline{\bm{w}}^{(t)}-\bm{w}_{\bm{S}}^*\big\|^2
   -2\gamma_t\mathbb{E}\big[F_{\bm{S}}(\overline{\bm{w}}^{(t)})-F_{\bm{S}}(\bm{w}_{\bm{S}}^*)\big]\nonumber\\
&\quad
   +\sum_{i=1}^{m}\frac{2\gamma_t}{m}
      \mathbb{E}\big\langle
      \nabla f(\overline{\bm{w}}^{(t)}; \bm {\xi}_{i}^{(t)})
      -\nabla f(\bm{z}_{i}^{(t)}; \bm {\xi}_{i}^{(t)}),\,
      \overline{\bm{w}}^{(t)}-\bm{w}_{\bm{S}}^*\big\rangle
   +\gamma_t^2G^2\nonumber \\
&\overset{\text{Ass}~\ref{ass:bounded}}{\leq}  
   \mathbb{E} \big\|\overline{\bm{w}}^{(t)}-\bm{w}_{\bm{S}}^*\big\|^2
   -2\gamma_t\mathbb{E}\big[F_{\bm{S}}(\overline{\bm{w}}^{(t)})-F_{\bm{S}}(\bm{w}_{\bm{S}}^*)\big]\nonumber\\
&\quad
   +\sum_{i=1}^{m}\frac{4r\gamma_t}{m}
      \mathbb{E}\big\|\nabla f(\overline{\bm{w}}^{(t)}; \bm {\xi}_{i}^{(t)})
      -\nabla f(\bm{z}_{i}^{(t)}; \bm {\xi}_{i}^{(t)})\big\|
   +\gamma_t^2G^2\nonumber \\
&\overset{\text{Ass}~\ref{ass:smooth}}{\leq}  
   \mathbb{E} \big\|\overline{\bm{w}}^{(t)}-\bm{w}_{\bm{S}}^*\big\|^2
   -2\gamma_t\mathbb{E}\big[F_{\bm{S}}(\overline{\bm{w}}^{(t)})-F_{\bm{S}}(\bm{w}_{\bm{S}}^*)\big]\nonumber \\
&\quad
   +\sum_{i=1}^{m}\frac{4 r L\gamma_t}{m}
      \mathbb{E}\big\|\overline{\bm{w}}^{(t)}-\bm{z}_{i}^{(t)}\big\|
   +\gamma_t^2G^2\nonumber \\
&\overset{\text{Lem}~\ref{lem:push_sum_consistency}}{\leq}  
   \mathbb{E} \big\|\overline{\bm{w}}^{(t)}-\bm{w}_{\bm{S}}^*\big\|^2
   -2\gamma_t\mathbb{E}\big[F_{\bm{S}}(\overline{\bm{w}}^{(t)})-F_{\bm{S}}(\bm{w}_{\bm{S}}^*)\big]\nonumber \\
&\quad
   +\sum_{i=1}^{m}\frac{4rL\gamma_t}{m}
     \Big( \frac{C}{\delta} {\lambda}^{t} C_{w_0}
           +\frac{CG}{\delta}\sum_{s=0}^{t-1} {\lambda}^{t-s}\gamma_s\Big)
   +\gamma_t^2G^2\nonumber \\
&\leq  
   \mathbb{E} \big\|\overline{\bm{w}}^{(t)}-\bm{w}_{\bm{S}}^*\big\|^2
   -2\gamma_t\mathbb{E}\big[F_{\bm{S}}(\overline{\bm{w}}^{(t)})-F_{\bm{S}}(\bm{w}_{\bm{S}}^*)\big]\nonumber \\
&\quad
   +\frac{4rCL\gamma_t {\lambda}^{t}C_{w_0}}{\delta }
   +\frac{4rCLG\gamma_t}{\delta}\sum_{s=0}^{t-1} {\lambda}^{t-s}\gamma_s
   +\gamma_t^2G^2\nonumber \\
&\overset{(a)}{\leq}  
   \mathbb{E} \big\|\overline{\bm{w}}^{(t)}-\bm{w}_{\bm{S}}^*\big\|^2
   -2\gamma_t\mathbb{E}\big[F_{\bm{S}}(\overline{\bm{w}}^{(t)})-F_{\bm{S}}(\bm{w}_{\bm{S}}^*)\big]\nonumber \\
&\quad
   +\frac{4rCL\gamma_t {\lambda}^{t}C_{w_0}}{\delta }
   +\Big(G^2+\frac{4rCLG}{\delta(1-{\lambda})}\Big)\gamma_t^2.
 \end{align}

Here, in $(a)$ we assume $\{\gamma_t\}$ is nonincreasing and use
\[
\gamma_t\sum_{s=0}^{t-1}\lambda^{t-s}\gamma_s
\le \gamma_t^2\sum_{s=0}^{t-1}\lambda^{t-s}
\le \frac{\gamma_t^2}{1-\lambda}.
\]

Rearranging the recursion and summing over $t$ from $0$ to $T-1$ yields
\begin{align}
\sum_{t=0}^{T-1}\gamma_t\,\mathbb{E}\big[F_{\bm{S}}(\overline{\bm{w}}^{(t)})-F_{\bm{S}} (\bm{w}_{\bm{S}}^*)\big]
&\leq \frac{1}{2}\big\|\overline{\bm{w}}^{(0)}-\bm{w}_{\bm{S}}^*\big\|^2
+\frac{2rCL C_{w_0}}{\delta }\sum_{t=0}^{T-1}\gamma_t \lambda^{t}\nonumber\\
\quad+\Big(\frac{2rCLG}{\delta(1-{\lambda})}+\frac{G^2}{2}\Big)\sum_{t=0}^{T-1}\gamma_t^2.
\end{align}

Recall the definition of the averaged model
$\bm{w}_{\text{avg}}^{(T)}=\frac{\sum_{t=0}^{T-1}\gamma_t\,\overline{\bm{w}}^{(t)}}{\sum_{t=0}^{T-1}\gamma_t}.$
By convexity of $F_{\bm{S}}$, we have
\begin{align}
 \epsilon_{\mathrm{opt}}
 &=\mathbb{E}\big[F_{\bm{S}}(\bm{w}_{\text{avg}}^{(T)})-F_{\bm{S}}(\bm{w}_{\bm{S}}^*)\big]
   \leq\frac{\sum_{t=0}^{T-1}\gamma_t\,\mathbb{E}\big[F_{\bm{S}}(\overline{\bm{w}}^{(t)})-F_{\bm{S}}(\bm{w}_{\bm{S}}^*)\big]}{\sum_{t=0}^{T-1}\gamma_t}\nonumber \\
 &\leq\frac{\big\|\overline{\bm{w}}^{(0)}-\bm{w}_{\bm{S}}^*\big\|^2}{2\sum_{t=0}^{T-1}\gamma_t}
   +\frac{2rCL C_{w_0}}{\delta \sum_{t=0}^{T-1}\gamma_t}\sum_{t=0}^{T-1}\gamma_t {\lambda}^{t}
   +\Big(\frac{2rCLG}{\delta(1-{\lambda})}+\frac{G^2}{2}\Big)
      \frac{\sum_{t=0}^{T-1}\gamma_t^2}{\sum_{t=0}^{T-1}\gamma_t}.
\end{align}

This completes the proof of Theorem~\ref{thm:optimal-convex}.
\end{proof}

\subsection{Proof of Corollary~\ref{cor:opt-convex}}
\label{pro:opt-convex}
\begin{proof}
\paragraph{Case 1: Constant learning rate.}
For constant learning rates $\gamma_t=\gamma$, substituting into Theorem~\ref{thm:optimal-convex} yields
\begin{align}
 \epsilon_{\mathrm{opt}}
 &\leq\frac{\|\overline{\bm{w}}^{(0)}-\bm{w}_{\bm{S}}^*\|^2}{2\sum_{t=0}^{T-1}\gamma}
   +\frac{2rCL C_{w_0}}{\delta \sum_{t=0}^{T-1}\gamma}
      \sum_{t=0}^{T-1}\gamma {\lambda}^{t}
   +\Big(\frac{2rCLG}{\delta (1-{\lambda})}+\frac{G^2}{2}\Big)
      \frac{\sum_{t=0}^{T-1}\gamma^2}{\sum_{t=0}^{T-1}\gamma}\nonumber \\
&=\frac{\|\overline{\bm{w}}^{(0)}-\bm{w}_{\bm{S}}^*\|^2}{2T \gamma}
   +\frac{2rCL C_{w_0}}{\delta T}\sum_{t=0}^{T-1} {\lambda}^{t}
   +\Big(\frac{2rCLG}{\delta (1-{\lambda})}+\frac{G^2}{2}\Big)\gamma\nonumber \\
&\overset{(a)}\leq
   \frac{\|\overline{\bm{w}}^{(0)}-\bm{w}_{\bm{S}}^*\|^2}{2T \gamma}
   +\frac{2rCL C_{w_0}}{\delta T (1-{\lambda})}
   +\Big(\frac{2rCLG}{\delta (1-{\lambda})}+\frac{G^2}{2}\Big)\gamma\nonumber \\
&=\left(\frac{\|\overline{\bm{w}}^{(0)}-\bm{w}_{\bm{S}}^*\|^2}{2\gamma}
   +\frac{2rCL C_{w_0}}{\delta(1-{\lambda})}\right)\frac{1}{T}
   +\frac{2rCLG}{\delta (1-{\lambda})}\gamma
   +\frac{G^2\gamma}{2}.
\end{align}

\medskip
\paragraph{Case 2: Decreasing learning rate.}

For decreasing step sizes $\gamma_t=\frac{v}{t+1}$, substituting into Theorem~\ref{thm:optimal-convex} yields
\begin{align}
\epsilon_{\mathrm{opt}}
&\leq\frac{\|\overline{\bm{w}}^{(0)}-\bm{w}_{\bm{S}}^*\|^2}
           {2\sum_{t=0}^{T-1}\frac{v}{t+1}}
   +\frac{2 r C L C_{w_0}}
           {\delta \sum_{t=0}^{T-1}\frac{v}{t+1}}
     \sum_{t=0}^{T-1}\frac{v}{t+1} {\lambda}^{t}\nonumber 
   +\Big(\frac{2rCLG}{\delta (1-{\lambda})}+\frac{G^2}{2}\Big)
      \frac{\sum_{t=0}^{T-1}\big(\frac{v}{t+1}\big)^2}
           {\sum_{t=0}^{T-1}\frac{v}{t+1}}.\nonumber
\end{align}
Using the lower bound
\[
\tag{e}
\sum_{t=0}^{T-1}\frac{1}{t+1}
\geq \sum_{t=1}^{T-1}\frac{1}{t+1}
\geq \frac{1}{2}\sum_{t=1}^{T-1}\frac{1}{t}
\geq \frac{1}{2}\sum_{t=1}^{T-1}\int_t^{t+1}\frac{1}{x}\,dx
\geq\frac{1}{2}\ln{T},
\]
we have $\sum_{t=0}^{T-1}\frac{v}{t+1}\ge \frac{v}{2}\ln T$. Moreover,
$\sum_{t=0}^{T-1}\frac{1}{(t+1)^2}\le 2,
\sum_{t=0}^{T-1}\frac{\lambda^t}{t+1}\le \sum_{t=0}^{T-1}\lambda^t.$

Therefore,
\begin{align}
\epsilon_{\mathrm{opt}}
&\leq \frac{\|\overline{\bm{w}}^{(0)}-\bm{w}_{\bm{S}}^*\|^2}{v\ln T}
   +\frac{2 r C L C_{w_0}}{\delta}\cdot
     \frac{\sum_{t=0}^{T-1}\frac{\lambda^t}{t+1}}{\sum_{t=0}^{T-1}\frac{1}{t+1}}
   +\Big(\frac{2rCLG}{\delta (1-{\lambda})}+\frac{G^2}{2}\Big)\cdot
     \frac{v\sum_{t=0}^{T-1}\frac{1}{(t+1)^2}}{\sum_{t=0}^{T-1}\frac{1}{t+1}}\nonumber\\
&\leq \frac{\|\overline{\bm{w}}^{(0)}-\bm{w}_{\bm{S}}^*\|^2}{v\ln T}
   +\frac{4 r C L C_{w_0}}{\delta\ln T}\sum_{t=0}^{T-1}\lambda^t
   +\Big(\frac{2rCLG}{\delta (1-{\lambda})}+\frac{G^2}{2}\Big)\frac{4v}{\ln T}\nonumber\\
&\overset{(a)}\leq 
   \frac{\|\overline{\bm{w}}^{(0)}-\bm{w}_{\bm{S}}^*\|^2}{v\ln T}
   +\frac{4 r C L C_{w_0}}{\delta(1-\lambda)\ln T}
   +\Big(\frac{2rCLG}{\delta (1-{\lambda})}+\frac{G^2}{2}\Big)\frac{4v}{\ln T}\nonumber\\
&=
   \Bigg(\frac{\|\overline{\bm{w}}^{(0)}-\bm{w}_{\bm{S}}^*\|^2}{v}
   +\frac{4 r C L C_{w_0}}{\delta(1-\lambda)}
   +\frac{8vrCLG}{\delta(1-\lambda)}
   +2vG^2\Bigg)\frac{1}{\ln T}.
\end{align}

This completes the proof of Corollary~\ref{cor:opt-convex}.
\end{proof}

\subsection{Proof of Corollary~\ref{cor:exc-convex} }
\label{pro:exc-convex}
\begin{proof}

\paragraph{Case 1 :Constant learning rate.}

Suppose the learning rate is constant and satisfies $\gamma \le 2/L$.
Applying the averaged model to the stability result in Theorem~\ref{thm:stability-convex}
and using the $G$-Lipschitz assumption~\ref{ass:lipschitz}, we obtain
\begin{align}
\epsilon_{\mathrm{ave\text{-}stab}}
&\le
G\,\mathbb{E}\bigl[\|\bm{w}_{\mathrm{avg}}^{(T)}-\bm{w}_{\mathrm{avg}}^{\prime (T)}\|\bigr]\nonumber\\
&=
G\,\frac{\sum_{t=0}^{T-1}\gamma\,\mathbb{E}\bigl[\|\overline{\bm{w}}^{(t)}-\overline{\bm{w}}^{\prime(t)}\|\bigr]}
{\sum_{t=0}^{T-1}\gamma}\nonumber\\
&\leq \frac{CGL\gamma C_{w_0}}{\delta (1-{\lambda})}
+\left(\frac{CG^2L{\gamma}^2}{\delta(1-{\lambda})}+ \frac{G^2 {\gamma}}{mn}\right)T.
\end{align}

Combining this with the optimization error bound in Corollary~\ref{cor:opt-convex}, we obtain
\begin{align}
\epsilon_{\mathrm{exc}}
&\le \epsilon_{\mathrm{ave\text{-}stab}} + \epsilon_{\mathrm{opt}} \nonumber\\
&\le \left(\frac{CG^2L{\gamma}^2}{\delta(1-{\lambda})}
+ \frac{G^2 {\gamma}}{mn}\right)T
+\left(\frac{\|\overline{\bm{w}}^{(0)}-\bm{w}_{\bm{S}}^*\|^2}{2\gamma}
+\frac{2rCL C_{w_0}}{\delta (1-{\lambda})}\right)\frac{1}{T} \nonumber\\
&\quad +\frac{CGL\gamma C_{w_0}}{\delta (1-{\lambda})}
+\frac{2rCLG\gamma}{\delta (1-{\lambda})}
+\frac{G^2\gamma}{2}.
\label{eq:exc-AT-BT-C}
\end{align}

For convenience, denote
\begin{align}
A &:= \frac{CG^2L{\gamma}^2}{\delta(1-{\lambda})}
+ \frac{G^2 {\gamma}}{mn}
=G^2\gamma\Bigl(\frac{CL\gamma}{\delta(1-\lambda)}+\frac{1}{mn}\Bigr), \nonumber\\
B &:= \frac{\|\overline{\bm{w}}^{(0)}-\bm{w}_{\bm{S}}^*\|^2}{2\gamma}
+\frac{2rCL C_{w_0}}{\delta (1-{\lambda})}, \nonumber\\
C &:= \frac{CGL\gamma C_{w_0}}{\delta (1-{\lambda})}
+\frac{2rCLG\gamma}{\delta (1-{\lambda})}
+\frac{G^2\gamma}{2}.
\end{align}
Then \eqref{eq:exc-AT-BT-C} can be rewritten as
\begin{equation}
\epsilon_{\mathrm{exc}}(T)
\le A T + \frac{B}{T} + C.
\end{equation}

The right-hand side is minimized over $T>0$ by
\begin{equation}
T^* = \sqrt{\frac{B}{A}}
= \frac{1}{G\sqrt{\gamma}}\,
\sqrt{
\frac{\displaystyle
\frac{\|\overline{\bm{w}}^{(0)}-\bm{w}_{\bm{S}}^*\|^2}{2\gamma}
+ \frac{2rCL C_{w_0}}{\delta (1-\lambda)}
}{
\displaystyle
\frac{CL\gamma}{\delta(1-\lambda)} + \frac{1}{mn}
}
}.
\end{equation}

Substituting $T^*$ back into the bound and using
$AT^* + B/T^* = 2\sqrt{AB}$,
we obtain the explicit optimal excess generalization bound
\begin{align}
\epsilon_{\mathrm{exc}}^*
&\le 2\sqrt{AB} + C \nonumber\\
&= 2\sqrt{
\Bigl(
\frac{CG^2L\gamma^2}{\delta(1-\lambda)} + \frac{G^2\gamma}{mn}
\Bigr)
\Bigl(
\frac{\|\overline{\bm{w}}^{(0)}-\bm{w}_{\bm{S}}^*\|^2}{2\gamma}
+ \frac{2rCL C_{w_0}}{\delta (1-\lambda)}
\Bigr)
} \nonumber\\
&\quad
+ \frac{CGL\gamma C_{w_0}}{\delta (1-{\lambda})}
+\frac{2rCLG\gamma}{\delta (1-{\lambda})}
+\frac{G^2\gamma}{2}.
\label{eq:exc-explicit-constant-gamma}
\end{align}

This gives the excess generalization error under constant learning rate.

\medskip

\paragraph{Case 2: Decreasing learning rate.}

Let $\gamma_t=\frac{v}{t+1}$ with $v\le 2/L$. Applying the averaged model to the stability result in Corollary~\ref{cor:stability-convex} and using Assumption~\ref{ass:lipschitz}, we have
\begin{align}
\epsilon_{\mathrm{avg\text{-}stab}}
&\leq
G\,\mathbb{E}\bigl[\|\bm{w}_{\mathrm{avg}}^{(T)}-\bm{w}_{\mathrm{avg}}^{\prime (T)}\|\bigr]
\leq
G\,\frac{\sum_{t=0}^{T-1}\gamma_t\,\mathbb{E}\bigl[\|\overline{\bm{w}}^{(t)}-\overline{\bm{w}}^{\prime(t)}\|\bigr]}{\sum_{t=0}^{T-1}\gamma_t} \nonumber\\
&=
G\,\frac{\sum_{t=0}^{T-1}\gamma_t\,\mathbb{E}[\Delta_t]}{\sum_{t=0}^{T-1}\gamma_t}
\le
G\,\frac{\sum_{t=0}^{T-1}\gamma_t\cdot \frac{1}{G}\epsilon_{\mathrm{stab}}(t)}{\sum_{t=0}^{T-1}\gamma_t} \nonumber\\
&\leq
\frac{\sum_{t=0}^{T-1}\gamma_t\left(
\frac{2GCL C_{w_0}}{\delta}\sum_{k=0}^{t-1}\gamma_k\lambda^k
+\frac{2CG^2L}{\delta(1-\lambda)}\sum_{k=0}^{t-1}\gamma_k^2
+\frac{2G^2}{mn}\sum_{k=0}^{t-1}\gamma_k
\right)}{\sum_{t=0}^{T-1}\gamma_t}. \nonumber
\end{align}
Substituting $\gamma_k=\frac{v}{k+1}$, and using (a),
we obtain
\begin{align}
\epsilon_{\mathrm{avg\text{-}stab}}
&\le
\frac{\sum_{t=0}^{T-1}\frac{v}{t+1}\left(
\frac{2GCL v C_{w_0}}{\delta(1-\lambda)}
+\frac{4CG^2L v^2}{\delta(1-\lambda)}
+\frac{2G^2 v}{mn}(1+\ln t)
\right)}{\sum_{t=0}^{T-1}\frac{v}{t+1}} \nonumber\\
&\overset{(e)}{\leq}
\frac{
\frac{4G^2 v}{mn}\sum_{t=1}^{T-1}\frac{\ln t}{t+1}
+\left(
\frac{4GCL v C_{w_0}}{\delta(1-\lambda)}
+\frac{8CG^2L v^2}{\delta(1-\lambda)}
+\frac{4G^2 v}{mn}
\right)\sum_{t=0}^{T-1}\frac{1}{t+1}}
{\ln T} \nonumber\\
&\overset{(f),(b)}{\leq}
\frac{
\frac{2G^2 v}{mn}\ln^2 T
+\left(
\frac{4GCL v C_{w_0}}{\delta(1-\lambda)}
+\frac{8CG^2L v^2}{\delta(1-\lambda)}
+\frac{4G^2 v}{mn}
\right)(1+\ln T)}
{\ln T}\nonumber\\
&\leq
\frac{2G^2 v}{mn}\ln T
+\left(
\frac{4GCL v C_{w_0}}{\delta(1-\lambda)}
+\frac{8CG^2L v^2}{\delta(1-\lambda)}
+\frac{4G^2 v}{mn}
\right)\frac{1}{\ln T}\nonumber\\
&\quad+
\left(
\frac{4GCL v C_{w_0}}{\delta(1-\lambda)}
+\frac{8CG^2L v^2}{\delta(1-\lambda)}
+\frac{4G^2 v}{mn}
\right).
\end{align}
Here we used
\[
\tag{f}
\sum_{t=1}^{T-1}\frac{\ln t}{t+1}
\le \int_{1}^{T}\frac{\ln x}{x}\,dx
=\frac{\ln^2T}{2}.
\]

Thus, together with Corollary~\ref{cor:opt-convex}, the excess generalization bound becomes
\begin{align}
\epsilon_{\mathrm{exc}}
&\leq \epsilon_{\mathrm{avg\text{-}stab}}+\epsilon_{\mathrm{opt}} \nonumber\\
&\leq
\frac{2G^2v}{mn}\ln T
+\left(
\frac{4GCL v C_{w_0}}{\delta(1-\lambda)}
+\frac{8CG^2L v^2}{\delta(1-\lambda)}
+\frac{4G^2 v}{mn}
\right)\frac{1}{\ln T} \nonumber\\
&\quad+
\left(
\frac{\|\overline{\bm{w}}^{(0)}-\bm{w}_{\bm{S}}^*\|^2}{v}
+\frac{4rCL C_{w_0}}{\delta(1-\lambda)}
+\frac{8vrCLG}{\delta(1-\lambda)}
+2vG^2
\right)\frac{1}{\ln T}\nonumber\\
&\quad+
\left(
\frac{4GCL v C_{w_0}}{\delta(1-\lambda)}
+\frac{8CG^2L v^2}{\delta(1-\lambda)}
+\frac{4G^2 v}{mn}
\right).
\end{align}

We have
\begin{align}
\epsilon_{\mathrm{exc}}(T)\le A\ln T+\frac{B}{\ln T}+C,
\end{align}
where one can take
\begin{align}
A&:=\frac{2G^2v}{mn},\nonumber\\
B&:=
\left(
\frac{4GCL v C_{w_0}}{\delta(1-\lambda)}
+\frac{8CG^2L v^2}{\delta(1-\lambda)}
+\frac{4G^2 v}{mn}
\right)
+\left(
\frac{\|\overline{\bm{w}}^{(0)}-\bm{w}_{\bm{S}}^*\|^2}{v}
+\frac{4rCL C_{w_0}}{\delta(1-\lambda)}
+\frac{8vrCLG}{\delta(1-\lambda)}
+2vG^2
\right),\nonumber\\
C&:=
\left(
\frac{4GCL v C_{w_0}}{\delta(1-\lambda)}
+\frac{8CG^2L v^2}{\delta(1-\lambda)}
+\frac{4G^2 v}{mn}
\right).\nonumber
\end{align}

Viewing the right-hand side as a function of $x=\ln T>0$,
\begin{align}
\phi(x)=Ax+\frac{B}{x}+C,
\end{align}
standard calculus shows that $\phi$ is minimized at $x^*=\sqrt{B/A}$, i.e.,
\begin{align}
\ln T^*=\sqrt{\frac{B}{A}}
\qquad\Longrightarrow\qquad
T^*=\exp\!\Bigl(\sqrt{B/A}\Bigr),
\end{align}
and the corresponding minimum value satisfies
\begin{align}
\epsilon_{\mathrm{exc}}^*\le 2\sqrt{AB}+C.
\end{align}

Moreover,
\begin{align}
\sqrt{AB}
&=
\sqrt{\frac{2G^2v}{mn}\cdot B}
\le
\frac{\sqrt{2}\,G}{\sqrt{mn}}\sqrt{B},
\end{align}
so keeping the key scalings in the problem parameters, we obtain
\begin{align}
\epsilon_{\mathrm{exc}}^*
=
\mathcal{O}\Bigg(
&\frac{G}{\sqrt{mn}}
\sqrt{
\frac{\|\overline{\bm{w}}^{(0)}-\bm{w}_{\bm{S}}^*\|^2}{v}
+\frac{CL v C_{w_0}}{\delta(1-\lambda)}
+\frac{CG^2L v^2}{\delta(1-\lambda)}
+\frac{G^2 v}{mn} } \nonumber\\
&\quad
+\frac{G}{\sqrt{mn}}
\sqrt{
\frac{rCL C_{w_0}}{\delta(1-\lambda)}
+\frac{vrCLG}{\delta(1-\lambda)}
+vG^2
}
+\frac{GCL v C_{w_0}}{\delta(1-\lambda)}
+\frac{CG^2L v^2}{\delta(1-\lambda)}
+\frac{G^2 v}{mn}
\Bigg).
\end{align}

This completes the proof of Corollary~\ref{cor:exc-convex}. 
\end{proof}

\subsection{Proof of Theorem~\ref{thm:stability-nonconvex}}
\label{pro:stability-nonconvex}
\begin{proof}
Assume that the two sample sets $\bm S$ and $\bm S^{\prime}$ differ in only one sample among the first $n$ samples.
Let $\Delta_t:=\| \overline{\bm{w}}^{(t)} - \overline{\bm{w}}^{\prime(t)} \|$.

\vspace{0.5ex}
\paragraph{Case 1: Identical Samples (with probability $1-\frac{1}{n}$).}
In this case, $\bm{\xi}_i^{(t)}$ are identical for both runs, and
\begin{align}
\Delta_{t+1}
&= \Big\|\overline{\bm{w}}^{(t)}-\overline{\bm{w}}^{\prime(t)}
-\frac{\gamma_t}{m}\sum_{i=1}^{m}\Big(\nabla f(\bm{z}_i^{(t)}; \bm{\xi}_i^{(t)})
-\nabla f(\bm{z}_i^{\prime (t)}; \bm{\xi}_i^{(t)})\Big)\Big\|\nonumber\\
&\leq \Big\|\overline{\bm{w}}^{(t)}-\overline{\bm{w}}^{\prime(t)}
-\frac{\gamma_t}{m}\sum_{i=1}^{m}\Big(\nabla f(\overline{\bm{w}}^{(t)}; \bm{\xi}_i^{(t)})
-\nabla f(\overline{\bm{w}}^{\prime(t)}; \bm{\xi}_i^{(t)})\Big)\Big\|\nonumber\\
&\quad+\Big\|\frac{\gamma_t}{m}\sum_{i=1}^{m}\Big(\nabla f(\overline{\bm{w}}^{(t)}; \bm{\xi}_i^{(t)})
-\nabla f(\bm{z}_i^{(t)}; \bm{\xi}_i^{(t)})\Big)\Big\|\nonumber\\
&\quad+\Big\|\frac{\gamma_t}{m}\sum_{i=1}^{m}\Big(\nabla f(\overline{\bm{w}}^{\prime(t)}; \bm{\xi}_i^{(t)})
-\nabla f(\bm{z}_i^{\prime (t)}; \bm{\xi}_i^{(t)})\Big)\Big\|\nonumber\\
&\overset{\text{Lem}(\ref{lem:hardt}),\,\text{Ass}(\ref{ass:smooth})}{\leq}
(1+L\gamma_t)\Delta_t
+\frac{L\gamma_t}{m}\sum_{i=1}^{m}\|\overline{\bm{w}}^{(t)}-\bm{z}_{i}^{(t)}\|
+\frac{L\gamma_t}{m}\sum_{i=1}^{m}\|\overline{\bm{w}}^{\prime(t)}-\bm{z}_{i}^{\prime (t)}\|.
\label{eq:nc_case1_pre}
\end{align}
Applying Lemma~\ref{lem:push_sum_consistency} to both runs and using the same $C_{w_0}$,
\begin{align}
\frac{1}{m}\sum_{i=1}^{m}\|\overline{\bm{w}}^{(t)}-\bm{z}_{i}^{(t)}\|
+\frac{1}{m}\sum_{i=1}^{m}\|\overline{\bm{w}}^{\prime(t)}-\bm{z}_{i}^{\prime (t)}\|
\le \frac{2C}{\delta}\Big(\lambda^t C_{w_0} + G\sum_{s=0}^{t-1}\lambda^{t-s}\gamma_s\Big).
\label{eq:nc_pushsum_avg}
\end{align}
Substituting \eqref{eq:nc_pushsum_avg} into \eqref{eq:nc_case1_pre} gives
\begin{align}
\Delta_{t+1}
\le (1+L\gamma_t)\Delta_t
+\frac{2CL\gamma_t}{\delta}\Big(\lambda^t C_{w_0} + G\sum_{s=0}^{t-1}\lambda^{t-s}\gamma_s\Big).
\label{eq:nc_case1}
\end{align}

\vspace{0.5ex}
\paragraph{Case 2: Different Sample (with probability $\frac{1}{n}$).}
Without loss of generality, assume node $1$ uses the replaced sample in the primed run, i.e., $\bm{\xi}_1^{\prime(t)}$.
Then
\begin{align}
\Delta_{t+1}
&\le
\Big\|\overline{\bm{w}}^{(t)}-\overline{\bm{w}}^{\prime(t)}
-\frac{\gamma_t}{m}\sum_{i=2}^{m}\Big(\nabla f(\bm{z}_i^{(t)}; \bm{\xi}_i^{(t)})
-\nabla f(\bm{z}_i^{\prime (t)}; \bm{\xi}_i^{(t)})\Big)\Big\|\nonumber\\
&\quad+\frac{\gamma_t}{m}\Big\|\nabla f(\bm{z}_1^{(t)}; \bm{\xi}_1^{(t)})
-\nabla f(\bm{z}_1^{\prime (t)}; \bm{\xi}_1^{\prime(t)})\Big\|\nonumber\\
&\le
\left(1+L\gamma_t-\frac{L\gamma_t}{mn}\right)\Delta_t
+\frac{2CL\gamma_t}{\delta}\Big(\lambda^t C_{w_0} + G\sum_{s=0}^{t-1}\lambda^{t-s}\gamma_s\Big)
+\frac{2G\gamma_t}{m}.
\label{eq:nc_case2}
\end{align}
The last term uses $\|\nabla f\|\le G$.

\vspace{0.5ex}
\paragraph{Expectation recursion.}
Taking expectation and using $\mathbb{P}(\text{Case 1})=1-\frac{1}{n}$ and $\mathbb{P}(\text{Case 2})=\frac{1}{n}$, from
\eqref{eq:nc_case1}--\eqref{eq:nc_case2} we obtain
\begin{align}
\mathbb{E}[\Delta_{t+1}]
&\le \left(1+L\gamma_t-\frac{L\gamma_t}{mn}\right)\mathbb{E}[\Delta_t]
+\frac{2CL\gamma_t}{\delta}\Big(\lambda^t C_{w_0} + G\sum_{s=0}^{t-1}\lambda^{t-s}\gamma_s\Big)
+\frac{2G\gamma_t}{mn}.
\label{eq:nc_rec}
\end{align}

Assume $t_0\in\{1,2,\dots,n\}$ to be determined, and condition on $\Delta_{t_0}=0$.
Unrolling \eqref{eq:nc_rec} yields
\begin{align}
\mathbb{E}[\Delta_T\mid \Delta_{t_0}=0]
&\le \sum_{t=t_0}^{T-1}\prod_{k=t+1}^{T-1}\left(1+L\gamma_k-\frac{L\gamma_k}{mn}\right)\nonumber\\
&\qquad\qquad\times\left[
\frac{2CL\gamma_t}{\delta}\Big(\lambda^t C_{w_0} + G\sum_{s=0}^{t-1}\lambda^{t-s}\gamma_s\Big)
+\frac{2G\gamma_t}{mn}\right].
\label{eq:nc_unroll}
\end{align}

Finally, the uniform stability under non-convex loss functions satisfies

\begin{align}
\epsilon_{\mathrm{stab}}
&\le \frac{t_0}{mn}
+G\sum_{t=t_0}^{T-1}\prod_{k=t+1}^{T-1}\left(1+L\gamma_k-\frac{L\gamma_k}{mn}\right)
\left[
\frac{2CL\gamma_t}{\delta}\Big(\lambda^t C_{w_0} + G\sum_{s=0}^{t-1}\lambda^{t-s}\gamma_s\Big)
+\frac{2G\gamma_t}{mn}\right]\nonumber\\
&\overset{(a)}{\le}
\frac{t_0}{mn}
+G\sum_{t=t_0}^{T-1}\prod_{k=t+1}^{T-1}\left(1+L\gamma_k-\frac{L\gamma_k}{mn}\right)
\left[
\frac{2CL\gamma_t}{\delta}\Big(\lambda^t C_{w_0} + \frac{G\gamma_t}{1-\lambda}\Big)
+\frac{2G\gamma_t}{mn}\right]
\label{eq:nonconvex-stab}
\end{align}

This completes the proof of Theorem~\ref{thm:stability-nonconvex}.
\end{proof}

\subsection{Proof of Corollary~\ref{cor:stability-nonconvex}}
\label{pro:stability-nonconvex-cor}

\begin{proof}

\paragraph{Case 1: Constant Learning Rate.}

For constant learning rates $\gamma_t=\gamma$, from Eq.~\ref{eq:nonconvex-stab}, we obtain
\small
{
\begin{align}
\mathbb{E}[\Delta_T\mid\Delta_{t_0}=0]
&\leq\sum_{t=t_0}^{T-1}\prod_{k=t+1}^{T-1}\left(1+L \gamma-\frac{L \gamma}{mn}\right)\times\left(
     \frac{2CL\gamma}{\delta}\Big(\lambda^t C_{w_0}+G\sum_{s=0}^{t-1}\lambda^{t-s}\gamma\Big)
     + \frac{2G\gamma}{mn} \right)\nonumber \\
&\leq\frac{2CL\gamma C_{w_0}}{\delta }\sum_{t=t_0}^{T-1}\left(1+L \gamma-\frac{L \gamma}{mn}\right)^{T-1-t}{\lambda}^t
 +\left(\frac{2CG L\gamma^2}{\delta(1-{\lambda})}+ \frac{2G\gamma}{mn}\right)\sum_{t=t_0}^{T-1}\left(1+L \gamma-\frac{L \gamma}{mn}\right)^{T-1-t}\nonumber \\
&\overset{(a)}\leq 
\frac{2CL\gamma C_{w_0}}{\delta }\left(1+L \gamma-\frac{L \gamma}{mn}\right)^{T-1}\frac{1+L \gamma-\frac{L \gamma}{mn}}{1+L \gamma-\frac{L \gamma}{mn}-{\lambda}}\nonumber \\
&\quad+\left(\frac{2CG L\gamma^2}{\delta (1-{\lambda})}+ \frac{2G\gamma}{mn}\right)\left(\frac{mn}{mn-1}\right)\left[\left(1+L \gamma-\frac{L \gamma}{mn}\right)^{T}-1\right]\nonumber\\
&\leq \frac{2CL\gamma C_{w_0}}{\delta (1-{\lambda})}\left(1+L \gamma-\frac{L \gamma}{mn}\right)^{T}
+\left(\frac{2CG L\gamma^2}{\delta (1-{\lambda})}+ \frac{2G\gamma}{mn}\right)\left(\frac{mn}{mn-1}\right)\left(1+L \gamma-\frac{L \gamma}{mn}\right)^{T}\nonumber\\
&\leq \left(\frac{2CL\gamma C_{w_0}}{\delta (1-{\lambda})}
+\frac{4CG L\gamma^2}{\delta (1-{\lambda})}
+ \frac{4G\gamma}{mn}\right)\left(1+L \gamma-\frac{L \gamma}{mn}\right)^{T}.
\end{align}
}
When using bound (a), we rely on the fact that
$\frac{{\lambda}}{(1+L\gamma-\frac{L\gamma}{mn})}\leq 1.$

Thus, the uniform stability satisfies
\begin{align}
\epsilon_{\mathrm{stab}}
&\leq \frac{t_0}{mn}
+ G\,\mathbb{E}[\Delta_T\mid\Delta_{t_0}=0]\nonumber\\
&\leq\frac{t_0}{mn}
+G\left(\frac{2CL\gamma C_{w_0}}{\delta (1-{\lambda})}
+\frac{4CG L\gamma^2}{\delta (1-{\lambda})}
+ \frac{4G\gamma}{mn}\right)\left(1+L \gamma-\frac{L \gamma}{mn}\right)^{T}.
\end{align}
Letting $t_0=0$ gives the minimal bound in this case.

\medskip

\paragraph{Case 2: Decreasing Learning Rate.}

From Eq.~\ref{eq:nonconvex-stab}, for diminishing learning rates $\gamma_t=\frac{v}{t+1}$, we have
\small
{
\begin{align}
\mathbb{E}[\Delta_T\mid\Delta_{t_0}=0]
&\leq\sum_{t=t_0}^{T-1}\prod_{k=t+1}^{T-1}\left(1+\left(1-\frac{1}{mn}\right)\frac{Lv}{k+1}\right)\times\Bigg(
\frac{2CLv\lambda^t C_{w_0}}{\delta (t+1)}
+\frac{2CG L v^2}{\delta (1-{\lambda})(t+1)^2}
+ \frac{2Gv}{mn(t+1)} \Bigg)\nonumber \\
&\overset{(g)}\leq\sum_{t=t_0}^{T-1}\exp\!\left(\left(1-\frac{1}{mn}\right)Lv\sum_{k=t+1}^{T-1}\frac{1}{k+1}\right)\times\Bigg(
\frac{2CLv\lambda^t C_{w_0}}{\delta (t+1)}
+\frac{2CG L v^2}{\delta (1-{\lambda})(t+1)^2}
+ \frac{2Gv}{mn(t+1)} \Bigg)\nonumber \\
&\overset{(h)}\leq\sum_{t=t_0}^{T-1}\left(\frac{T}{t+1}\right)^{Lv(1-\frac{1}{mn})}\times\Bigg(
\frac{2CLv\lambda^t C_{w_0}}{\delta (t+1)}
+\frac{2CG L v^2}{\delta (1-{\lambda})(t+1)^2}
+ \frac{2Gv}{mn(t+1)} \Bigg)\nonumber \\
&\leq T^{Lv(1-\frac{1}{mn})}\sum_{t=t_0}^{T-1}\Bigg(
\frac{2CLv\lambda^t C_{w_0}}{\delta (t+1)^{Lv(1-\frac{1}{mn})+1}}
+\frac{2CG L v^2}{\delta (1-{\lambda})(t+1)^{Lv(1-\frac{1}{mn})+2}}
+ \frac{2Gv}{mn(t+1)^{Lv(1-\frac{1}{mn})+1}} \Bigg)\nonumber \\
&\leq T^{Lv(1-\frac{1}{mn})}\Bigg[
\frac{2CLv C_{w_0}}{\delta}\sum_{t=t_0}^{T-1}(t+1)^{-Lv(1-\frac{1}{mn})-1}
+\frac{2Gv}{mn}\sum_{t=t_0}^{T-1}(t+1)^{-Lv(1-\frac{1}{mn})-1}\nonumber\\
&\qquad\qquad\qquad
+\frac{2CG L v^2}{\delta(1-\lambda)}\sum_{t=t_0}^{T-1}(t+1)^{-Lv(1-\frac{1}{mn})-2}
\Bigg]\nonumber \\
&\overset{(i)(j)}\leq 
T^{Lv(1-\frac{1}{mn})}\Bigg[
\left(\frac{4CLv C_{w_0}}{\delta}+\frac{4Gv}{mn}\right)\frac{t_0^{-Lv(1-\frac{1}{mn})}}{Lv}
+\frac{2CG L v^2}{\delta(1-\lambda)}\cdot \frac{t_0^{-Lv(1-\frac{1}{mn})-1}}{Lv(1-\frac{1}{mn})+1}
\Bigg]\nonumber\\
&\leq
\left(\frac{4C C_{w_0}}{\delta}+\frac{4G}{mnL}\right)\left(\frac{T}{t_0}\right)^{Lv(1-\frac{1}{mn})}
+\frac{2CG v}{\delta(1-\lambda)}\left(\frac{T}{t_0}\right)^{Lv(1-\frac{1}{mn})+1}\frac{1}{T}.
\end{align}
}

Finally,
setting 
$t_0 = v^{\frac{1}{2+vL}}\,T^{\frac{1+vL}{2+vL}},$
when $v$ is small enough such that $t_0\le mn$, we have
\begin{align}
\left(\frac{T}{t_0}\right)^{vL+1}
&= \left(\frac{T}{v^{\frac{1}{2+vL}}\,T^{\frac{1+vL}{2+vL}}}\right)^{vL+1}
= v^{-\frac{vL+1}{2+vL}}\,T^{\frac{vL+1}{2+vL}},
\end{align}
and
\begin{align}
\epsilon_{\mathrm{stab}}
&\le \frac{t_0}{mn} + G\,\mathbb{E}[\Delta_T\mid \Delta_{t_0}=0]\nonumber\\
&\le \frac{v^{\frac{1}{2+vL}}}{mn}\,T^{\frac{1+vL}{2+vL}}
+G\left(\frac{4CmnC_{w_0}+4\delta G}{\delta mn}
+\frac{2CGv}{\delta(1-\lambda)}\right)\left(\frac{T}{t_0}\right)^{vL+1}\nonumber\\
&= \frac{v^{\frac{1}{2+vL}}}{mn}\,T^{\frac{1+vL}{2+vL}}
+G\left(\frac{4CmnC_{w_0}+4\delta G}{\delta mn}
+\frac{2CGv}{\delta(1-\lambda)}\right)\,
v^{-\frac{vL+1}{2+vL}}\,T^{\frac{vL+1}{2+vL}}.
\end{align}

This completes the proof of Corollary~\ref{cor:stability-nonconvex}.
\end{proof}

\subsection{Proof of Theorem~\ref{thm:opt-nonconvex}}
\label{pro:optimal-nonconvex}
\begin{proof}
Let $F_{\mathcal S}$ be $L$-smooth. By the smoothness descent inequality, for any $t\ge 0$,
\begin{align}
\mathbb E\!\left[F_{\mathcal S}(\overline{\bm w}^{(t+1)})-F_{\mathcal S}(\overline{\bm w}^{(t)})\right]
&\le
\mathbb E\!\left\langle \nabla F_{\mathcal S}(\overline{\bm w}^{(t)}),\,\overline{\bm w}^{(t+1)}-\overline{\bm w}^{(t)}\right\rangle
+\frac{L}{2}\mathbb E\!\left\|\overline{\bm w}^{(t+1)}-\overline{\bm w}^{(t)}\right\|^2. \label{eq:noncvx_smooth_step}
\end{align}
Using Proposition~\ref{prop:update},
\begin{align}
\mathbb E\!\left[F_{\mathcal S}(\overline{\bm w}^{(t+1)})-F_{\mathcal S}(\overline{\bm w}^{(t)})\right]
&\le
-\frac{\gamma_t}{m}\,
\mathbb E\!\left\langle \nabla F_{\mathcal S}(\overline{\bm w}^{(t)}),\,\sum_{i=1}^m \nabla f(\bm z_i^{(t)};\bm\xi_i^{(t)})\right\rangle
+\frac{L\gamma_t^2}{2}\,
\mathbb E\!\left\|\frac{1}{m}\sum_{i=1}^m \nabla f(\bm z_i^{(t)};\bm\xi_i^{(t)})\right\|^2 \nonumber\\
&=
-\gamma_t\,\mathbb E\!\left\|\nabla F_{\mathcal S}(\overline{\bm w}^{(t)})\right\|^2
+\frac{\gamma_t}{m}\,
\mathbb E\!\left\langle \nabla F_{\mathcal S}(\overline{\bm w}^{(t)}),\,
\sum_{i=1}^m\Big(\nabla f(\overline{\bm w}^{(t)};\bm\xi_i^{(t)})-\nabla f(\bm z_i^{(t)};\bm\xi_i^{(t)})\Big)\right\rangle \nonumber\\
&\quad
+\frac{L\gamma_t^2}{2}\,
\mathbb E\!\left\|\frac{1}{m}\sum_{i=1}^m \nabla f(\bm z_i^{(t)};\bm\xi_i^{(t)})\right\|^2. \label{eq:noncvx_expand}
\end{align}
For the middle inner-product term, apply Cauchy--Schwarz and the uniform gradient bound $\|\nabla f(\cdot)\|\le G$, which implies $\|\nabla F_{\mathcal S}(\cdot)\|\le G$:
\begin{align}
\frac{\gamma_t}{m}\,
&\mathbb E\!\left\langle \nabla F_{\mathcal S}(\overline{\bm w}^{(t)}),\,
\sum_{i=1}^m\Big(\nabla f(\overline{\bm w}^{(t)};\bm\xi_i^{(t)})-\nabla f(\bm z_i^{(t)};\bm\xi_i^{(t)})\Big)\right\rangle\nonumber\\
&\le
\frac{\gamma_t}{m}\sum_{i=1}^m
\mathbb E\!\left[\|\nabla F_{\mathcal S}(\overline{\bm w}^{(t)})\|\cdot
\big\|\nabla f(\overline{\bm w}^{(t)};\bm\xi_i^{(t)})-\nabla f(\bm z_i^{(t)};\bm\xi_i^{(t)})\big\|\right] \nonumber\\
&\le
\frac{G\gamma_t}{m}\sum_{i=1}^m
\mathbb E\!\left[\big\|\nabla f(\overline{\bm w}^{(t)};\bm\xi_i^{(t)})-\nabla f(\bm z_i^{(t)};\bm\xi_i^{(t)})\big\|\right] \nonumber\\
&\le
\frac{GL\gamma_t}{m}\sum_{i=1}^m
\mathbb E\!\left\|\overline{\bm w}^{(t)}-\bm z_i^{(t)}\right\|, \label{eq:noncvx_inner_bd}
\end{align}
where the last step uses $L$-smoothness of $f(\cdot;\xi)$ (gradient $L$-Lipschitz).
For the last term in \eqref{eq:noncvx_expand}, Jensen's inequality and $\|\nabla f(\cdot)\|\le G$ give
\begin{align}
\frac{L\gamma_t^2}{2}\,
\mathbb E\!\left\|\frac{1}{m}\sum_{i=1}^m \nabla f(\bm z_i^{(t)};\bm\xi_i^{(t)})\right\|^2
\le \frac{L\gamma_t^2}{2}\,\mathbb E\!\left[\frac{1}{m}\sum_{i=1}^m \|\nabla f(\bm z_i^{(t)};\bm\xi_i^{(t)})\|^2\right]
\le \frac{LG^2\gamma_t^2}{2}. \label{eq:noncvx_sq_bd}
\end{align}
Substituting \eqref{eq:noncvx_inner_bd} and \eqref{eq:noncvx_sq_bd} into \eqref{eq:noncvx_expand}, we obtain
\begin{align}
\mathbb E\!\left[F_{\mathcal S}(\overline{\bm w}^{(t+1)})-F_{\mathcal S}(\overline{\bm w}^{(t)})\right]
&\le
-\gamma_t\,\mathbb E\!\left\|\nabla F_{\mathcal S}(\overline{\bm w}^{(t)})\right\|^2
+\frac{GL\gamma_t}{m}\sum_{i=1}^m\mathbb E\!\left\|\overline{\bm w}^{(t)}-\bm z_i^{(t)}\right\|
+\frac{LG^2\gamma_t^2}{2}. \label{eq:noncvx_prePL}
\end{align}
By the P{\L} condition, $\|\nabla F_{\mathcal S}(w)\|^2 \ge 2\alpha\big(F_{\mathcal S}(w)-F_{\mathcal S}(w_{\mathcal S}^*)\big)$, hence
\begin{align}
\mathbb E\!\left[F_{\mathcal S}(\overline{\bm w}^{(t+1)})-F_{\mathcal S}(\overline{\bm w}^{(t)})\right]
&\le
-2\alpha\gamma_t\,\mathbb E\!\left[F_{\mathcal S}(\overline{\bm w}^{(t)})-F_{\mathcal S}(\bm w_{\mathcal S}^*)\right]
+\frac{GL\gamma_t}{m}\sum_{i=1}^m\mathbb E\!\left\|\overline{\bm w}^{(t)}-\bm z_i^{(t)}\right\|
+\frac{LG^2\gamma_t^2}{2}. \label{eq:noncvx_PL}
\end{align}

Now apply Lemma~\ref{lem:push_sum_consistency},
Thus,
\begin{align}
\frac{GL\gamma_t}{m}\sum_{i=1}^m\mathbb E\!\left\|\overline{\bm w}^{(t)}-\bm z_i^{(t)}\right\|
&\le
GL\gamma_t\cdot \frac{C}{\delta}\left(\lambda^t C_{w_0} + G\sum_{s=0}^{t-1}\lambda^{t-s}\gamma_s\right) \nonumber\\
&=
\frac{CGL\gamma_t\lambda^t C_{w_0}}{\delta}
+\frac{CG^2L\gamma_t}{\delta}\sum_{s=0}^{t-1}\lambda^{t-s}\gamma_s. \label{eq:noncvx_cons_term}
\end{align}
Assuming the stepsizes are nonincreasing, we have for $s\le t-1$ that $\gamma_s\ge \gamma_t$, and hence
\[
\sum_{s=0}^{t-1}\lambda^{t-s}\gamma_s
\le
\sum_{s=0}^{t-1}\lambda^{t-s}\gamma_t
=
\gamma_t\sum_{k=1}^{t}\lambda^{k}
\le \frac{\gamma_t}{1-\lambda}.
\]
Substituting this into \eqref{eq:noncvx_cons_term} and then into \eqref{eq:noncvx_PL} yields
\begin{align}
\mathbb E\!\left[F_{\mathcal S}(\overline{\bm w}^{(t+1)})-F_{\mathcal S}(\overline{\bm w}^{(t)})\right]
&\le
-2\alpha\gamma_t\,\mathbb E\!\left[F_{\mathcal S}(\overline{\bm w}^{(t)})-F_{\mathcal S}(\bm w_{\mathcal S}^*)\right]\nonumber\\
&+\frac{CGL\gamma_t\lambda^t C_{w_0}}{\delta}
+\frac{CG^2L}{\delta(1-\lambda)}\gamma_t^2
+\frac{LG^2}{2}\gamma_t^2. \label{eq:noncvx_final_rec}
\end{align}

Rearranging \eqref{eq:noncvx_final_rec} gives
\begin{align}
\gamma_t\,\mathbb E\!\left[F_{\mathcal S}(\overline{\bm w}^{(t)})-F_{\mathcal S}(\bm w_{\mathcal S}^*)\right]
&\le
\frac{1}{2\alpha}\,\mathbb E\!\left[F_{\mathcal S}(\overline{\bm w}^{(t)})-F_{\mathcal S}(\overline{\bm w}^{(t+1)})\right]\nonumber\\
&+\frac{CGL}{2\delta\alpha}\gamma_t\lambda^t C_{w_0}
+\left(\frac{CG^2L}{2\delta\alpha(1-\lambda)}+\frac{LG^2}{4\alpha}\right)\gamma_t^2. \label{eq:noncvx_weight_gap}
\end{align}
Summing \eqref{eq:noncvx_weight_gap} over $t=0,\dots,T-1$ yields
\begin{align}
\sum_{t=0}^{T-1}\gamma_t\,\mathbb E\!\left[F_{\mathcal S}(\overline{\bm w}^{(t)})-F_{\mathcal S}(\bm w_{\mathcal S}^*)\right]
&\le
\frac{1}{2\alpha}\,\mathbb E\!\left[F_{\mathcal S}(\overline{\bm w}^{(0)})-F_{\mathcal S}(\overline{\bm w}^{(T)})\right]
+\frac{CGL C_{w_0}}{2\delta\alpha}\sum_{t=0}^{T-1}\gamma_t\lambda^t \nonumber\\
&\quad+
\left(\frac{CG^2L}{2\delta\alpha(1-\lambda)}+\frac{LG^2}{4\alpha}\right)\sum_{t=0}^{T-1}\gamma_t^2. \label{eq:noncvx_sum}
\end{align}
Using the Lipschitz conditionto upper bound the initial function gap,
\[
\mathbb E\!\left[F_{\mathcal S}(\overline{\bm w}^{(0)})-F_{\mathcal S}(\overline{\bm w}^{(T)})\right]
\le
\mathbb E\!\left[F_{\mathcal S}(\overline{\bm w}^{(0)})-F_{\mathcal S}(\bm w_{\mathcal S}^*)\right]
\le Gr,
\]
and dividing \eqref{eq:noncvx_sum} by $\sum_{t=0}^{T-1}\gamma_t$ gives
\begin{align}
\epsilon_{\mathrm{opt}}
&=
\frac{\sum_{t=0}^{T-1}\gamma_t\,\mathbb E\!\left[F_{\mathcal S}(\overline{\bm w}^{(t)})-F_{\mathcal S}(\bm w_{\mathcal S}^*)\right]}
{\sum_{t=0}^{T-1}\gamma_t} \nonumber\\
&\le
\frac{Gr}{\alpha\sum_{t=0}^{T-1}\gamma_t}
+\frac{CGL C_{w_0}}{2\delta\alpha\,\sum_{t=0}^{T-1}\gamma_t}\sum_{t=0}^{T-1}\gamma_t\lambda^t
+\left(\frac{CG^2L}{2\delta\alpha(1-\lambda)}+\frac{LG^2}{4\alpha}\right)
\frac{\sum_{t=0}^{T-1}\gamma_t^2}{\sum_{t=0}^{T-1}\gamma_t}. \label{eq:noncvx_epsopt_final}
\end{align}
Finally, letting $\kappa:=L/\alpha$ completes the proof.
\end{proof}

\subsection{Proof of Corollary~\ref{cor:opt-nonconvex}}
\label{pro:opt-nonconvex}
\begin{proof}

\paragraph{Case 1: Constant Learning Rate.}

For the case of constant learning rate where $\gamma_t=\gamma$ for all $t$,
the optimization error bound in Theorem~\ref{thm:opt-nonconvex} becomes
\begin{align}
\epsilon_{\mathrm{opt}}
&\leq\frac{Gr}{\alpha\sum_{t=0}^{T-1}\gamma}
+\frac{CG\kappa C_{w_0}}{2\delta\sum_{t=0}^{T-1}\gamma}
\sum_{t=0}^{T-1}\gamma {\lambda}^{t}
+\left(\frac{CG^2\kappa}{2\delta(1-{\lambda})}
+\frac{G^2\kappa}{4}\right)
\frac{\sum_{t=0}^{T-1}\gamma^2}{\sum_{t=0}^{T-1}\gamma}\nonumber \\
&=\frac{Gr}{T\alpha\gamma}
+\frac{CG\kappa C_{w_0}}{2\delta T\gamma}
\sum_{t=0}^{T-1}\gamma {\lambda}^{t}
+\left(\frac{CG^2\kappa}{2\delta(1-{\lambda})}
+\frac{G^2\kappa}{4}\right)\gamma\nonumber \\
&=\frac{Gr}{T\alpha\gamma}
+\frac{CG\kappa C_{w_0}}{2\delta T}
\sum_{t=0}^{T-1}{\lambda}^{t}
+\left(\frac{CG^2\kappa}{2\delta(1-{\lambda})}
+\frac{G^2\kappa}{4}\right)\gamma\nonumber \\
&\overset{(a)}\leq
\left(\frac{Gr}{\alpha\gamma}
+\frac{CG\kappa C_{w_0}}{2\delta(1-{\lambda})}\right)\frac{1}{T}
+\left(\frac{CG^2\kappa}{2\delta(1-{\lambda})}
+\frac{G^2\kappa}{4}\right)\gamma .
\end{align}

\medskip

\paragraph{Case 2: Decreasing Learning Rate.}

For the case of decreasing learning rates, \ie, $\gamma_t=\frac{v}{t+1}$,
we again start from the bound in Theorem~\ref{thm:opt-nonconvex}:
\begin{align}
\epsilon_{\mathrm{opt}}
&\leq\frac{Gr}{\alpha\sum_{t=0}^{T-1}\frac{v}{t+1}}
+\frac{CG\kappa C_{w_0}}{2\delta\sum_{t=0}^{T-1}\frac{v}{t+1}}
\sum_{t=0}^{T-1}\frac{v}{t+1}{\lambda}^{t}\nonumber \\
&\quad
+\left(\frac{CG^2\kappa}{2\delta(1-{\lambda})}
+\frac{G^2\kappa}{4}\right)
\frac{\sum_{t=0}^{T-1}\big(\frac{v}{t+1}\big)^2}
{\sum_{t=0}^{T-1}\frac{v}{t+1}}\nonumber \\
&=\frac{Gr}{\alpha v\sum_{t=0}^{T-1}\frac{1}{t+1}}
+\frac{CG\kappa C_{w_0}}{2\delta v\sum_{t=0}^{T-1}\frac{1}{t+1}}
\sum_{t=0}^{T-1}\frac{v}{t+1}{\lambda}^{t}\nonumber \\
&\quad
+\left(\frac{CG^2\kappa}{2\delta(1-{\lambda})}
+\frac{G^2\kappa}{4}\right)
v\frac{\sum_{t=0}^{T-1}\frac{1}{(t+1)^2}}{\sum_{t=0}^{T-1}\frac{1}{t+1}}.\nonumber
\end{align}
Using the standard lower bound for the harmonic series, there exists a constant
$T_0$ such that for all $T\ge T_0$,
\[
\sum_{t=0}^{T-1}\frac{1}{t+1}\ \ge\ \frac{1}{2}\ln T,
\qquad\text{thus}\qquad
\frac{1}{\sum_{t=0}^{T-1}\frac{1}{t+1}}\ \le\ \frac{2}{\ln T}.
\]
Applying this to the first two terms and absorbing constants into the third term, we obtain
\begin{align}
\epsilon_{\mathrm{opt}}
&\overset{(e)}\leq
\frac{2Gr}{v\alpha\ln{T}}
+\frac{CG\kappa C_{w_0}}{\delta\ln{T}}
\sum_{t=0}^{T-1}\frac{{\lambda}^{t}}{t+1}
+\left(\frac{CG^2\kappa}{\delta(1-{\lambda})}
+\frac{G^2\kappa}{2}\right)
\frac{v\sum_{t=0}^{T-1}\frac{1}{(t+1)^2}}{\ln{T}}\nonumber \\
&\overset{(c)}\leq
\frac{2Gr}{v\alpha\ln{T}}
+\frac{CG\kappa C_{w_0}}{\delta\ln{T}}
\sum_{t=0}^{T-1}\frac{{\lambda}^{t}}{t+1}
+\left(\frac{2CG^2\kappa}{\delta(1-{\lambda})}
+\frac{G^2\kappa}{\alpha}\cdot\frac{\alpha}{1}\right)\frac{v}{\ln{T}},\nonumber
\end{align}
where in step~(c) we used $\sum_{t=0}^{\infty}\frac{1}{(t+1)^2}\leq 2$.
Finally, since $\frac{1}{t+1}\le 1$ and $\sum_{t=0}^{T-1}\lambda^t\le \frac{1}{1-\lambda}$, we have
\begin{align}
\sum_{t=0}^{T-1}\frac{\lambda^t}{t+1}\le \sum_{t=0}^{T-1}\lambda^t \overset{(a)}\le \frac{1}{1-\lambda},
\end{align}
and therefore
\begin{align}
\epsilon_{\mathrm{opt}}
&\overset{(a)}\leq
\left(\frac{CG\kappa C_{w_0}}{\delta(1-{\lambda})}
+\frac{2Gr}{v\alpha}
+\left(\frac{2CG^2\kappa}{\delta(1-{\lambda})}
+\frac{G^2\kappa}{2}\right)v\right)\frac{1}{\ln{T}}.
\end{align}

This completes the proof of Corollary~\ref{cor:opt-nonconvex}.
\end{proof}

\subsection{Proof of Corollary~\ref{cor:exc-nonconvex}}
\label{pro:exc-nonconvex}
\begin{proof}

\paragraph{Case 1: Constant learning rate.}

Let $\gamma_t=\gamma$ and denote
\[
\rho := 1+L\gamma-\frac{L\gamma}{mn}, \qquad \mu:=\ln\rho.
\]
Recall from Corollary~\ref{cor:stability-nonconvex} (constant stepsize case) that
\begin{equation}\label{eq:noncvx_stab_const_recall}
\epsilon_{\mathrm{stab}}(t)
\le
\Bigg(
\frac{2GCL\gamma C_{w_0}}{\delta (1-\lambda)}
+\frac{4CG^2L\gamma^{2}}{\delta(1-\lambda)}
+\frac{4G^2\gamma}{mn}
\Bigg)\rho^{t},
\end{equation}
Since $\epsilon_{\mathrm{stab}}(t)\le G\,\mathbb{E}\|\overline{\bm{w}}^{(t)}-\overline{\bm{w}}^{\prime(t)}\|
=G\,\mathbb{E}[\Delta_t]$, \eqref{eq:noncvx_stab_const_recall} implies
\begin{align}
\mathbb{E}[\Delta_t]
&\le
\Bigg(
\frac{2CL\gamma C_{w_0}}{\delta (1-\lambda)}
+\frac{4CG^2L\gamma^{2}}{G\,\delta(1-\lambda)}
+\frac{4G\gamma}{mn}
\Bigg)\rho^{t} \nonumber\\
&=
\Bigg(
\frac{2CL\gamma C_{w_0}}{\delta (1-\lambda)}
+\frac{4CGL\gamma^{2}}{\delta(1-\lambda)}
+\frac{4G\gamma}{mn}
\Bigg)\rho^{t}.
\label{eq:Delta_t_const_bound}
\end{align}

Applying the averaged model and Assumption~\ref{ass:lipschitz} (the loss is $G$-Lipschitz) yields
\begin{align}
\epsilon_{\mathrm{ave\text{-}stab}}
&\le G\,\mathbb{E}\big\|\bm{w}_{\mathrm{avg}}^{(T)}-\bm{w}_{\mathrm{avg}}^{\prime (T)}\big\|
=G\,\frac{\sum_{t=0}^{T-1}\gamma\,\mathbb{E}[\Delta_t]}{\sum_{t=0}^{T-1}\gamma}
= \frac{G}{T}\sum_{t=0}^{T-1}\mathbb{E}[\Delta_t]\nonumber\\
&\overset{\eqref{eq:Delta_t_const_bound}}\le
\frac{G}{T}\Bigg(
\frac{2CL\gamma C_{w_0}}{\delta (1-\lambda)}
+\frac{4CGL\gamma^{2}}{\delta(1-\lambda)}
+\frac{4G\gamma}{mn}
\Bigg)\sum_{t=0}^{T-1}\rho^{t}\nonumber\\
&=
\frac{1}{T}\Bigg(
\frac{2GCL\gamma C_{w_0}}{\delta (1-\lambda)}
+\frac{4CG^2L\gamma^{2}}{\delta(1-\lambda)}
+\frac{4G^2\gamma}{mn}
\Bigg)\sum_{t=0}^{T-1}\rho^{t}\nonumber\\
&=
\frac{1}{T}\Bigg(
\frac{2GCL\gamma C_{w_0}}{\delta (1-\lambda)}
+\frac{4CG^2L\gamma^{2}}{\delta(1-\lambda)}
+\frac{4G^2\gamma}{mn}
\Bigg)\cdot\frac{\rho^{T}-1}{\rho-1}\nonumber\\
&\le
\frac{1}{T}\Bigg(
\frac{2GCL\gamma C_{w_0}}{\delta (1-\lambda)}
+\frac{4CG^2L\gamma^{2}}{\delta(1-\lambda)}
+\frac{4G^2\gamma}{mn}
\Bigg)\cdot\frac{\rho^{T}}{\rho-1}.
\label{eq:avgstab_const_final}
\end{align}
Noting $\rho-1=L\gamma(1-\frac{1}{mn})$, define the constant
\begin{equation}\label{eq:C_stab_def}
C_{\mathrm{stab}}
:=
\frac{
\frac{2GCL\gamma C_{w_0}}{\delta (1-\lambda)}
+\frac{4CG^2L\gamma^{2}}{\delta(1-\lambda)}
+\frac{4G^2\gamma}{mn}
}{
L\gamma\left(1-\frac{1}{mn}\right)
},
\end{equation}
so that \eqref{eq:avgstab_const_final} becomes
\begin{equation}\label{eq:avgstab_const_compact}
\epsilon_{\mathrm{ave\text{-}stab}}
\le
C_{\mathrm{stab}}\cdot\frac{\rho^{T}}{T}.
\end{equation}

On the other hand, Corollary~\ref{cor:opt-nonconvex} gives (with $\kappa:=L/\alpha$)
\begin{equation}\label{eq:opt_const_recall}
\epsilon_{\mathrm{opt}}
\le
\Bigg(\frac{Gr}{\alpha\gamma}+\frac{CG\kappa C_{w_0}}{2\delta(1-\lambda)}\Bigg)\frac{1}{T}
+
\Bigg(\frac{CG^2\kappa}{2\delta (1-\lambda)}+\frac{G^2\kappa}{4}\Bigg)\gamma.
\end{equation}
Combining \eqref{eq:avgstab_const_compact} and \eqref{eq:opt_const_recall}, we obtain
\begin{align}
\epsilon_{\mathrm{exc}}(T)
&\le \epsilon_{\mathrm{ave\text{-}stab}}+\epsilon_{\mathrm{opt}}\nonumber\\
&\le
\Bigg(\frac{Gr}{\alpha\gamma}+\frac{CG\kappa C_{w_0}}{2\delta(1-\lambda)}\Bigg)\frac{1}{T}
+
C_{\mathrm{stab}}\frac{\rho^{T}}{T}
+
\Bigg(\frac{CG^2\kappa}{2\delta (1-\lambda)}+\frac{G^2\kappa}{4}\Bigg)\gamma.
\label{eq:exc_const_master}
\end{align}

We now choose $T$ to balance the two $1/T$-terms in \eqref{eq:exc_const_master}. Let
\[
C_{\mathrm{opt}}:=\frac{Gr}{\alpha\gamma}+\frac{CG\kappa C_{w_0}}{2\delta(1-\lambda)}.
\]
Take
\begin{equation}\label{eq:T_star_const_choice}
T^* := \left\lceil \frac{1}{\mu}\ln\!\Big(\frac{C_{\mathrm{opt}}}{C_{\mathrm{stab}}}\Big)\right\rceil,
\qquad \mu=\ln\rho,
\end{equation}
(assuming $C_{\mathrm{opt}}>C_{\mathrm{stab}}$; otherwise one may take $T^*=1$).
Then $\rho^{T^*}\le e\,\frac{C_{\mathrm{opt}}}{C_{\mathrm{stab}}}$, and hence
\begin{align}
\frac{C_{\mathrm{opt}}}{T^*}+C_{\mathrm{stab}}\frac{\rho^{T^*}}{T^*}
&\le
\frac{C_{\mathrm{opt}}}{T^*}+e\,\frac{C_{\mathrm{opt}}}{T^*}
=(1+e)\frac{C_{\mathrm{opt}}}{T^*}\nonumber\\
&\le
(1+e)\,C_{\mathrm{opt}}\cdot\frac{\mu}{\ln(C_{\mathrm{opt}}/C_{\mathrm{stab}})}.
\label{eq:balance_const}
\end{align}
Substituting \eqref{eq:balance_const} into \eqref{eq:exc_const_master} yields the explicit bound
\begin{align}
\epsilon_{\mathrm{exc}}^*
&\le
(1+e)\,C_{\mathrm{opt}}\cdot\frac{\ln\!\big(1+L\gamma-\frac{L\gamma}{mn}\big)}
{\ln(C_{\mathrm{opt}}/C_{\mathrm{stab}})}
+
\Bigg(\frac{CG^2\kappa}{2\delta (1-\lambda)}+\frac{G^2\kappa}{4}\Bigg)\gamma,
\end{align}
where $C_{\mathrm{stab}}$ is given in \eqref{eq:C_stab_def}.

Noting $\mu=\ln\rho=\Theta(L\gamma)$ for $\gamma$ small, the choice~\eqref{eq:T_star_const_choice} gives
\[
T^*=\mathcal{O}\!\left(\frac{1}{L\gamma}\log\!\Big(\frac{C_{\mathrm{opt}}}{C_{\mathrm{stab}}}\Big)\right).
\]
Moreover, since $\ln(C_{\mathrm{opt}}/C_{\mathrm{stab}})=\Theta(\log(1/\gamma))$ in the typical regime,
the minimized excess risk satisfies
\[
\epsilon_{\mathrm{exc}}^*
=\mathcal{O}\!\left(
C_{\mathrm{opt}}\cdot\frac{L\gamma}{\log(1/\gamma)}
\;+\;
G^2\kappa\,\gamma
\right)
=
\mathcal{O}\!\left(
G^2\kappa\,\gamma
\right)
\quad\text{(up to logarithmic factors).}
\]

\medskip

\paragraph{Case 2: Decreasing Learning Rate.}

Let $\gamma_t=\frac{v}{t+1}$.
Recall that the averaged iterate is
\[
\bm{w}_{\mathrm{avg}}^{(T)}
=\frac{\sum_{t=0}^{T-1}\gamma_t\,\overline{\bm{w}}^{(t)}}{\sum_{t=0}^{T-1}\gamma_t},
\qquad
\bm{w}_{\mathrm{avg}}^{\prime(T)}
=\frac{\sum_{t=0}^{T-1}\gamma_t\,\overline{\bm{w}}^{\prime(t)}}{\sum_{t=0}^{T-1}\gamma_t}.
\]
By $G$-Lipschitzness (Assumption~\ref{ass:lipschitz}) and Jensen's inequality,
\begin{align}
\epsilon_{\mathrm{ave\text{-}stab}}
&\le G\,\mathbb{E}\bigl\|\bm{w}_{\mathrm{avg}}^{(T)}-\bm{w}_{\mathrm{avg}}^{\prime(T)}\bigr\|
= G\,\mathbb{E}\Bigl\|\frac{\sum_{t=0}^{T-1}\gamma_t(\overline{\bm{w}}^{(t)}-\overline{\bm{w}}^{\prime(t)})}{\sum_{t=0}^{T-1}\gamma_t}\Bigr\| \nonumber\\
&\le \frac{G}{\sum_{t=0}^{T-1}\gamma_t}\sum_{t=0}^{T-1}\gamma_t\,\mathbb{E}\bigl\|\overline{\bm{w}}^{(t)}-\overline{\bm{w}}^{\prime(t)}\bigr\|.
\label{eq:avgstab-start}
\end{align}
For each $t$, applying the (latest) non-convex stability bound in Corollary~\ref{cor:stability-nonconvex} at horizon $t$
(and dividing by $G$ on both sides), we have
\begin{align}
\mathbb{E}\bigl\|\overline{\bm{w}}^{(t)}-\overline{\bm{w}}^{\prime(t)}\bigr\|
\le \frac{1}{G}\,\epsilon_{\mathrm{stab}}(t)
&\le
\Bigl(\frac{4C v^{a} C_{w_0}}{\delta}
+\frac{v^{a}}{Gmn}+\frac{4G}{mn}\Bigr)t^{p}
+\frac{2GC v^{a}}{\delta(1-\lambda)}t^{q},
\label{eq:Delta-from-stab}
\end{align}
where we denote
\[
a:=\frac{1}{2+vL},\qquad
p:=\frac{1+vL}{2+vL}=1-a,\qquad
q:=\frac{vL}{2+vL}=1-2a.
\]
Substituting \eqref{eq:Delta-from-stab} into \eqref{eq:avgstab-start} and using $\gamma_t=\frac{v}{t+1}$ yields
\begin{align}
\epsilon_{\mathrm{ave\text{-}stab}}
&\le
\frac{G}{\sum_{t=0}^{T-1}\frac{v}{t+1}}
\sum_{t=0}^{T-1}\frac{v}{t+1}
\Biggl[
\Bigl(\frac{4C v^{a} C_{w_0}}{\delta}
+\frac{v^{a}}{Gmn}+\frac{4G}{mn}\Bigr)t^{p}
+\frac{2GC v^{a}}{\delta(1-\lambda)}t^{q}
\Biggr]\nonumber\\
&=
\frac{G}{\sum_{t=0}^{T-1}\frac{1}{t+1}}
\Biggl[
\Bigl(\frac{4C v^{a} C_{w_0}}{\delta}
+\frac{v^{a}}{Gmn}+\frac{4G}{mn}\Bigr)
\sum_{t=1}^{T-1} t^{p-1}
+\frac{2GC v^{a}}{\delta(1-\lambda)}
\sum_{t=1}^{T-1} t^{q-1}
\Biggr].
\label{eq:avgstab-sums}
\end{align}
Next, we bound the denominator and the two power sums by integrals. For $T\ge 3$, the harmonic lower bound gives
\[
\sum_{t=0}^{T-1}\frac{1}{t+1}\ge \frac{1}{2}\ln T,
\]
and for any $\theta\in(0,1]$,
\[
\sum_{t=1}^{T-1}t^{\theta-1}
\le 1+\int_{1}^{T}x^{\theta-1}\,dx
=1+\frac{T^{\theta}-1}{\theta}
\le \frac{2}{\theta}T^{\theta}.
\]
Applying these with $\theta=p$ and $\theta=q$ to \eqref{eq:avgstab-sums} yields
\begin{align}
\epsilon_{\mathrm{ave\text{-}stab}}
&\le
\frac{2G}{\ln T}
\Biggl[
\Bigl(\frac{4C v^{a} C_{w_0}}{\delta}
+\frac{v^{a}}{Gmn}+\frac{4G}{mn}\Bigr)\cdot \frac{2}{p}T^{p}
+\frac{2GC v^{a}}{\delta(1-\lambda)}\cdot \frac{2}{q}T^{q}
\Biggr]\nonumber\\
&\le
\frac{K_{\mathrm{stab},1}\,T^{p}+K_{\mathrm{stab},2}\,T^{q}}{\ln T},
\label{eq:avgstab-final}
\end{align}
where we set the explicit constants
\[
K_{\mathrm{stab},1}:=\frac{4G}{p}\Bigl(\frac{4C v^{a} C_{w_0}}{\delta}
+\frac{v^{a}}{Gmn}+\frac{4G}{mn}\Bigr),
\qquad
K_{\mathrm{stab},2}:=\frac{8G^2C v^{a}}{q\,\delta(1-\lambda)}.
\]

Combining \eqref{eq:avgstab-final} with the optimization error bound in Corollary~\ref{cor:opt-nonconvex},
\[
\epsilon_{\mathrm{opt}}\le \frac{K_{\mathrm{opt}}}{\ln T},
\]
we obtain the excess risk bound
\begin{align}
\epsilon_{\mathrm{exc}}
\le \epsilon_{\mathrm{ave\text{-}stab}}+\epsilon_{\mathrm{opt}}
\le
\frac{K_{\mathrm{stab},1}\,T^{p}+K_{\mathrm{stab},2}\,T^{q}+K_{\mathrm{opt}}}{\ln T}.
\label{eq:exc-case2-master}
\end{align}

Finally, since $p>q$, the dominant stability growth is $T^{p}$. Define
\[
K_{\mathrm{stab}}:=K_{\mathrm{stab},1}+K_{\mathrm{stab},2}.
\]
Then \eqref{eq:exc-case2-master} implies
\[
\epsilon_{\mathrm{exc}}
\le \frac{K_{\mathrm{stab}}\,T^{p}+K_{\mathrm{opt}}}{\ln T}.
\]
Treating the right-hand side as $f(T)=(K_{\mathrm{stab}}T^p+K_{\mathrm{opt}})/\ln T$ and using the standard balance condition
(at the minimizer, the two contributions are comparable up to a $\ln T$ factor),
\[
p\,K_{\mathrm{stab}}T^{p}\ln T \approx K_{\mathrm{opt}},
\]
we obtain the explicit stopping time
\begin{equation}
T^*
:=\left(\frac{K_{\mathrm{opt}}}{pK_{\mathrm{stab}}\ln\!\Big(\frac{K_{\mathrm{opt}}}{pK_{\mathrm{stab}}}\Big)}\right)^{\!\frac{1}{p}}
\quad(\text{for } \tfrac{K_{\mathrm{opt}}}{pK_{\mathrm{stab}}}\ge e),
\end{equation}
and substituting $T^*$ back gives an explicit minimum bound
\begin{align}
\epsilon_{\mathrm{exc}}^*
\le
\frac{p\,K_{\mathrm{opt}}}{\ln\!\Big(\frac{K_{\mathrm{opt}}}{pK_{\mathrm{stab}}}\Big)}.
\end{align}

Ignoring only logarithmic factors, the optimal stopping time is determined by
balancing the two terms $K_{\mathrm{stab}}T^{p}\approx K_{\mathrm{opt}}$, which yields
\[
T^*
=\tilde{\mathcal{O}}\!\left(
\left(
\frac{
\frac{rG}{v\alpha}
+\frac{CG\kappa C_{w_0}}{\delta(1-\lambda)}
+\frac{vG^2\kappa}{\delta(1-\lambda)}
+\frac{vG^2}{\alpha}
}{
\frac{G C_{w_0}}{\delta}
+\frac{G^2}{mn}
+\frac{G^2}{\delta(1-\lambda)}
}
\right)^{\!\frac{1}{p}}
\right).
\]

Substituting $T^*$ back, the minimized excess risk satisfies
\[
\epsilon_{\mathrm{exc}}^*
=\tilde{\mathcal{O}}\!\left(
\frac{rG}{v\alpha}
+\frac{CG\kappa C_{w_0}}{\delta(1-\lambda)}
+\frac{vG^2\kappa}{\delta(1-\lambda)}
+\frac{vG^2}{\alpha}
\right),
\]
which makes explicit the dependence on the network parameters
$(\lambda,\delta)$, the initialization scale $C_{w_0}$, and the gradient bound $G$.

This completes the proof of Corollary~\ref{cor:exc-nonconvex}.

\end{proof}